\documentclass[12pt]{article}

 \usepackage{times}

\usepackage{natbib}

\usepackage[utf8]{inputenc} %
\usepackage[T1]{fontenc}    %
\usepackage{url}            %
\usepackage{booktabs}       %
\usepackage{nicefrac}       %
\usepackage{microtype}      %
\usepackage{wrapfig}
\usepackage{algorithm}

\usepackage{xspace}
\usepackage{amsfonts}
\usepackage{amsmath}
\usepackage{amssymb}
\usepackage{amsthm}
\usepackage{mathtools}

\mathtoolsset{%
}

\usepackage[utf8]{inputenc}
\usepackage[margin=1.04in,headheight=15pt]{geometry}

\usepackage[proportional,tabular,lining,sf,mono=false]{libertine}

\newcommand{\mem}{\rvz}

\usepackage{enumitem}

\usepackage[dvipsnames,svgnames,table]{xcolor}
\colorlet{MyRed}{Crimson!75!Black}
\colorlet{MyGreen}{DarkGreen!80!Black}
\colorlet{MyBlue}{MediumBlue}

\usepackage{tikz}
\usetikzlibrary{calc}

\usepackage[showdeletions]{color-edits}		%
\usepackage[normalem]{ulem}		%

\newcommand{\debug}[1]{{#1}}		%

\usepackage{algorithm}
\usepackage{algpseudocode}  
\newcounter{algsubstate}
\renewcommand{\thealgsubstate}{\alph{algsubstate}}

\usepackage{hyperref}
\hypersetup{
colorlinks=true,
linktocpage=true,
pdfstartview=FitH,
breaklinks=true,
pdfpagemode=UseNone,
pageanchor=true,
pdfpagemode=UseOutlines,
plainpages=false,
bookmarksnumbered,
bookmarksopen=false,
bookmarksopenlevel=1,
hypertexnames=true,
pdfhighlight=/O,
urlcolor=MyBlue,linkcolor=MyBlue,citecolor=MyBlue,	%
pdftitle={},
pdfauthor={},
pdfsubject={},
pdfkeywords={},
pdfcreator={pdfLaTeX},
pdfproducer={LaTeX with hyperref}
}
\usepackage[showdeletions]{color-edits}		%
\usepackage[normalem]{ulem}		%

\usepackage[sort&compress,capitalize,nameinlink]{cleveref}	

\crefname{assumption}{Assumption}{Assumptions}

\usepackage{thmtools}		%
\usepackage{thm-restate}		%

\usepackage[font=small]{caption}
\usepackage{subcaption}

\usepackage{arydshln}

\newtheorem{theorem}{Theorem}		%
\newtheorem{corollary}{Corollary}		%
\newtheorem{lemma}{Lemma}		%

\newtheorem{remark}{Remark}

\newtheorem*{corollary*}{Corollary}		%
\newtheorem*{theorem*}{Theorem}

\newtheorem{definition}{Definition}		%
\newtheorem*{definition*}{Definition}		%
\newtheorem*{assumption*}{Assumptions}		%
\newtheorem*{example*}{Example}		%

\DeclarePairedDelimiter{\braces}{\{}{\}}		%
\DeclarePairedDelimiter{\bracks}{[}{]}		%
\DeclarePairedDelimiter{\parens}{(}{)}		%

\DeclarePairedDelimiter{\abs}{\lvert}{\rvert}		%

\DeclarePairedDelimiterX{\setdef}[2]{\{}{\}}{#1:#2}		%
\DeclarePairedDelimiterXPP{\exclude}[1]{\mathopen{}\setminus}{\{}{\}}{}{#1}

\newcommand{\newmacro}[2]{\newcommand{#1}{\debug{#2}}}		%
\newcommand{\R}{\mathbb{R}}		%

\DeclareMathOperator{\ex}{\mathbb{E}}		%
\DeclareMathOperator{\prob}{\mathbb{P}}		%

\DeclarePairedDelimiterXPP{\exof}[1]{\ex}{[}{]}{}{%
 #1}

\DeclarePairedDelimiterXPP{\probof}[1]{\prob}{(}{)}{}{%
 #1}

\DeclarePairedDelimiterXPP{\oneof}[1]{\one}{\{}{\}}{}{%
 #1}

\DeclarePairedDelimiter{\norm}{\lVert}{\rVert}		%
\DeclarePairedDelimiterXPP{\dnorm}[1]{}{\lVert}{\rVert}{_{\ast}}{#1}		%

\DeclarePairedDelimiterXPP{\onenorm}[1]{}{\lVert}{\rVert}{_{1}}{#1}		%
\DeclarePairedDelimiterXPP{\twonorm}[1]{}{\lVert}{\rVert}{_{2}}{#1}		%
\DeclarePairedDelimiterXPP{\supnorm}[1]{}{\lVert}{\rVert}{_{\infty}}{#1}		%

\DeclarePairedDelimiterX{\braket}[2]{\langle}{\rangle}{#1\mathopen{}\delimsize\vert\mathopen{}#2}

\DeclarePairedDelimiterX{\inner}[2]{\langle}{\rangle}{#1,#2}		%

\newmacro{\mat}{\R} %
\newmacro{\field}{\mathbf{F}} %

\newcommand{\id}{\rmI} %

\newcommand{\zero}{\vZero}
\newcommand{\one}{\mathds{1}}
\newcommand{\eg}{\emph{e.g.},\xspace}		%
\newcommand{\ie}{\emph{i.e.},\xspace}		%

\newmacro{\dep}{l} %
\newmacro{\wid}{d} %

\newmacro{\weights}{\rmW} %
\newmacro{\weight}{\rvw} %
\newmacro{\altweights}{\rmV} %
\newmacro{\altweight}{v} %
\newmacro{\bias}{\rvb}
\newmacro{\iw}{w^*}

\newmacro{\data}{\rvx} %
\newmacro{\altdata}{\rvz}

\newmacro{\batchscale}{\gamma} %
\newmacro{\batchshift}{\beta} %
\newmacro{\mscale}{\rmGamma} %
\newmacro{\mshift}{\bm{\beta}} %

\newmacro{\neuron}{i} %
\newmacro{\altneuron}{j} %

\newmacro{\relu}{\textsf{ReLU}} %
\newmacro{\sigmoid}{\sigma}

\newmacro{\tf}{\mathrm{TF}}
\newmacro{\att}{\mathrm{Attn}}
\newmacro{\val}{\rmV}
\newmacro{\key}{\rmK}
\newmacro{\query}{\rmQ}
\newmacro{\head}{i}
\newmacro{\out}{\rmW_o}
\newmacro{\heads}{H}
\newmacro{\embed}{\rmP}
\newmacro{\emb}{\rvp}
\newmacro{\normal}{\matcal{N}} %
\newmacro{\uni}{\mathcal{U}} %

\usepackage{bm}
\usepackage{bbm}
\usepackage[mathscr]{eucal}

\def\1{\bm{1}}

\def\eps{{\epsilon}}

\def\rva{{\mathbf{a}}}
\def\rvb{{\mathbf{b}}}
\def\rvc{{\mathbf{c}}}

\def\rve{{\mathbf{e}}}

\def\rvg{{\mathbf{g}}}

\def\rvp{{\mathbf{p}}}

\def\rvw{{\mathbf{w}}}
\def\rvx{{\mathbf{x}}}

\def\rvz{{\mathbf{z}}}

\def\rmA{{\mathbf{A}}}
\def\rmB{{\mathbf{B}}}
\def\rmC{{\mathbf{C}}}
\def\rmD{{\mathbf{D}}}
\def\rmE{{\mathbf{E}}}

\def\rmI{{\mathbf{I}}}

\def\rmK{{\mathbf{K}}}

\def\rmM{{\mathbf{M}}}

\def\rmP{{\mathbf{P}}}
\def\rmQ{{\mathbf{Q}}}

\def\rmS{{\mathbf{S}}}

\def\rmV{{\mathbf{V}}}
\def\rmW{{\mathbf{W}}}
\def\rmX{{\mathbf{X}}}
\def\rmY{{\mathbf{Y}}}

\def\rmGamma{{\mathbf{\Gamma}}}

\def\vZero{{\bm{0}}}
\def\vOne{{\bm{1}}}

\def\va{{\bm{a}}}
\def\vb{{\bm{b}}}
\def\vc{{\bm{c}}}
\def\vd{{\bm{d}}}
\def\ve{{\bm{e}}}

\def\vv{{\bm{v}}}

\def\vx{{\bm{x}}}

\def\vz{{\bm{z}}}

\DeclareMathAlphabet{\mathsfit}{\encodingdefault}{\sfdefault}{m}{sl}
\SetMathAlphabet{\mathsfit}{bold}{\encodingdefault}{\sfdefault}{bx}{n}

\def\cD{{\mathcal{D}}}

\def\cT{{\mathcal{T}}}

\def\bR{{\mathbb{R}}}

\newcommand{\softmax}{\sigma_{\text{S}}}

\addauthor[Shashank]{SR}{MediumBlue}

\addauthor[Dimitris]{DP}{Green}

\addauthor[Kangwook]{KL}{DarkGreen}

\addauthor[Angel]{AG}{DarkOrange}

\addauthor[GD]{GD}{DarkBlue}

\addauthor[Jason]{JL}{DarkMagenta}

\usepackage{colortbl}
\newcommand\y{\cellcolor{clight2}}
\definecolor{clight2}{rgb}{0.54, 0.81, 0.94}
\newcommand\tikznode[3][]%
   {\tikz[remember picture,baseline=(#2.base)]
      \node[minimum size=0pt,inner sep=0pt,#1](#2){#3};%
   }

\definecolor{applegreen}{rgb}{0.8, 0.6, 0.8}
\newcommand\yy{\cellcolor{applegreen}} %
\definecolor{teal}{rgb}{0.86, 0.44, 0.58}

\definecolor{oldlavander}{rgb}{0.8, 0.8, 1.0}
\newcommand\ol{\cellcolor{oldlavander}} %

\definecolor{bluebell}{rgb}{0.64, 0.64, 0.82}
\newcommand\lo{\cellcolor{bluebell}}  %
\definecolor{bluegray}{rgb}{0.4, 0.6, 0.8}
\definecolor{ceil}{rgb}{0.57, 0.63, 0.81}
\definecolor{darkpastelblue}{rgb}{0.47, 0.62, 0.8}
\definecolor{babypink}{rgb}{0.96, 0.76, 0.76}
\definecolor{bubblegum}{rgb}{1, 0.76, 0.85}

\newcommand\scr{\cellcolor{magen}}
\definecolor{magen}{rgb}{0.6, 0.4, 0.8} %

\definecolor{olive}{rgb}{0.69, 0.61, 0.85}
\newcommand\pc{\cellcolor{olivene}}
\definecolor{olivene}{rgb}{0.86, 0.82, 1.0} %

\tikzset{>=stealth}

\usepackage{blkarray}

\newmacro{\gar}{\rvg}
\newmacro{\Input}{\rmX}
\newmacro{\dpoints}{\rmD} %
\newmacro{\dpoint}{\rvx}    %
\newcolumntype{s}{>{\columncolor[HTML]{AAACED}} p{3cm}}
\newcommand{\rvOne}{\mathbf{1}}
\newmacro{\commands}{\rmC} %
\newmacro{\command}{\rvc} %

\newmacro{\pointers}{\rmC}  %
\newmacro{\pointer}{\texttt{PC}}   %
\newmacro{\per}{\mathrm{per}}

\newmacro{\loss}{l}
\newmacro{\ssize}{\eta}

\title{\bf Looped Transformers as Programmable Computers}

\usepackage{fancyhdr}

\fancypagestyle{fancy_custom}{%
    \lhead{}
    \rhead{}
    \rfoot{}
    \cfoot{\thepage}
    \chead{\em Looped Transformers as Programmable Computers}
       
}
\author{
Angeliki Giannou$^{w}$\thanks{Equal contribution. The title of this paper was not created by a transformer, but we can't guarantee the same for this footnote.}\;, 
Shashank Rajput$^{w*}$,
Jy-yong Sohn$^{w}$,
\\
Kangwook Lee$^{w}$,
Jason D. Lee$^{p}$,
 Dimitris Papailiopoulos$^{w}$ \vspace{0.5cm}\\
\normalsize $^p$ Princeton University \\
\normalsize$^w$ University of Wisconsin-Madison \\ 
}

\ifdefined\usebigfont

\usepackage{times}
\usepackage[fontsize=13pt]{scrextend}
\makeatletter
\@ifpackageloaded{geometry}{\AtBeginDocument{\newgeometry{letterpaper,left=1.56in,right=1.56in,top=1.71in,bottom=1.77in}}}{\usepackage[letterpaper,left=1.56in,right=1.56in,top=1.71in,bottom=1.77in]{geometry}}
\makeatother
\else
\fi

\begin{document}

\maketitle

\begin{abstract}
\noindent We present a framework for using transformer networks as universal computers by programming them with specific weights and placing them in a loop. Our input sequence acts as a punchcard, consisting of instructions and memory for data read/writes. We demonstrate that a constant number of encoder layers can emulate basic computing blocks, including embedding edit operations, non-linear functions, function calls, program counters, and conditional branches.
Using these building blocks, we emulate a small instruction-set computer. This allows us to map iterative algorithms  to programs that   can be executed by a looped, 13-layer transformer. We show how this transformer, instructed by its input, can  emulate a basic calculator, a basic linear algebra library, and in-context learning algorithms that employ backpropagation. Our work highlights the versatility of the attention mechanism, and demonstrates that even shallow transformers can execute full-fledged, general-purpose programs.
\end{abstract}
\section{Introduction}
Transformers (TFs) have become a popular choice for a wide range of machine learning tasks, achieving state-of-the-art results in fields such as natural language processing and computer vision~\citep{vaswani2017attention,khan2022transformers,yuan2021tokens,dosovitskiy2020image}. One key reason for their success is their ability to capture higher-order relationships and long-range dependencies across tokens, through attention. This allows TFs to model contextual information and  makes them effective in tasks such as machine translation and language modeling, where they have consistently outperformed other methods \citep{vaswani2017attention,kenton2019bert}.

Language models with billions of parameters, such as GPT-3 (175B parameters \cite{brown2020language}) and PaLM (540B parameters \cite{chowdhery2022palm}), have achieved state-of-the-art performance on many natural language processing tasks. Interestingly, some of these large language models (LLMs) can also perform in-context learning, adapting to and performing a specific task, {\it on-the-fly}, based on a brief prompt and a few examples.
The ability to perform in-context learning (ICL) arises without explicit training for it, and allows these large models to efficiently perform new tasks without requiring weight updates.

Surprisingly, through in-context learning LLMs can perform algorithmic tasks and reasoning, as demonstrated in several works including \cite{nye2021show,wei2022chain,lewkowycz2022solving,wei2022emergent,zhou2022teaching,dasgupta2022language,chung2022scaling}. For example, \cite{zhou2022teaching} showed that LLMs can successfully perform addition on unseen examples when prompted with a multidigit addition algorithm and a few examples of addition. These results suggest that LLMs can apply algorithmic principles and perform pre-instructed commands on a given input at inference time, {\it as if interpreting natural language as code}.

Constructive arguments have demonstrated that Transformers can simulate Turing Machines with enough depth or recursive links between attention layers \cite{perez2022attention, Perez2019turing, wei2021statistically}. This demonstrates the potential of transformer networks to precisely follow algorithmic instructions specified by the input. Yet, these constructions are more generalized and do not provide insight into how to create Transformers that can carry out particular algorithmic tasks, or compile programs in a higher-level programming language. 

More specialized designs can however allow TFs to execute higher level programs. For example, in~\cite{weiss2021thinking}, the authors design a computational model and a programming language that maps simple selection and aggregation commands on indexed input tokens. This language can be used to create several interesting algorithms, such as counting tokens, sorting, creating histograms, and recognizing Dyck-$k$ languages. Programs written in Restricted Access Sequence Processing Language (RASP) can then be mapped into transformer networks, which typically scale in size with the size of the program.
 
Another line of research has demonstrated methods for selecting the weights of a Transformer model to function as an optimization algorithm for learning linear regression models on-the-fly, performing implicit training at inference time when given training data as input~\citep{akyurek2022learning,von2022transformers}. These methods typically require a number of layers proportional to the number of iterations of the learning algorithm and are limited to a small set of loss functions and models. 

The ability to program transformer models to emulate the abstract computation of a Turing Machine, the specialized commands of languages like RASP, and the specific algorithms of in-context learning, highlights the potential for transformer networks as versatile programmable computers. Our research aims to explore this promising prospect, uncovering how the mechanics of attention can enable the emulation of a general-purpose computer inspired by instruction-set architectures.

 \paragraph{Our Contributions:}
In this paper, we demonstrate that transformer networks can simulate complex algorithms and programs by hardcoding them with specific weights and placing them in a loop. We do this by reverse engineering attention to emulate basic computing blocks, such as edit operations on the input sequence, nonlinear functions, function calls, program counters and conditional branches. Our paper demonstrates the importance of using a single loop or recursion to connect the transformer's output sequence back to its input, avoiding the need for a deep model.

We accomplish this by designing a transformer that can execute programs written in a generalized version of a single instruction, known as \texttt{SUBLEQ}(A,B,C), \ie  \texttt{SUB}tract and branch if \texttt{L}ess-than or \texttt{EQ}ual to zero. 
\texttt{SUBLEQ} is a single instruction language, defining a one-instruction set computer (OISC, pronounced ``whisk''). \texttt{SUBLEQ} consists of 3 memory address operands and when executed it subtracts the value at memory address A from the value at memory address B, and stores the result in B. If the result in B is less than or equal to zero, the execution jumps to address C, otherwise it proceeds to the next instruction.
Programs written in \texttt{SUBLEQ} language use only this command, yet this single instruction is capable of defining a universal computer~\citep{mavaddat1988urisc,subleq}.

\begin{wrapfigure}{r}{0.5\textwidth}
\vspace{-0.5cm}
\centering
\includegraphics[width=0.5\textwidth]{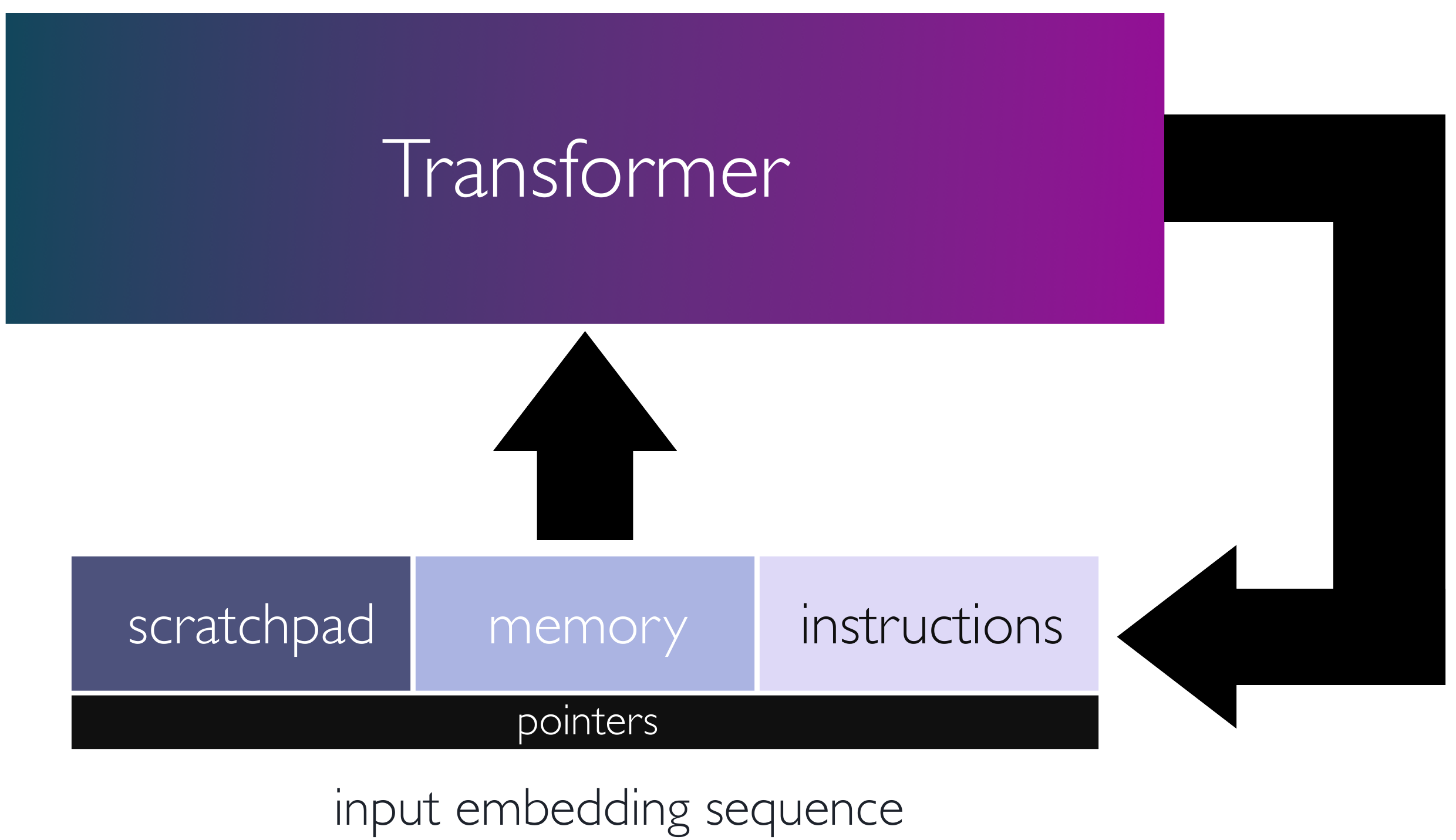}
\caption{A sketch of the looped transformer architecture, where the input sequence stores the commands, memory where the data is read/written from, and a scratchpad where intermediate results are stored. The input is processed by the network and the output is used as the new input, allowing the network to iteratively update an implicit state and perform complex computations. }
\label{fig:input}
\vspace{-1.5cm}
\end{wrapfigure}

We construct explicit transformers that implement \texttt{SUBLEQ}-like programs, of a more flexible single instruction which we call \texttt{FLEQ} which takes the form 
 \begin{align}
        &\texttt{mem}[c] = f_m (\texttt{mem}[a], \text{mem}[b])\nonumber \\ &\texttt{if }\texttt{mem}[\text{flag}]\leq 0\nonumber \\
        &\quad \texttt{ goto}\text{ instruction }p\nonumber
\end{align}
where $f_m$ can be selected from a set of functions (matrix multiplication/non-linear functions/polynomials/etc), which we can hardcode into the network.
The depth of a looped transformer that can execute \texttt{FLEQ} programs is not dependent on the depth of the program or the number of lines of code, but rather on the depth required to implement a single \texttt{FLEQ} instruction, which is constant. This is achieved by running the transformer in cycles over the input sequence, similar to how a CPU operates.

Using this framework, we demonstrate the ability to emulate a variety of functions at inference time, including a basic calculator, a basic linear algebra library (matrix transpose, multiplication, inversion, power iteration) and an in-context learning algorithm that implements backpropagation on implicit fully-connected networks. The input sequence, or the prompt, acts as a punchcard that includes the program in the form of instructions that the transformer needs to execute, while providing space for storing and processing the variables used in the program. {\it The transformer networks used to execute these programs are all of depth smaller or equal to thirteen, and the exact weight matrices for all these models are provided.} The following informal theorem summarizes our main findings:
\begin{theorem}[Informal]
There exists a looped transformer with less than 13 layers that can emulate a general purpose computer (see Sec.~\ref{ss:oneLine}), a basic calculator (see Sec.~\ref{sec:calc}),  numerical linear algebra methods, such as approximate matrix inverse and power iteration (see Sec.~\ref{sec:linearalg}), and in-context learning algorithms, such as SGD, on neural networks (See Sec.~\ref{sec:SGD}).
\end{theorem}
The precise size of the transformers constructed in this paper is also summarized in \cref{tab:sizeTFsummary}.

\begin{table}[H]

\centering
\small
\begin{tabular}{c||c|c|c}
     &  \# Layers &  \# Heads & Formal Statement \\
     \hline
    SUBLEQ & 9 & 2 & Lemma.~\ref{lem:OISC}\\
    Matrix Inversion & 13 & 1 & Lemma.~\ref{lem:matrix_inversion}\\
    Power Iteration & 13 & 1 & Lemma.~\ref{lem:power_iteration}\\
    SGD & 13 & 1 & Lemma.~\ref{lem:SGD}
\end{tabular}\label{tab:sizeTFsummary}
\caption{Looped transformer sizes required to successfully emulate the functionalities of a one instruction set computer (OISC), perform basic calculations, run numerical linear algebra algorithms, and in-context learning using Stochastic Gradient Descent on a neural network. The width of these networks depends on the complexity of the functions implemented, and typically range from $O(\log(\textsf{length\_input})+\textsf{embedding\_dimension})$ to at most polynomial in the approximation error required when implementing arbitrary loss functions for in-context learning.}
\end{table}

Our research highlights the flexibility of the attention mechanism and the importance of even a single loop making it possible to design models that can emulate complex iterative algorithms and execute general programs. It further  demonstrates the ability of transformer models to efficiently perform complex mathematical and algorithmic tasks. It is conceivable that modern transformers, such as GPT-3, utilize similar internal subroutines when performing various tasks. In a way, these models may possess the ability to elicit a specific skill or algorithm, akin to a function call, when given in-context examples and instructions.   However, this hypothesis should be taken with caution, as the way we design our constructions shares no similarities with how real-world language models are trained.

We hope that our study will encourage further research into the potential of attention mechanisms, and the ability of language models to execute algorithmic instructions. 
Our proposed designs can aid in determining the minimal transformer network size required to perform specific algorithmic tasks.
Additionally, we hope that our findings will contribute to the development of methods to enhance the capabilities of trained language models by utilizing smaller, reverse-engineered transformer networks for specific algorithmic tasks

\section{Prior Work}

Our work is inspired by the recent results on the expressive power of Transformer networks and their in-context learning capabilities.

 In 
 \citep{perez2022attention,Perez2019turing,wei2021statistically}
the authors explore the computational properties of Transformers  establishing that they are Turing complete, meaning that they can simulate a Turing machine. 
 The constructions typically require high/infinite precision (apart from that of \cite{wei2021statistically}), and recursion around attention layers.
  In \cite{yun2019transformers}, the authors prove that given access to sufficient width/depth TFs can act as universal sequence to sequence approximators.

In \cite{weiss2021thinking}, the authors propose a computational model for the transformer-encoder in the form of a  domain-specific  language called the Restricted Access Sequence Processing Language (RASP). The model maps the basic components of a TF encoder into simple primitives. Examples of tasks that could be learned by a Transformer are provided, and the maximum number of heads and layers necessary to encode a task in a transformer are analyzed.

In a recent and related work, \cite{lindner2023tracr} suggests using transformer networks as programmable units and introduces a compiler called Tracr which utilizes RASP. However, the expressivity limitations and unclear Turing completeness of the language are discussed in \cite{weiss2021thinking, merrill2022saturated, lindner2023tracr}. Our approach, in contrast, demonstrates the potential of transformer networks to serve as universal computers, enabling the implementation of arbitrary nonlinear functions and emulating iterative, non-linear algorithms. Furthermore, our framework allows the depth of our transformers {\it to not} scale in proportion to the lines of code that they execute, allowing the implementation of iterative algorithms, expanding the potential applications.

In \cite{gargcan} the authors demonstrate that standard Transformers (\eg GPT-2) can be trained from scratch to perform in-context learning of linear functions and more complex model classes, such as two-layer neural networks, with performance that matches or exceeds task-specific learning algorithms. A useful element of their analysis is the fact that language is completely removed from the picture, and they perform all operations on the level of vector embeddings. This allows a higher abstraction level than using language as an input, and in fact is what also allows us to obtain our derivations.

Motivated by the above experimental work, in \cite{akyurek2022learning},
the authors investigate the hypothesis that TF-based in-context learners emulate standard learning algorithms implicitly at inference time. The authors provide evidence for this hypothesis by constructing transformers that implement SGD for linear models, showing that trained in-context learners closely match the predictors computed by these algorithms.

In a similar vein, \cite{von2022transformers} argues that training Transformers on auto-regressive tasks is closely related to gradient-based meta-learning formulations. The authors also provide a hard-coded weight construction showing the equivalence between data transformations induced by a single linear self-attention layer and gradient descent on a regression loss. The authors empirically show that when training linear attention TFs on simple regression tasks, the models learned by GD and Transformers have intriguing similarities. 

In \cite{liu2022transformers}, the authors 
test the hypothesis that TFs can perform algorithmic reasoning using fewer layers than the number of reasoning steps, in the context of finite automata. The authors characterized ``shortcut solutions'' that allow shallow Transformer models to exactly replicate the computation of an automaton on an input sequence, and showed that these solutions can be learned through standard training methods. As is expected this hypothesis is only true for a certain family of automata, as the general existence of shortcut solutions would imply the collapse of complexity classes that are widely believed not to be identical.

Other experimental studies have utilized recursion in transformer architectures in a similar manner to our constructions, although in our case we only utilize a single recursive link that feeds the output of the transformer back as an input~\citep{hutchins2022block,
shen2022sliced,
dehghani2018universal}.

\section{Preliminaries}\label{sec:prelim}

\paragraph{The transformer architecture.}
Our work follows a similar problem setting as previous studies (e.g. \cite{yun2019transformers,gargcan,akyurek2022learning,von2022transformers}) in which the input sequence consists of $d$-dimensional embedding vectors rather than tokens. This simplifies our results without sacrificing generality, as an embedding layer can map tokens to the desired vector constructions.

The input to each layer, $\rmX \in \bR^{d\times n}$, is a vector representation of a sequence of $n$ tokens, where each token is a $d$-dimensional column. In this paper, the terms ``token'' and ``column'' may be used interchangeably.

A  transformer layer outputs $f(\rmX)$, where $f$ is defined as follows:%
\begin{subequations}
\label{eq:TF}
\begin{align}
    \att({\rmX}) &= \rmX + \sum_{i=1}^\heads\val^\head\rmX\softmax(\rmX^\top\key^{\head \top}\query^\head\rmX)\label{eq:att}\\
    f(\rmX) &= \att(\rmX) + \weights_2\relu(\weights_1\att(\rmX) + \bias_1\vOne^\top_n) + \bias_2\vOne^\top_n\label{eq:relulayer}
\end{align}
\end{subequations}
where $\softmax$ is the softmax function applied on the columns of the input matrix, \ie 
$$[\softmax(\rmX,\lambda)]_{i,j} = \frac{e^{\lambda X_{i,j}}}{\sum_{k=1}^{n} e^{\lambda X_{k,j}}},$$ 
where $\lambda\geq 0$ is the temperature parameter, $\relu(x)=x\cdot { 1}_{x>0}$ is the ReLU activation, and $\vOne_n$ is the all ones vector of length $n$.
We refer to the $\key,\query,$ and $\val$ matrices as the key, query, and value matrices respectively\footnote{We'd like to note that typically the weight matrices are denoted as $\rmW_Q, \rmW_K, \rmW_V$ but to make notation cleaner, we use instead $\rmQ,\rmK,\rmV$.}
; the superscript $i$ that appears on the weight matrices indicates those corresponding to the $i$-th attention head.%
 Consistent with previous literature, the first equation \cref{eq:att} represents the attention layer. We refer to the combination of attention and ReLU layers as a single transformer layer.

\paragraph{Iterative computation through a simple loop.}

In the following sections, we utilize TF networks with multiple transformer layers. Let us refer to the output of such a multilayer TF as $\mathsf{TF}(\rmW;\rmX)$, where for simplicity $\rmW$ is the collection of all weight matrices required to define such a multi-layer TF.
\begin{wrapfigure}{r}{0.35\textwidth}
    \begin{minipage}{0.35\textwidth}
      \begin{algorithm}[H]
        \caption{\\ Looped Transformer}
        \begin{algorithmic}[1]
          \For{$i=1:T$}
         \State $\rmX \gets \mathsf{TF}(\rmW;\rmX)$
  \EndFor
        \end{algorithmic}      
      \end{algorithm}
      \vspace{-0.7cm}
    \end{minipage}
  \end{wrapfigure}

We use our constructions recursively, and feed the output back as an input sequence, allowing the network to perform iterative computation through a simple fixed-point like iteration. This recursive transformer is similar to past work on adding recursion to TF networks. We refer to these simple recursive TFs as {\it Looped Transformers}. 

Feeding the output back to its input is similar to how a traditional computer processes machine code, where it continually reads/writes data in memory, by executing one instruction at a time. The input sequence $\rmX$ includes the instructions and memory. Similar to how a CPU processes each line of code in a program, the transformer network processes parts of the input sequence to perform complex computations.
Like a CPU, the TF acts as a self-contained computational unit. The use of loops in this process is analogous to how CPUs operate using cycles.

While the analogy between TFs and CPUs can be entertaining, there are also many differences in implementation. 
It is important to keep these differences in mind and not rely too heavily on the analogy. The results obtained from using TFs as computational units do not require the analogy to be valid.

To be able to build compute boxes out of a TF network, it is crucial to format the input sequence $\rmX$ in a way that separates memory, a cache-like scratchpad, and commands.

\paragraph{Input sequence format.} The input to our transformer network has the following abstract form:
\begin{equation}\label{eq:input}
\Input =   \left[
\begin{array}{ccc|ccc|ccc}
 &\rmS && &\rmM & &&\rmC&\\
\emb_1 &\dots &\emb_s&\emb_{s+1} &\dots &\emb_{s+m} &\emb_{s+m+1}& \dots&\emb_{n}
\end{array},
\right]    
\end{equation}where $\rmS$ represents the portion of the input that serves as a ``scratchpad,'' $\rmM$ represents the portion that acts as memory that can be read from and written to, and $\rmC$ represents the portion that contains the commands provided by the user. The $\emb_1,\dots, \emb_n$ are positional encodings for the $n$ columns, which will be described in more detail in the following paragraph, and will be used as pointers to data and instructions. The structure of our input sequence bares similarities to that of \cite{wei2021statistically, akyurek2022learning} that also use scratchspace, and have a separate part for the input data.

\paragraph{Scratchpad.} The scratchpad is a crucial component of our constructions. This is the central location where the inputs and outputs of all computation are recorded. It is perhaps useful to think of this as an analogue to a CPU's cache memory. 
 It functions as a temporary workspace where data is copied, transformed, and manipulated in order to perform a wide variety of operations, ranging from simple arithmetic to more complex tasks such as matrix inversion.
 Regardless of the specific computation that is performed, the data necessary for the operation is always transferred from the memory to the scratchpad, and once the computation is completed, the data is transferred back to the memory. This allows the TF to perform the necessary calculations in a designated area, separate from other parts of the input sequence.

\paragraph{Memory.}
All the compute boxes we create require memory to perform specific actions. The memory component of the input sequence serves as a storage location for data. This data can take various forms, including scalars, vectors, and matrices, and is subject to manipulation through various operations. When computation is needed, the data is first copied from the memory to the scratchpad, where it is updated and transformed as necessary. Once the computation is complete, the updated data is then returned and copied back to the memory for future use or reference. In this way, the memory serves as a central repository for all relevant data, allowing it to be accessed and manipulated as needed.

\paragraph{Commands.}
Our framework implements a set of commands within a transformer network; these serve as instructions that guide the internal functioning of the transformer, similar to a low-level programming language. These commands include indicators for memory locations and operation directives, allowing the TF to execute complex computations and tasks in a consecutive and organized manner.

\section{Building Transformer Blocks towards General Computation}

\begin{wrapfigure}{r}{0.55\textwidth}
\centering
\vspace{-1.2em}
\includegraphics[width=0.55\textwidth]{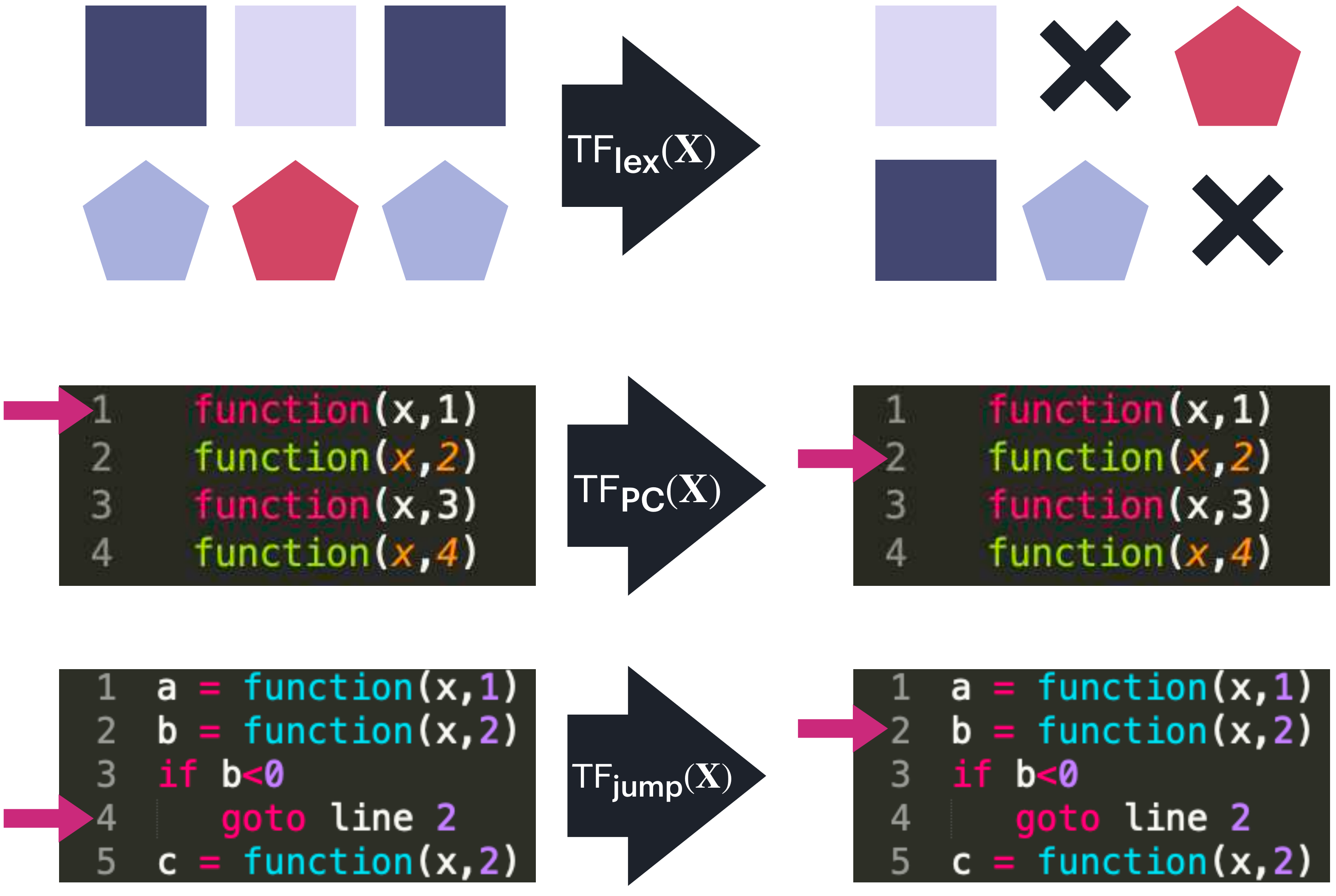}
\caption{A sketch of the three transformer blocks used as building blocks to implement a small instruction-set computer. These blocks handle edits in the input sequence (such as moving or copying from one block to another), keep track of the program counter, and execute a program counter jump if a specified condition is met.}
\label{fig:legoblocks}
\vspace{-1.2em}
\end{wrapfigure}

To build general compute boxes using transformer networks, specialized compute blocks are required. These blocks will be assembled to create the desired end functionality.
In this section, we highlight various operations that transformer layers can perform. These operations will serve the building blocks to create more complex routines and algorithms. These operations are designed to be interoperable with each other, leveraging the ability of attention to perform various tasks, such as producing approximate permutation matrices and approximating general functions through sigmoid activations. 

In the following sections, we focus on the fundamental components necessary to emulate a general-purpose computer, reserving the examination of how attention can replicate sigmoid-based functions in the sections that follow.

   \subsection{Positional Encodings, Program Counter, and Data Pointers}\label{ss:increaseCounter} 

To aid the transformer in locating the position of each token, each column of $\rmX$ is appended with positional encodings that is based on the column index. In this case, similar to \cite{wei2021statistically}, the positional encodings is the binary representation of the column index, which is appended to each column to keep the encoding dimension low, i.e., logarithmic in the sequence length. This approach to using positional encodings is slightly different from the typical method of adding them to the encodings of the input sequence. However, in this case, appending them as suffixes to the encodings allows for cleaner arguments and constructions.

In particular, the encoding for token/column indexed by $i$ is a $\log(n)$-dimensional $\pm1$ binary vector $\emb_i\in{\pm1}^{\log(n)}$, where $n$ is the length of the input sequence. Using the standard binary representation of an integer $i$, meaning $i=\sum_{k=0}^{\log(n)-1} 2^k \cdot b_k$, the positional encoding vector $\emb_i$ is set to $-1$ at index $j$ if the binary representation of $i$ has $0$ at the $j$-th index, \ie $b_i=0$, otherwise it is $+1$.
As a result, we have $\emb_i^T\emb_i = \log(n)$ and by Cauchy-Schwarz inequality, $\emb_i^T\emb_j < |\emb_i||\emb_j| = \sqrt{\log(n)}\sqrt{\log(n)}=\log(n)$ whenever $i\neq j$, since $\emb_i, \emb_j$ differ in at least one coordinate.

In the applications presented, the transformer often needs to execute iterative algorithms or go through a sequence of commands. To achieve this, we utilize a program counter that iterates through the commands. The counter contains the encoding of the location where the next command is stored. Additionally, a command may have data pointers that point to the location of the data the command needs to read and write to. Both the program counter and data pointers utilize the same positional encodings as discussed in the previous paragraph.
Using binary vectors as positional encodings allows us to easily increment the program counter by 1 (or any other amount) using the feed forward ReLU layers in the transformer architecture \eqref{eq:TF}. This is formalized in the following lemma, for the proof see \cref{app:increase}. %

\begin{lemma}
   Given two $d$-dimensional binary vectors representing two non-negative integers, there exists a 1-hidden layer feedforward network with ReLU activation, containing $8d$ activations in the hidden layer and $d$ neurons in the output layer, that can output the binary vector representation of their sum, as long as the sum is less than $2^{d+1}$.
\end{lemma}

Our positional encoding scheme can also be used to point to specific data locations for reading or writing, as discussed in the following section. This is achieved by using the same binary vectors as positional encodings for both the program counter and data pointers. 
Furthermore, this technique for pointing to specific data locations enables the transformer to effectively read and write from/to data during the execution of the algorithm or sequence of commands that is build to implement.

\subsection{\texttt{read} / \texttt{write}: Copying Data/Instructions to/from the Scratchpad}\label{ss:copy}

\begin{figure}[H]
  \centering
\includegraphics[scale =0.35%
]{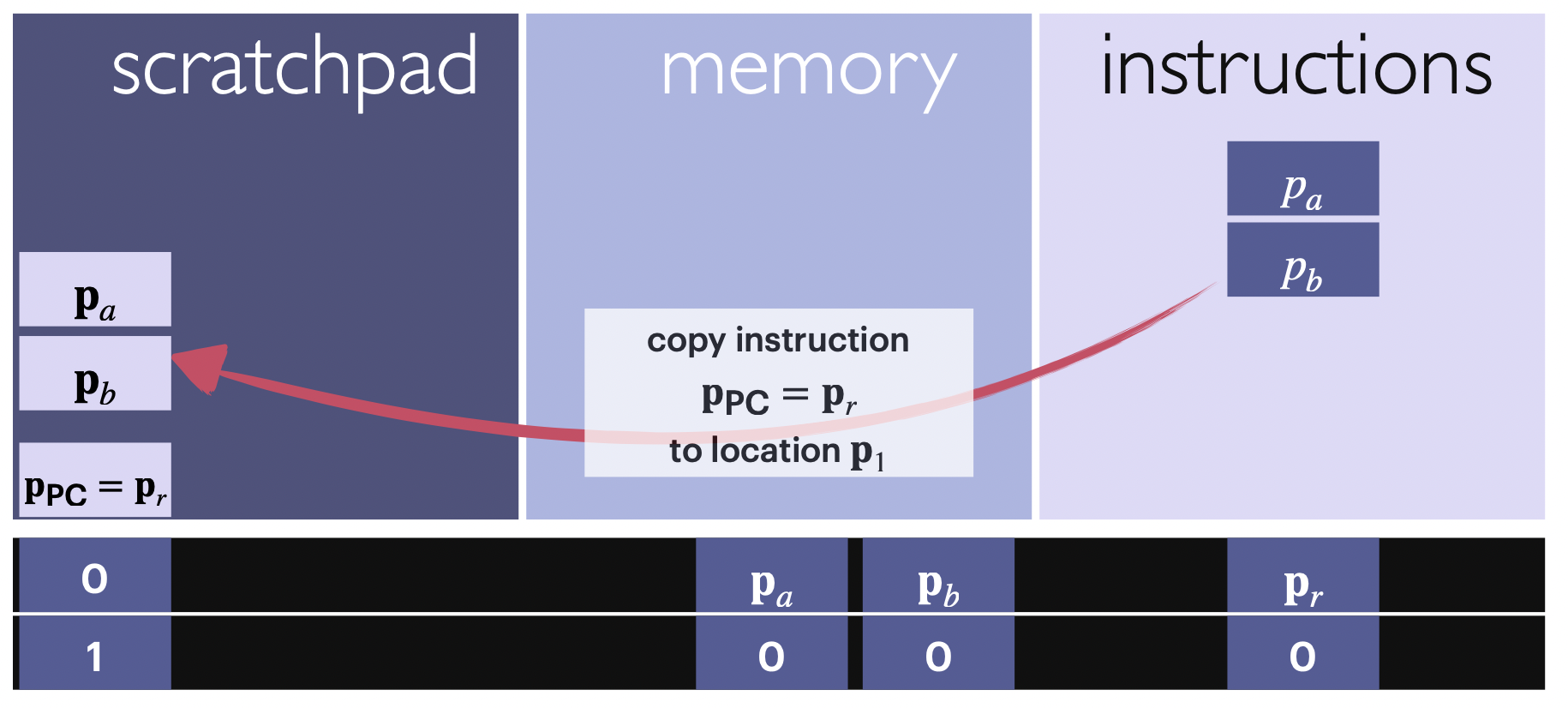}
  \caption{A sketch of the read operation. Arrows show command blocks being copied from the part of the input that is allocated to commands to the scratchpad. Typically an instruction is another set of pointers. Positional encodings and counters are used for tracking what is copied where.}
  \label{fig:copy}
\end{figure}

As previously stated, the scratchpad serves as a temporary memory for storing all information needed for computation. This includes copying commands and data to it, performing computation, and writing results back to memory. This process has similarities with the copy/write mechanism developed in \cite{akyurek2022learning}.

The following lemma states that the command pointed to by the program counter or the data from a location specified in the current command can be copied to the scratchpad for further computation. The location of the program counter is conventionally placed right below the contents of the scratchpad, but it can be changed arbitrarily. Keeping it in a specific location throughout the entire computation helps retain a good organization of the construction.

\begin{restatable}[\texttt{read}]{lemma}{read}\label{lem:read}
A transformer with one layer, one head, and width of $O(\log n +d)$, where $d$ is the dimension of the data vectors and $n$ is the length of the input, can read data/command vectors from the input to the scratchpad from the location pointed to by the position embedding vector in the scratchpad.
\end{restatable}
\begin{proof}

Consider a simplified input where the scratchpad only has one column, and we have positional encodings, denoted as $\emb_i$, that point to the location where data or commands should be copied from. In this case, the operation we want to perform is as follows:

\begin{align*}
\Input = 
     \left[\begin{array}{c|cccc}
        \zero&\vv_2&\cdots&\vv_i&\cdots\\
        \vv_1&\zero&\cdots&\zero&\cdots\\
        \emb_i& \zero&\cdots &\zero &\cdots\\
        \zero &\emb_2&\cdots & \emb_i&\cdots\\
        \zero&\zero&\cdots&\zero&\cdots\\
        1&0&\dots&0&\dots
    \end{array}\right] \xrightarrow{}  \left[\begin{array}{c|cccc}             \zero&\vv_2&\cdots&\vv_i&\cdots\\
        \vv_i&\zero&\cdots&\zero&\cdots\\
        \emb_i& \zero&\cdots &\zero &\cdots\\
        \zero &\emb_2&\cdots & \emb_i&\cdots\\
        \zero&\zero&\cdots&\zero&\cdots\\
        1&0&\dots&0&\dots
    \end{array}\right],
\end{align*}
which moves data/command embedding vector $\vv_i$ from the memory/command part of the input to the scratchpad. 
The first row contains the data to be read, the second row has the data written in the scratchpad, the third row contains the program counter, the fourth row contains the positional encodings, the fifth row is used by for temporary storage 
and the last row is just a bit that indicates whether the column is in the scratchpad or not. 

We use the following key and query matrices:
$
    \key =  \query= \begin{bmatrix}
          \zero &\zero &\id &\id &\zero&0
         \end{bmatrix},
$
so that the key and query become equal to
$
    \key\Input =\query\Input =\left[\begin{array}{ccccc}
             \emb_i &\emb_2&\cdots & \emb_i&\cdots
    \end{array}\right] \nonumber ,
$
and hence,
\begin{align}
(\key\Input)^\top \query\Input  
    &= \left[\begin{array}{cccc}
    \emb_i^\top\emb_i&\emb_i^\top\emb_2&\dots\\
    \emb_2^\top\emb_i&\emb_2^\top\emb_2&\dots\\
    \vdots&\vdots&\vdots\\
    \emb_i^\top\emb_i&\emb_i^\top\emb_2&\dots\\
    \vdots&\vdots&\vdots
    \end{array}\right]\nonumber
\end{align}
Recall that $\emb_i$ is a $\log(n)$-dimensional $\pm 1$ vector such that $\emb_i^T\emb_i = \log(n)$ and each $\emb_i^T\emb_j \leq \log(n) - 1$ for $j\neq i$. 
We show in the appendix 
that if we apply
the softmax with temperature $\lambda \geq  \log \frac{n^3}{\epsilon}$, we have $\softmax((\key\Input)^\top \query\Input)$ to be an $n\times n$ matrix of the following form
\begin{align*}
    \begin{bmatrix}
        \frac{1}{2}&0&0&\cdots&\frac{1}{2}&\cdots &0\\
        0&1&0&\cdots&0&\cdots &0\\
        0&0&1&\cdots&0&\cdots &0\\
        \vdots&\vdots&\vdots&\ddots&\vdots&\ddots&\vdots\\
        \frac{1}{2}&0&0&\cdots&\frac{1}{2}&\cdots &0\\
        \vdots&\vdots&\vdots&\ddots&\vdots&\ddots&\vdots\\
         0&0&0&\cdots&0&\cdots &1\\
    \end{bmatrix}  +\epsilon \rmM = \begin{bmatrix}
       \frac{\ve_1+\ve_i}{2}&\ve_2&\ve_3&\cdots& \frac{\ve_1+\ve_i}{2}&\cdots
    \end{bmatrix} + \epsilon \rmM,
\end{align*}
where $\ve_i$ is the $i$th column of the identity matrix, $\|\rmM\|\leq 1$, and $\epsilon$ is as defined in \cref{app:error}. 
For the purpose of the proof, we ignore the error term $\epsilon \rmM$, because it can be reduced arbitrarily by increasing the temperature (it can be made precisely equal to $0$, if we consider hardmax instead of softmax), and  overall does not limit us from deriving arbitrarily small error bounds. 

Next we set the output and value weight matrices as follows
\begin{equation}
    \val = \begin{bmatrix}
\zero&\zero&\zero&\zero&\zero&0\\
\zero&\zero&\zero&\zero&\zero&0\\
\zero&\zero&\zero&\zero&\zero&0\\
    \zero&\zero&\zero&\zero&\zero&0\\
    \id&\id&\zero&\zero&\zero&0\\
    \zero&\zero&\zero&\zero&\zero&0
    \end{bmatrix}.\nonumber
\end{equation}
 
Using this, the output of the head is 
\begin{equation}
\Input+\val\Input \softmax((\key\Input)^\top \query\Input)= \left[\begin{array}{c|cccc}
        \zero&\vv_2&\cdots&\vv_i&\cdots\\
        \vv_1&\zero&\cdots&\zero&\cdots\\
        \emb_i& \zero&\cdots &\zero &\cdots\\
        \zero &\emb_2&\cdots & \emb_i&\cdots\\
        \frac{\vv_1+\vv_i}{2}&\vv_2&\cdots&\frac{\vv_1+\vv_i}{2}&\cdots\\
        1&0&\dots&0&\dots
    \end{array}\right] \nonumber
\end{equation}
Each column above has the following form:
\begin{align*}
    \begin{bmatrix}
        \vv_{\text{orig}}^{0}\\
                \vv_{\text{orig}}^{1}\\
        \vv_{\text{orig}}\\
        \emb^{(0)}\\
        \emb^{(1)}\\
        \vv_{\text{new}}\\
        b
    \end{bmatrix},
\end{align*}
where $\vv_{\text{orig}}^{(0)}$ and $\vv_{\text{orig}}^{(1)}$ are the original value vectors (present in the top two row blocks) contained in that column, $\emb^{(0)}$ and $\emb^{(1)}$ are the corresponding embeddings of each column, $\vv_{\text{new}}$ is the new value, and $b$ is the bit indicating whether the column is part of the scratchpad or not.

The feedforward layers have the following form:
\begin{align*}
   \vv_{\text{orig}}^{(1)}&:=\vv_{\text{orig}}^{(1)} +\relu(C(b-1)\vOne +2\vv_{\text{new}}-2\vv_{\text{orig}}^{(1)})-\relu(C(b-1)\vOne -2\vv_{\text{new}}+2\vv_{\text{orig}}^{(1)})\\
   \vv_{\text{new}}&:=\vv_{\text{new}} - \relu(\vv_{\text{new}})+\relu(- \vv_{\text{new}})=\vZero,
\end{align*}
where $C$ is a large positive constant.
The first equation  is performing the operation of subtracting $\vv_{\text{new}}$ from $\vv_{\text{orig}}$ but only when the sum and difference of $C(b-1)\vOne$ and $\vv_{\text{new}}$ are positive, otherwise the subtraction does not occur.
The second equation  is resetting the value of $\vv_{\text{new}}$ to zero after it has been copied to $\vv_{\text{orig}}$, where $\relu(- \vv_{\text{new}})$ is the rectified linear unit (ReLU) applied to the negative of $\vv_{\text{new}}$.

It can be verified that the output of the feedforward layers would then be the desired result
\begin{align*}
    \Input =  \left[\begin{array}{c|cccc}             
    \zero&\vv_2&\cdots&\vv_i&\cdots\\
        \vv_i&\zero&\cdots&\zero&\cdots\\
        \emb_i& \zero&\cdots &\zero &\cdots\\
        \zero &\emb_2&\cdots & \emb_i&\cdots\\
        \zero&\zero&\cdots&\zero&\cdots\\
        1&0&\dots&0&\dots
    \end{array}\right].
\end{align*}
\end{proof}

The next lemma explains that the vector $\vv$ stored in the scratchpad can be copied to a designated location in memory, as specified within the scratchpad itself. This allows for the transfer of data from the scratchpad to a specific location in memory for further use or storage.

\begin{figure}[H]
  \centering
\includegraphics[scale =0.35]{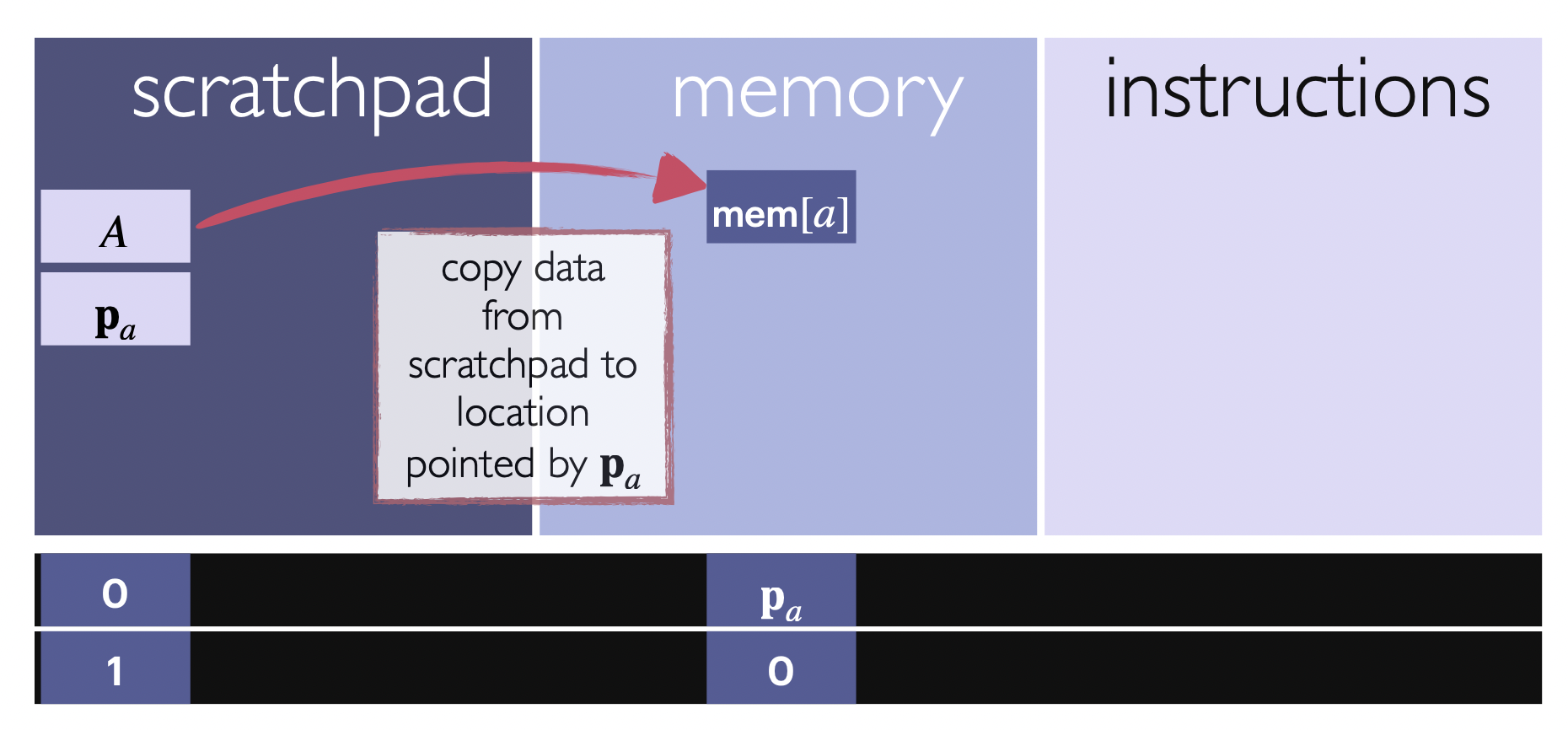}
  \caption{A sketch of the write operation. Arrows show data blocks being copied from the scratchpad to a designated location in the part of the input allocated for memory. Positional encodings are used for tracking the destination location and ensuring data is written at the correct memory location.}
  \label{fig:writeback}
\end{figure}

\begin{restatable}[\texttt{write}]{lemma}{paste}\label{lem:write}
A transformer network with a single layer, one head, and width $O(\log n + d)$, where $d$ is the dimension of the data vectors and $n$ is the length of the input, can effectively write a data vector stored in the scratchpad to a specific location in the input, as designated by a positional encoding vector in the scratchpad.
\end{restatable}

\begin{proof}
We want to achieve the following operation 
\begin{align*}
\Input = 
     \left[\begin{array}{c|cccc}
        \zero&\vv_2&\cdots&\vv_i&\cdots\\
        \vv_1&\zero&\cdots&\zero&\cdots\\
        \emb_i& \zero&\cdots &\zero &\cdots\\
        \zero &\emb_2&\cdots & \emb_i&\cdots\\
        \zero&\zero&\cdots&\zero&\cdots\\
        1&0&\dots&0&\dots
    \end{array}\right] \xrightarrow{}  \left[\begin{array}{c|cccc}             \zero&\vv_2&\cdots&\vv_1&\cdots\\
        \vv_1&\zero&\cdots&\zero&\cdots\\
        \emb_i& \zero&\cdots &\zero &\cdots\\
        \zero &\emb_2&\cdots & \emb_i&\cdots\\
        \zero&\zero&\cdots&\zero&\cdots\\
        1&0&\dots&0&\dots
    \end{array}\right],
\end{align*}
The construction for this is identical to the one for \texttt{read} (see the proof of \cref{lem:read}), except that the feedforward layers are outputting the following:
\begin{align*}
      \vv_{\text{orig}}^{(0)}&:=\vv_{\text{orig}}^{(0)} +\relu(-Cb\vOne +2\vv_{\text{new}}-2\vv_{\text{orig}}^{(0)})+\relu(-Cb\vOne -2\vv_{\text{new}}+2\vv_{\text{orig}}^{(0)})\\
   \vv_{\text{new}}&:=\vv_{\text{new}} - \relu(\vv_{\text{new}})+\relu(- \vv_{\text{new}})=\vZero,
\end{align*}
where $C$ is a large positive constant.
The first equation updates the value of a vector $\vv_{\text{orig}}$ in memory with the value of a vector $\vv_{\text{new}}$ from the scratchpad. 
The second equation is resetting the new vector in the scratchpad to zero. 
It can be verified that the output of the feedforward layers would be 
\begin{align*}
 \Input =  \left[\begin{array}{c|cccc}             
    \zero&\vv_2&\cdots&\vv_1&\cdots\\
        \vv_1&\zero&\cdots&\zero&\cdots\\
        \emb_i& \zero&\cdots &\zero &\cdots\\
        \zero &\emb_2&\cdots & \emb_i&\cdots\\
        \zero&\zero&\cdots&\zero&\cdots\\
        1&0&\dots&0&\dots
    \end{array}\right].
\end{align*}
\end{proof}

\subsection{\texttt{if} $\langle condition\rangle$ \texttt{then goto} $\langle instruction\rangle$: Conditional branching}\label{ss:goto}

In this subsection, we will implement a conditional branching instruction that evaluates a condition and sets the program counter to a specified location if the condition is true, or increments the program counter by 1 if the condition is false. The form of the command is as follows: \texttt{if} mem$[a]\leq 0$, \texttt{then goto} $i$, where mem$[a]$ is a value of some location in the memory part of the input sequence.  This command has two parts: evaluating the inequality and modifying the program counter accordingly.

The first thing we do is read from mem$[a]$, as described in the previous subsection. Then, we evaluate the inequality. Let us say that ``flag'' is the truth value of the inequality. Since we assume that for such conditional branching command, mem$[a]$ contains an integer, the following ReLU network can be used to compute the flag:
\begin{align}
    \text{flag} = 1-\relu(\texttt{mem}[a]) + \relu(\texttt{mem}[a]-1).\label{eq:compute_flag}
\end{align}

In \cref{sec:oisc}, we consider $\texttt{mem}[a]$ to be vectors contain the binary $\pm 1$ representation of integers. There we use 2's complement convention to represent negative integers. Let the vector be $[b_N\;\dots\;b_1]$, where $b_N$ is the most significant bit and $b_1$ the least significant. As we explain in that section, the sign of $b_N$ indicates whether the integer is negative or positive (The number is negative if $b_N=+1$ and non-negative otherwise). Hence, the flag is 1 if $b_N=+1$ or if all the bits are $-1$ (which is the case when $\texttt{mem}[a]$ represents the integer 0).
\begin{align}\label{eq:compute_flag_oisc}
    \text{flag} = \relu(b_N) + \relu\left(1+N-\sum_{i=1}^N b_i\right)  
\end{align}

Let the current Program Counter be $\emb_\pointer$, which points to a given command. Thus, if flag is $1$, we want the program counter to ``jump'' and become $\emb_i$, else if flag is $0$ the program counter will be incremented by one, and set to be $\emb_{\pointer+1}$.

Consider that the simplified input currently has the following scratchpad
\begin{equation}
 \begin{bmatrix}
        * & * &\hdots &* &*\\
        \mathrm{flag} & \zero& \hdots&\zero  &\zero\\
       \emb_\pointer& \zero&\hdots &\zero  &\zero \\
        \emb_i& \zero& \hdots& \zero  &\zero\\
    \end{bmatrix}, \nonumber
\end{equation}
where $'*'$ are inconsequential values.
The incremented pointer, $\emb_{\pointer+1}$, %
can be computed using the pointer incrementing operation that we described in the Subsection \ref{ss:increaseCounter}, using one feedforward layer of \eqref{eq:relulayer}.Then, 
\begin{align*}
    \emb_{\text{next}} = 2\relu(\emb_{\pointer+1} - \vOne \mathrm{flag}) + 2\relu(\emb_{i} - \vOne(1 -\mathrm{flag}))-1, %
\end{align*} 
where $\vOne$ is the all ones vector.
Notice that we can implement this with just the feed forward layers of \cref{eq:relulayer}. To account for the residual connection we can add the  expression $ - \relu(\emb_\pointer) + \relu(-\emb_\pointer)$ in the equation above. 

Hence, this entire operation requires 3 feed forward layers of \cref{eq:relulayer}, and hence 2 transformer layers. Note that to ensure that the attention layer of the transformer do not modify the input, we simply set the $\val$  matrix to zero in \eqref{eq:att}.

\label{sec:lego}

\section{Emulating a Generalized One-instruction Set Computer}\label{ss:oneLine}

\subsection{A \texttt{SUBLEQ} Transformer}
\cite{mavaddat1988urisc} showed that there exists an instruction such that any computer program can be translated to a program consisting of instantiation of this single instructions. 
A variant of such an instruction is  \texttt{SUBLEQ}, where different registers, or memory locations are accessed. The way that \texttt{SUBLEQ} works is simple. It accesses two registers in memory, takes the difference of their contents and stores it back to one of the registers, and then if the result is negative it jumps to a different predefined line of code, or continues on the next instruction from the current line of code.\footnote{This version of the \texttt{SUBLEQ} instruction is a slightly restricted version of the original instruction; here we separate the memory / registers from the instructions. We show that this restriction does not make our version computationally less powerful by proving in \cref{app:subleq} that our version is also Turing Complete.}
A computer that is built to execute \texttt{SUBLEQ} programs is called an One-Instruction Set Computer, and is a universal computer, \ie it is {\it Turing Complete}, if given access to infinite memory.

\begin{algorithm}[h]
\caption{\texttt{SUBLEQ}($a$, $b$, $c$)}\label{alg:cap}
\begin{algorithmic}[1]
    \State {\texttt{mem}[$b$] = 
    \texttt{mem}[$b$] - \texttt{mem}[$a$]}   
    \If {\texttt{mem}[$b$] $\leq$ $0$} \State{\texttt{goto} instruction $c$}
    \Else { \texttt{goto} next instruction}
    \EndIf
\end{algorithmic}
\end{algorithm}

The following describes the construction of a looped transformer that can execute a program written in a specific set of instructions. The transformer keeps track of the lines of code, memory locations, and a program counter, using the memory part of the input as memory registers and the command part as lines of code/instructions. The scratchpad is used to record the additions and pointers involved in each instruction, and the read, write, and conditional branch operations are utilized.

\begin{figure}
    \centering
    \includegraphics[ scale = 0.4]{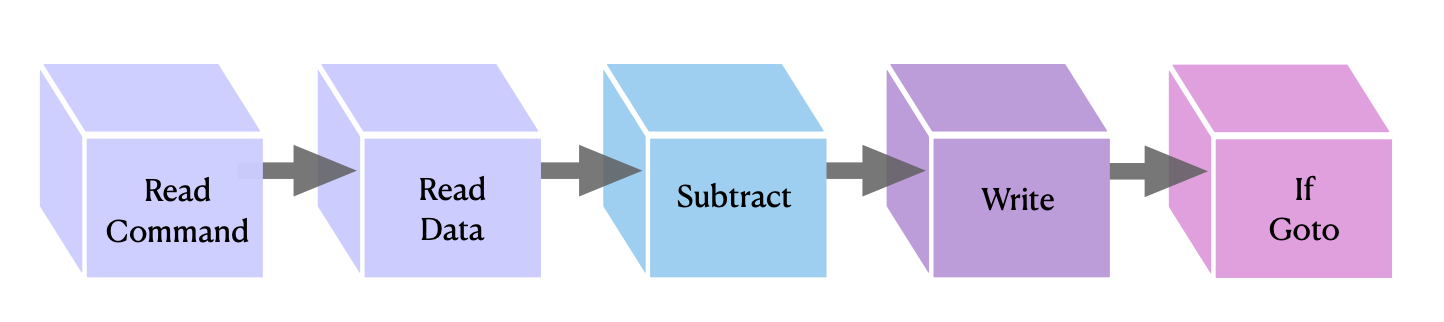}
    \caption{Graphical representation of the building blocks necessary to implement the OISC instruction. The first two blocks transfer the data/command to the scratchpad, the second and third implement the substraction and store the result, while the last one implements the if goto command that completes the instruction.}
    \label{fig:OISC}
\end{figure}
\begin{lemma}
   There exists a looped transformer architecture that can run SUBLEQ programs. This architecture has nine layers, two heads, and a width of $O(\log(n) + N)$, where $n$ is the length of the input sequence that is proportional to the length of the program and memory used by the emulated OISC, and $N$ is the number of bits we use to store each integer. The integers are considered to be in the range $[-2^{N-1}+1,2^{N-1}-1]$ 
   \label{lem:OISC}
\end{lemma}

Before we present our construction some observations are in place. 

\paragraph{The importance of loops.} 
The use of a loop outside the transformer is crucial as it allows the computer to keep track of the program counter and execute the instructions in the correct order. Without this loop, the size of the transformer would have to scale with the number of lines of code, making the implementation impractical. Note that the overall complexity of running a \texttt{SUBLEQ} program is going to scale with the number of lines of code, which is to be expected given standard complexity theoretic assumptions on the circuit depth of functions. Note however that the depth of the looped transfromer itself does not scale with the size of the program. 

\paragraph{Can we avoid the logarithmic width scaling?} 
Finally note, that the width of the transformer scales logarithmically with the length of the program, and memory used. This is a side-effect of the bit-complexity of our positional encodings, and could be overcome by considering higher bit-complexity.

\paragraph{OISC as a basis for a more flexible attention-based computer.}
The following construction describes an implementation of a fully functioning one-instruction set computer (OISC) using a transformer architecture. The memory stores integers and the instructions are executed in a sequential manner. The key to this construction is the reverse engineering of the attention mechanism to perform read/write operations and taking full advantage of each piece of the transformer architecture, including the feedforward layers. This implementation serves as the foundation for a more general attention-based computer presented in the next subsection, where the subtraction of two contents of memory can be replaced with a general function, allowing for the implementation of arbitrary iterative algorithms.

\begin{proof}[Proof of Lemma~\ref{lem:OISC}]
Looking at \cref{alg:cap}, note that each instruction can be specified by just 3 indices, $a,b,$ and $c$. 
Since we use binary representation of indices to form positional encodings and pointers, each of these indices can be represented by a $\log n$ dimensional vector.
We represent each instruction by simply concatenating these embedding vectors to form a $3\log n$ dimensional vector as follows: $$\command = \begin{bmatrix}\emb_{a}\\ \emb_{b}\\ \emb_{c}\end{bmatrix}.$$
The input then takes the following form:
\vspace{1em}
\begin{equation}
\Input= \left[
\begin{array}{cc|c|ccccc}
\zero & \zero &\zero & \y\tikznode{com1}{$\command_{s+m+1}$}& \y\tikznode{com2}{$\command_{s+m+2}$}&\y{\hdots} &\y\tikznode{com3}{$\command_{n-1} $}&\y \tikznode{ceof}{$\command_{\texttt{EOF}}$}\\  
\zero &\zero &\yy \tikznode{mem}{$\rmM$} &\zero &\zero&\hdots&\zero&\zero \\
\scr\tikznode{scr}{$\zero$}&\scr{\zero}  &\zero &\zero &\zero&\hdots&\zero&\zero \\
\pc\tikznode{point}{$\emb_{\pointer}$}&\zero &\zero &\zero &\zero&\hdots&\zero&\zero\\
\rowcolor{oldlavander}\tikznode{ee}{$\zero$}&\emb_{2:s} &\emb_{s+1:s+m} & \emb_{s+m+1 } &\emb_{s+m+2}&\hdots &\emb_{n-1} &\emb_n\\ 
\rowcolor{oldlavander} \tikznode{ee1}{1}&1_{2:s}&0_{s+1:s+m}&0_{s+m+1}&0_{s+m+2}&\cdots&0_{n-1}&0_{n}
\end{array}
\right]
\end{equation}
\vspace{1em}

\begin{tikzpicture}[remember picture,overlay,cyan,rounded corners]
   \draw[<-] (com1)
    |- +(4.5,1.1)%
    coordinate (com)%
    node[right]{Commands};
  \draw[<-] (com2)
    |- (com)%
    coordinate (comm);
   \draw[<-] (com3)
    |- (comm);
    \draw[<-] (ceof)
    |- +(0.5,0.6)%
    node[right]{EOF};
     \draw[<-,color=teal] (mem)
    -- +(-0.4,0)%
    |- +(-1,+1)%
    node[left]{Block of memory};
     \draw[<-,color = violet] (scr)
     -- +(-0.4,0) %
    |- +(-1,+0.8)%
    node[left]{Scratchpad};
    \draw[<-,color=olive] (point)
     -- +(-0.4,0) %
    |- +(-0.8,+0.8)%
    node[left]{Program Counter};
    \draw[<-,color=gray] (ee)
     -- +(-0.4,0) %
    |- +(-1,+0.2)%
    coordinate (ee2)
    node[left]{Encodings};
     \draw[<-,color=gray] (ee1)
     -- +(-1,0)
    |- +(0,-1)
     node[right]{Indicator of the scratchpad};
\end{tikzpicture}

where $\command_i\in \R^{3\log(n)}$, $\rmM\in\R^{N\times m}$ and  $\Input\in\R^{(8\log(n)+3N+1)\times n}$. 
The first $s$ columns constitute the scratchpad, the next $m$ constitute the memory section, and the last $n-m-s$ columns contain the instructions. 

The program counter, $\emb_\pointer$ points to the next instruction that is to be executed, and hence it is initialized to the first instruction as $\emb_\pointer := \emb_{s+m+1}$.
The contents of the memory section are $N$ dimensional $\pm 1$ binary vectors which represent the corresponding integers. We follow the 2's complement convention to represent the integers, described as follows. Let's say the bits representing an integer are $b_{N},\dots,b_1$, with $b_N$ being the most significant bit. Then,
\begin{enumerate}
    \item If $b_N = -1$, then the integer is considered positive with the value $\sum_{i=1}^{N-1}2^{i-1}\frac{b_i+1}{2}$.
    \item If $b_N = +1$, then the integer is considered negative with the value $-2^{N-1}+\sum_{i=1}^{N-1}2^{i-1}\frac{b_i+1}{2}$.
\end{enumerate}

\paragraph{Step 1 - Read the instruction $\command_\pointer$.} The first thing to do is to read and copy the instruction pointed to by $\emb_\pointer$ in the scratchpad. The current instruction is located at column index $\pointer$, and is pointed to by the current program counter $\emb_\pointer$. The instruction, $\command_\pointer$ consists of three pointers, each of length $\log n$.
In particular we copy the elements at the location $(1:3\log(n),\pointer)$ to the location $(3\log(n)+4: 6\log(n)+3,1)$. %
This can be done using the \texttt{read} operation as described in \cref{ss:copy}. Hence, after this operation, the input looks as follows:
\begin{align*}
    \Input &=  \left[\begin{array}{cc|c|ccccc}
        \zero &\zero & \zero & \command_1 & \command_2&\hdots &\command_{n-m-s} &\command_{\texttt{EOF}}\\
        \zero &\zero &\rmM &\zero &\zero&\hdots&\zero &\zero \\
         \zero &\zero &\zero &\zero &\zero &\hdots &\zero &\zero \\
        \zero &\zero &\zero &\zero &\zero &\hdots &\zero &\zero \\
        \command_{\pointer}&\zero &\zero &\zero &\zero&\hdots&\zero &\zero \\
         \emb_{\pointer}&\zero  &\zero &\zero &\zero&\hdots&\zero &\zero\\
        \zero&\emb_{2:s} &\emb_{s+1:s+m} & \emb_{s+m+1 } &\emb_{s+m+2}&\hdots &\emb_{n-1} &\emb_n\\
        {1}&1_{2:s}&0_{s+1:s+m}&0_{s+m+1}&0_{s+m+2}&\hdots&0_{n-1}&0_{n}
        \end{array}\right]
\end{align*}
\begin{align*}
    &=  \left[\begin{array}{cc|c|ccccc}
      \zero &\zero & \zero & \command_1 & \command_2&\hdots &\command_{n-m-s-1}  &\command_{\texttt{EOF}}\\
        \zero &\zero &\rmM &\zero &\zero&\hdots&\zero &\zero\\
\zero &\zero &\zero &\zero &\zero &\hdots &\zero &\zero \\
        \zero &\zero &\zero &\zero &\zero &\hdots &\zero &\zero \\
        \emb_a&\zero &\zero &\zero &\zero&\hdots&\zero &\zero\\
        \emb_b&\zero &\zero &\zero &\zero&\hdots&\zero &\zero\\
        \emb_c&\zero &\zero &\zero &\zero&\hdots&\zero &\zero\\
        \emb_{\pointer}&\zero  &\zero &\zero &\zero&\hdots&\zero &\zero\\
        \zero&\emb_{2:s} &\emb_{s+1:s+m} & \emb_{s+m+1 } &\emb_{s+m+2}&\hdots &\emb_{n-1} &\emb_n\\        
        {1}&1_{2:s}&0_{s+1:s+m}&0_{s+m+1}&0_{s+m+2}&\hdots&0_{n-1}&0_{n}
    \end{array}\right]
\end{align*}
This step can be done in one layer.

\paragraph{Step 2 - Read the data required by the instruction.} We need to read the data  that  the columns $a,b$ contain. To do so, we again use the \texttt{read} operation on the pointers $\emb_a,\emb_b$. Note that we need two heads for this operation, one each for reading $a$ and $b$.
The resulting output sequence looks like
\begin{equation}\label{eq:input1}
    \Input =  \left[\begin{array}{cc|c|ccccc}
        \zero &\zero & \zero & \command_1 & \command_2&\hdots &\command_{n-m-s-1} &\command_{\texttt{EOF}}\\
        \zero &\zero &\rmM &\zero &\zero&\hdots&\zero&\zero\\
        \mathrm{mem}[a] &\zero &\zero &\zero &\zero &\hdots &\zero&\zero\\
         \mathrm{mem}[b] &\zero &\zero &\zero &\zero &\hdots &\zero&\zero\\
        \emb_a&\zero &\zero &\zero &\zero&\hdots&\zero&\zero\\
        \emb_b&\zero &\zero &\zero &\zero&\hdots&\zero&\zero\\
        \emb_c&\zero &\zero &\zero &\zero&\hdots&\zero&\zero\\        
        \emb_{\pointer}&\zero  &\zero &\zero &\zero&\hdots&\zero&\zero\\
        \zero&\emb_{2:s} &\emb_{s+1:s+m} & \emb_{s+m+1 } &\emb_{s+m+2}&\hdots &\emb_{n-1}&\emb_n\\
        {1}&1_{2:s}&0_{s+1:s+m}&0_{s+m+1}&0_{s+m+2}&\hdots&0_{n-1}&0_{n}
    \end{array}\right].
\end{equation}
This step can be done in one layer.

\paragraph{Step 3 - Perform subtraction.}
Let $\vx$ denote a column of the input $\Input$. Let it have the following structure:
\begin{align*}
    \vx = \begin{bmatrix}
        *\\
        *\\
        \vb_r\\
        \vb_s\\
        *\\
        *\\
        *\\
        *\\
        *
    \end{bmatrix},
\end{align*}
where each entry above represents the corresponding column element of the matrix $\Input$ in \eqref{eq:input1}. Thus, $\vb_r=\texttt{mem}[a], \vb_s = \texttt{mem}[b]$ for the first column, and $\vb_r=\vb_s=\zero$ otherwise. 

Hence, to perform $\vb_{s-r}$, we first need to compute the binary representation of $-r$, which is $\vb_{-r}$, and then simply add it to $\vb_s$. To compute $\vb_{-r}$, which is the 2's complement of $\vb_r$, we just need to flip the bits of $\vb_r$ and add 1.
Bit flipping a $\pm 1$ bit can be done with a neuron simply as $b_{\text{flipped}}=2*\relu(-b)-1$.
For adding 1, we can use \cref{app:increase}.
Hence, each of these operations can be done using 1 ReLU layer of width $O( N)$, and so we need 2 transformer layers to perform this (Here we make the intermediate attention layers become the identity mapping by setting their value matrices to $\zero$).
Finally, we need one more ReLU layer to add $\vb_s$ to $\vb_{-r}$, hence bringing the total to 3 transformer layers.

This results in the following:
\begin{equation*}
     \Input =  \left[\begin{array}{cc|c|ccccc}
        \zero &\zero & \zero & \command_1 & \command_2&\hdots &\command_{n-m-s-1}&\command_{\texttt{EOF}}\\
        \zero &\zero &\rmM &\zero &\zero&\hdots&\zero&\zero\\
         \zero&\zero &\zero &\zero &\zero&\hdots&\zero&\zero\\
         \mathrm{mem}[b]-\mathrm{mem}[a] &\zero &\zero &\zero &\zero&\hdots&\zero&\zero\\
        \emb_a&\zero &\zero &\zero &\zero&\hdots&\zero&\zero\\
        \emb_b&\zero &\zero &\zero &\zero&\hdots&\zero&\zero\\
        \emb_c&\zero &\zero &\zero &\zero&\hdots&\zero&\zero\\
        \emb_{\pointer}&\zero  &\zero &\zero &\zero&\hdots&\zero&\zero\\
        \zero&\emb_{2:s} &\emb_{s+1:s+m} & \emb_{s+m+1 } &\emb_{s+m+2}&\hdots &\emb_{n-1}&\emb_n\\
        {1}&1_{2:s}&0_{s+1:s+m}&0_{s+m+1}&0_{s+m+2}&\hdots&0_{n-1}&0_{n}
    \end{array}\right]%
\end{equation*} 
Note that since this can be done in the feedforward layers of the previous step, this does not require an additional layer.

 \paragraph{Step 4 -  Write the result back to memory.} 
 Writing $\mathrm{mem}[b] - \mathrm{mem}[a]$ back to location $b$ can be done using the pointer $\emb_b$ and the set of embeddings and applying the \texttt{write} operation described in  \cref{ss:copy}.
This operation requires one layer.

\paragraph{Step 5 - Conditional branching.} 
We first use \cref{eq:compute_flag_oisc} as described in \cref{ss:goto} to create the flag, which is $1$ if $\mathrm{mem}[b] - \mathrm{mem}[a] \leq 0$ and $0$ otherwise.
This can be done using the \cref{eq:relulayer} of the transformer. 
Thus, we have
\begin{equation}
     \Input =  \left[\begin{array}{cc|c|ccccc}
        \zero &\zero & \zero & \command_1 & \command_2&\hdots &\command_{n-m-s-1}&\command_{\texttt{EOF}}\\
        \zero &\zero &\rmM &\zero &\zero&\hdots&\zero&\zero\\
         \zero&\zero &\zero &\zero &\zero&\hdots&\zero&\zero\\
         \text{flag} &0 &0 &0 &0 &\hdots &0&0\\
        \emb_a&\zero &\zero &\zero &\zero&\hdots&\zero&\zero\\
        \emb_b&\zero &\zero &\zero &\zero&\hdots&\zero&\zero\\
        \emb_c&\zero &\zero &\zero &\zero&\hdots&\zero&\zero\\
        \emb_{\pointer}&\zero  &\zero &\zero &\zero&\hdots&\zero&\zero\\
        \zero&\emb_{2:s} &\emb_{s+1:s+m} & \emb_{s+m+1 } &\emb_{s+m+2}&\hdots &\emb_{n-1}&\emb_n\\
        {1}&1_{2:s}&0_{s+1:s+m}&0_{s+m+1}&0_{s+m+2}&\hdots&0_{n-1}&0_{n}
    \end{array}\right]\label{eq:proof_input_1}
\end{equation} 

This operation requires one layer.

Next we use the construction described in \cref{ss:goto} to choose, depending on the value of the flag, whether we want to increment the current program counter or we want to jump in the command $c$. Similar to \cref{ss:goto}, this step needs 2 layers of transformers.

\paragraph{Step 6 - Error Correction.} Note that some of the steps above we incur some error while reading and writing due to the fact that we are using softmax instead of hardmax. This error can be made arbitrarily small by increasing the temperature of the softmax. In this step, we push the error down to zero. Note that all the elements of $\Input$ can only be one of $\{-1,0,1\}$, with some additive error from reads and writes as explained before. Assume that the temperature is set high enough that the error is at most $\epsilon < 0.5$. Then, a noisy bit $b$ can be fixed using the following ReLU:
\begin{align*}
    b_{\text{noiseless}} &= \frac{1}{1-2\epsilon} (\relu(b+1-\epsilon) - \relu(b+\epsilon)) \\
    &\quad + \frac{1}{1-2\epsilon} (\relu(b-\epsilon) - \relu(b -1 +\epsilon)) - 1.
\end{align*}
This operation can be done with a single layer of transformer.

\paragraph{Step 7 - Program Termination.} The special command $\command_{\texttt{EOF}}$ is used to signal the end of a program to the transformer. This command is made up of three encodings: $\emb_{s+1}$, $\emb_{s+2}$, and $\emb_{n}$. The first encoding, $\emb_{s+1}$, points to the first entry in the memory, which we hard-code to contain the value $0$. The second encoding, $\emb_{s+2}$, points to the second entry in the memory, which is hard-codeded to contain the value $-1$. The third encoding, $\emb_{n}$, points to itself, signaling the end of the program and preventing further execution of commands. 
Hence, on executing this command, the next command pointer is set to point to this command again.
This ensures that the transformer maintains the final state of the input.
\begin{itemize}
    \item For this, we ensure that the last instruction in each program is $\command_{\texttt{EOF}}$, and that $\texttt{mem}[s+1]=0$ and $\texttt{mem}[s+2]=-1$.
    \item For this case $a=s+1$, $b=s+2$, and $c=n$.
    \item The memory is updated with the value $\mathrm{mem}[b] = \mathrm{mem}[b] - \mathrm{mem}[a]$. Since $\mathrm{mem}[a] = 0$ here, the memory remains unchanged.
    \item  Since $\mathrm{mem}[b]\leq  0$ here, the branch is always true and thus the pointer for the next instruction is again set to point to $\command_{\texttt{EOF}}$.
\end{itemize}
\end{proof}

\label{sec:oisc}

\subsection{\texttt{FLEQ}: A More Flexible Attention-based Computer}\label{sec:fleq}

In this section, we introduce \texttt{FLEQ}, a generalization of \texttt{SUBLEQ} that defines a more flexible reduced-instruction set computer. This implied set of additional instructions is based on a more advanced version of \texttt{SUBLEQ} that allows for the implementation of multiple functions within the same transformer network. This is achieved by generalizing the previous OISC construction to include not just addition of registers, but any function from a set of $M$ predefined functions implementable by a transformer network. In the following, we use the term \texttt{FLEQ} to refer interchangably to the instruction, the language, and the attention-based computer it defines. 

The design of \texttt{FLEQ} allows for the implementation of complex and sophisticated algorithms by generating more general functions beyond simple subtraction, such as matrix multiplication, computation of square roots, activation functions, etc. 
This not only increases the flexibility of the system, but  also makes it possible to implement %
nonlinear computations, linear algebra calculations, and  iterative optimization algorithms for in-context learning while containing the length of the corresponding programs. 

\begin{definition}\label{def:tf_func_block}
    Let $\cT_i$ be a transformer network of the form \eqref{eq:TF} with $l_i$-layers, $h_i$-heads and dimensionality $r$.  %
    We call this a \textbf{``transformer-based function block''} if it implements a function $f(\rmA, \rmB)$ where the input and output sequence format is assumed to be the following: $\rmA\in\bR^{d_h\times d_w}$ is assumed to be provided in the first set of $d$ columns (columns $1$ to $d$) and $\rmB\in\bR^{d_h\times d_w}$ the second set of $d$ columns (columns $d+1$ to $2d$); after passing the input through the $l_i$ layers, the output of $f(\rmA, \rmB)\in\bR^{d_h\times d_w}$ is stored in the third $d$ columns (columns $2d+1$ to $3d$), where $d$ is the maximum size that the input could have and it is a constant that we determine. Note that $d_h,d_w\leq d$. Finally, the sequence length of the block is $s\geq 3d$.  Similarly to $d$, $s$ is a predetermined constant.%
\end{definition}
The parameters $\rmA, \rmB$ %
can be scalars, vectors or matrices as long as they can fit within a $d\times d$ matrix. Hence, the above definition is minimally restrictive, with the only main constraint being the input and output locations. More details about the input and output requirements will be explained towards the end of this subsection.

\begin{theorem}\label{thm:unified}
    Given $M$ different transformer-based function blocks $\cT_1,\cdots, \cT_M$, there exists a transformer $\cT$ of the form \eqref{eq:TF} with number of layers $9+ \max \{l_1,\cdots, l_M\}$, a number of $\sum_{i=1}^M h_i$ heads , and dimensionality $O(Md + \log n)$ such that running it recurrently $T$ times can run $T$ instructions of any program where each instruction is {\normalfont $\texttt{FLEQ}(a,b,c, m, \text{flag}, p,d_h,d_w)$}, and executes the following:
    {\normalfont
    \begin{equation}
        \texttt{mem}[c] = f_m (\texttt{mem}[a], \texttt{mem}[b])\quad; \quad\text{if }\texttt{mem}[\text{flag}]\leq 0\text{ goto instruction }p
    \end{equation}
    }Here 
    $n$ is the total length of the program and we assume that $\texttt{mem}[\text{flag}]$ is an integer. The parameters $d_h,d_w$ are explained in \cref{rem:width} below. 
\end{theorem}
\begin{remark}\label{rem:width}
Note that, the transformer $\mathcal{T}$ contains $M$ transformer-based function blocks and each one may use different input parameters. We thus define with $d$ the max length that each of the parameters $\rmA, \rmB, \rmC$ (stored in locations $a,b,c$) as in \cref{def:tf_func_block} can have; this is a global constant and it is fixed for all the different instances that we can create. Now,   $d_h,d_w$ refer to the maximum dimension that the parameters can have in a specific instance of the transformer $\mathcal{T}$; the rest of the columns $d-d_w$ and rows $d-d_h$ are set to zero. %
\end{remark}
The proof of this theorem can be found in \cref{app:unified}. Below we explain some of our design choices.

\paragraph{Execution cycle of the unified attention-based computer.} 
In each iteration of the looped transformer, one instruction is fetched from the set of instructions in the input according to the program counter. The instruction is then copied to the scratchpad. Depending on the function to be implemented, a different function block location is used to locally record the results of that function. Once the result is calculated, it is copied back to a specified memory location provided by the instruction. The execution cycle is similar to the one-instruction set computer (OISC) in the previous section, with the main difference being that for each instruction, we can choose from a pre-selected list of functions that take inputs in the form of arbitrary arrays of numbers, such as matrices, vectors, and scalars.

\paragraph{The format of the input sequence.} In Fig.~\ref{fig:unifiedinput}, we illustrate the input $\Input$ to our looped transformer, which can execute a program written as a series of FLEQ instructions. Note that $\Input$ is divided into three sections: Scratchpad, Memory, and Instructions. As in the left bottom part of Fig.~\ref{fig:unifiedinput}, we allocate a separate part of the scratchpad for each of the $M$ functions that are internally implemented by the transformer. For example, if we have matrix multiplication and element-wise square root as two functions, we would allocate a different function block for each one.

\begin{figure}[H]
    \centering
    \includegraphics[scale = 0.4]{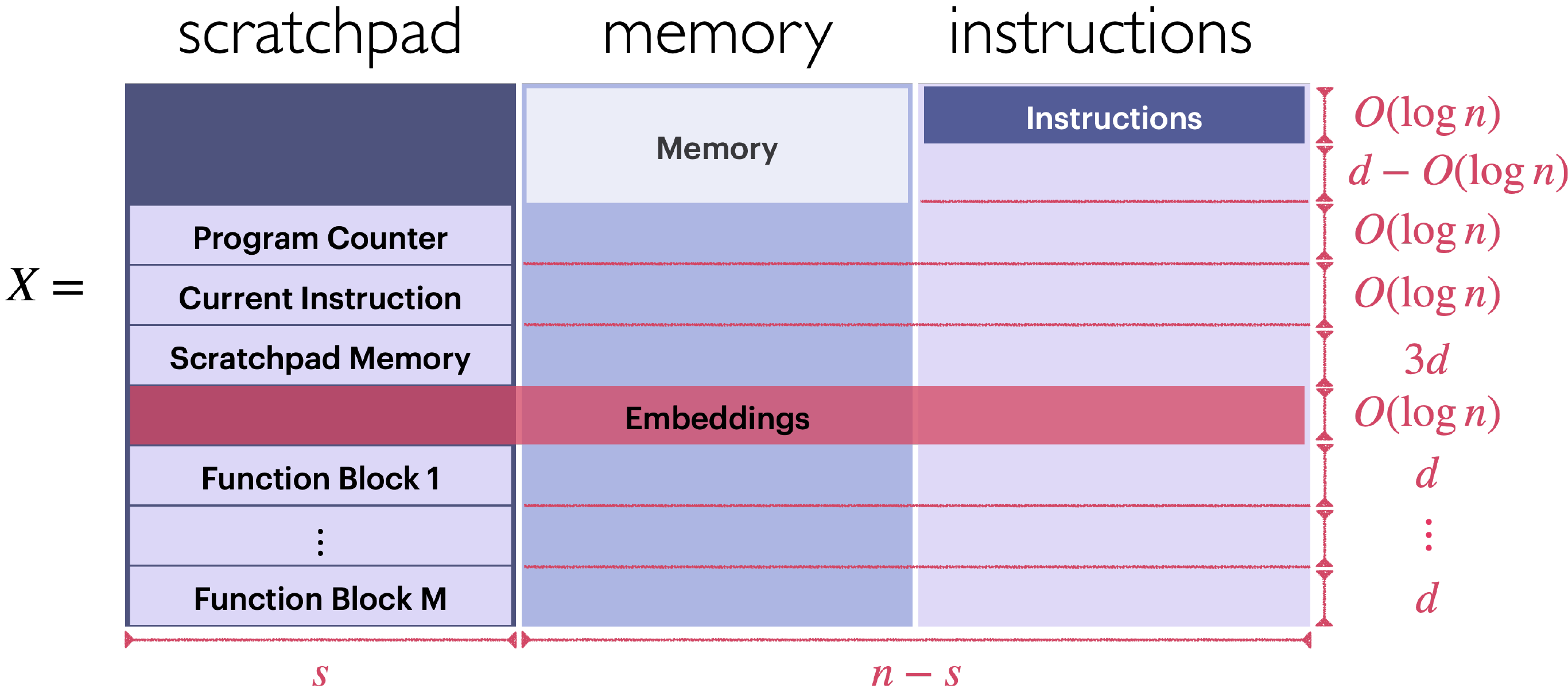}
    \caption{The structure of input $\Input$, to execute FLEQ commands.}
    \label{fig:unifiedinput}
\end{figure}%
 This design may not be the most efficient, but our goal is to demonstrate the possibilities of looped transformers. Additionally, since the number of different functions is typically small in the applications we have in mind, the design does not significantly increase in size. The choice to reserve different function blocks for each predefined function is for convenience, as it allows for separate treatment of functions without worrying about potentially overlapping results. We believe that a design with a single function block is feasible, but it would significantly complicate the rest of the transformer construction.

\paragraph{Instruction format.} 
The instruction in \cref{thm:unified} is essentially a composition of the following two components: %
the function call to $f_m$ and the conditional branching (if ... goto ...). 
The instruction, located at the top right side of Fig.~\ref{fig:unifiedinput} contains the following components:

\begin{equation}\label{eqn:fleq_instruction}
  \left[\begin{array}{c}
    \y{\emb_a} \\
    \y\tikznode{a2}{$\emb_b$} \\
    \yy\tikznode{out}{$\emb_c$} \\
    \scr\tikznode{function}{$\emb_m$} \\
    \ol\tikznode{flag}{$\emb_{\text{flag}}$} \\
    \pc\tikznode{p}{$\emb_p$}\\
        \lo\tikznode{dim}{$d_h$} \\
        \lo\tikznode{dim}{$d_w$} \\
    \end{array}\right]
\end{equation}

\begin{tikzpicture}[remember picture,overlay,cyan,rounded corners]
  \draw[<-] (a2)
    -- +(-1.5,0.0)
    node[left]{Pointers to parameters of $f_m$}; 
    \draw[<-,color = teal] (out)
    -- +(-1.5,0.0)
    node[left]{Position to write result};
    \draw[<-,color =violet] (function)
    -- +(-1.5,0.0)
    node[left]{Pointer to function block};
    \draw[<-,color = gray] (flag)
    -- +(1.0,+0.0)
    node[right]{Position of flag};
    \draw[<-,color = olive] (p)
    -- +(1.0,0.0)
    node[right]{Next instruction};
  \draw[<-,color = darkpastelblue] (dim)
    -- +(1.0,-0.0)
    node[right]{Dimensions of the inputs and the output};
\end{tikzpicture}

The goal of each positional encoding vector in~
\cref{eqn:fleq_instruction} is to point to the corresponding space of the input where each component required by the instruction is located. To be specific,
$\emb_a$ and $\emb_b$ point to the locations that the inputs $a$ and $b$ are located, $\emb_c$ points to the location to which we will record the final result of the function $f_m$. Similarly, $\emb_m$ points to the function block in the scratchpad that the intermediate computations required for $f_m$ are recording, $\emb_{\text{flag}}$ points to the variable that we check if it is non-positive (the result is used for conditional branching), and $\emb_p$ points to the address of the line of code that we would jump if the variable in pointed by $\emb_{\text{flag}}$ is non-positive. 

\paragraph{Execute a function; Jump to command.}
Recall that the first four parameters ($a,b,c,m$) of FLEQ, as well as the last two $(d_h,d_w)$ are related to the implementation of the function block, while the other two ($\text{flag}, p$) are related with the conditional branching. Since there is no overlap between the two components of each instruction,  it is
possible to use each of these components independently. %
 By having a fixed location $\text{flag}_0$ where $\texttt{mem}[\text{flag}_0]$ is always set to $1$, we can have the simpler command $\texttt{FLEQ}(a,b,c, m, \text{flag}_0, p,d_h,d_w)$ which implements $$\texttt{mem}[c] = f_m (\texttt{mem}[a], \texttt{mem}[b]).$$
Further, by having fixed locations $a_0,b_0, c_0$ which are not used elsewhere in the program, and hence inconsequential, we can have the simpler command $\texttt{FLEQ}(a_0,b_0,c_0, m, \text{flag}, p,d_h,d_w)$ which implements $$\text{if }\texttt{mem}[\text{flag}]\leq 0\text{ goto instruction }p.$$

Using this, we get the following corollary:
\begin{corollary}
The Unified Attention Based Computer presented in \cref{thm:unified} can run programs where each instruction can be \textbf{either} of the following two simple instructions:
\begin{itemize}
    \item {\normalfont
    $\texttt{mem}[c] = f_m (\texttt{mem}[a], \texttt{mem}[b])$
    }
    \item {\normalfont
    $\text{if }\texttt{mem}[\text{flag}]\leq 0\text{ goto instruction }p$
    }
\end{itemize}
\end{corollary}

\paragraph{Format of Transformer-Based Function Blocks.}
Recall that each function block is located at the bottom left part of the input $\Input$, as shown in Fig.~\ref{fig:unifiedinput}. Each transformer-based function block is expected to operate using the following format of the input: 
\begin{itemize}
    \item The number of rows in the input is $r$,  while the number of columns is $s$ and $s\geq 3d$. Here $s$ will dictate the total maximum number of columns that any transformer-based function block needs to operate. The reason that $s$ might be larger than $3d$ has to do with the fact that some blocks may need some extra scratchpad space to perform some calculations.
    \item The function block specifies the dimensions of input and output. Say they are $d_h \times d_w$, where $d_h, d_w \le d$
    . These will be part of the instruction which calls this function inside the FLEQ framework, as in~\eqref{eqn:fleq_instruction}.
    \item Suppose each function block has two inputs ($\rmA\in \bR^{d_h\times d_w}$ and $\rmB\in \bR^{d_h\times d_w}$) and one output $f(\rmA, \rmB) = \rmC\in \bR^{d_h\times d_w}$.
    As in~\eqref{eqn:fleq_instruction_partition}, the function block is divided into four parts: (1) the first input $\rmA$ is placed in the first $d_h$ rows and the first $d_w$ columns, (2) the second input $\rmB$ is placed in the first $d_h$ rows and the columns $d+1:d+d_w$, (3) %
    the output $f(\rmA, \rmB) = \rmC$ is in the first $d_h$ rows and the columns $2d+1:2d+d_w$ columns and 4) the rest $s-3d$ column used as scratchpad space for performing necessary calculations. %
    Note that the unused columns are set to zero.
    \item The last $r-d_h$ rows can be used by the transformer-based function block in any way, \eg to store any additional positional encodings. 
\end{itemize}
We put the format of the input of each \emph{transformer-based function block} in~\eqref{eqn:fleq_instruction_partition}. The first input $\rmA = [\vz_a^1, \cdots, \vz_a^{d_w}]$ of the function is zero padded and stored in the first $d$ columns. Similarly,  the second input $\rmB = [\vz_b^1, \cdots, \vz_b^{d_w}]$ is stored in the next $d$ columns. The output/result of the function block $\rmC = [\vz_c^1, \cdots, \vz_c^{d_w}]$ is located in the next $d$ columns while we have some extra $s-3d$ columns which can be used as scratchpad. 
\setcounter{MaxMatrixCols}{20}

\vspace{5mm}
\begin{equation}\label{eqn:fleq_instruction_partition}
\left[
\begin{array}{cccc|cccc|cccc|cc}
\pc\tikznode{sar}{$\mem_{a}^{1}$} &\pc{\hdots} & \pc\tikznode{sar}{$\mem_{a}^{d_w}$} & \zero 
& 
\yy\tikznode{sbr}{$\mem_{b}^{1}$} &\yy{\hdots} & \yy\tikznode{sbr}{$\mem_{b}^{d_w}$} & \zero 
&
\y\tikznode{scr}{$\mem_{c}^{1}$} &\y{\hdots} & \y\tikznode{scr}{$\mem_{c}^{d_w}$} & \zero 
& \hdots & \zero
\\  
* & \hdots &* &*&* & \hdots &* &*&* & \hdots &* &*& \hdots & *
\end{array}
\right]
\end{equation}

\begin{tikzpicture}[remember picture,overlay,cyan,rounded corners]
     \draw[<-,color = olive] (sar)
     |- +(+1,+0.8) %
    node[right]{Input $A$};
     \draw[<-,color = violet] (sbr)
     |- +(+1,+0.8) %
    node[right]{Input $B$};
     \draw[<-,color = cyan] (scr)
     |- +(+1,+0.8) %
    node[right]{Output $C=f(A,B)$};
\end{tikzpicture}

Let us consider the case where we wish to multiply a matrix $\rmA\in \R^{d\times d}$,with a vector $\rvb\in\R^{d\times 1}$. The resulting output  matrix would look as follows:
\begin{equation}
    \bracks*{\begin{array}{c|cc|cc|c}
    \rmA &\rvb &\zero & \rmA^\top\rvb &\zero &\zero
    \end{array}}. \nonumber
\end{equation}

\paragraph{Computational concerns: Do we need full attention?}
In our construction, the computational complexity of each layer depends on the number of embedding vectors that each part of the input has to attend to. Typically, this is quite sparse, as only a few of them need global attention. In our specific construction, only the columns within the scratchpad require global attention. By focusing only on these columns, we can reduce the computational complexity of the attention mechanism from $O(n^2d)$ to $O(nd)$, where n is the number of input sequences, $d$ is the dimension of the embedding vectors.

This reduction in computational complexity is achieved by limiting the attention mechanism to only the columns within the scratchpad, which helps to improve the overall efficiency of the model. Additionally, since the computational complexity grows linearly with the number of input sequences, rather than quadratically, it enables us to scale the model to handle larger input sequences.

\section{Functions in the Unified Template Form}\label{sec:function}

In this section, we demonstrate how to implement a variety of nonlinear functions and basic linear algebra operations using transformers. These techniques will be crucial in the construction of iterative algorithms in the following sections. %
Each transformer-based function block in this section fits in our unified template in terms of input/output parameters' locations. We note here that each transformer-based function block might have its own positional encodings used to transfer the output in the correct place or perform some \texttt{read/write} operations and they are part of the design of the block.

\subsection{Encoding Non-linear Functions within the Attention Mechanism}%
One key ingredient of our constructions %
is encoding various functions within the attention mechanism. We do this by forcing the softmax to act as a sigmoid function and by storing multiple coefficients in the query and value weight matrices. As far as we know, this is the first work that shows how general non-linear functions can be emulated by attention layers. This allows us to create linear combinations of sigmoids that can be accessed by an indicator vector in the input. Our analysis is based on the result of \cite{Barron93} which we present below. 

\begin{definition} Let $\Gamma_{C,B}$ be the set of functions defined in a bounded domain $B$, $f:B\to \mathbb{R}, B\subseteq \R^d$ with a proper extension to $\R^d$ such that they have $C$ bounded Fourier integral, \ie \\$\int \sup_{x \in B} | w \cdot x | \ F(dw) \le C$ holds where $F(dw)$ is the magnitude of the Fourier distribution.
\end{definition}
\begin{definition} Given $\tau >0, C>0$ and a bounded set $B$, 
    let 
    \begin{equation}
        G_{\sigmoid,\tau} = \braces{\gamma\sigmoid(\tau(\rva^T\rvx + b)): \abs{\gamma}\leq 2C,\norm{\rva}_B \leq 1, \abs{b} \leq 1} \nonumber
    \end{equation}
    where $\norm{\rva}_B = \sup_{\rvx\in B} \braces{\rvx^T\rva}$ and $\sigmoid$ is the sigmoid function, \ie $\sigmoid(x) =\frac{1}{1+e^{-x}}$.
\end{definition}
\begin{theorem}[Theorem 3 in \cite{Barron93}]\label{th:Barron}
Every function $f\in\Gamma_{C,B}$ with $f(0) =0 $ and  can be approximated by a linear  combination of sigmoids $f_i\in G_{\sigmoid,\tau}$, $i=1,\hdots m$. If $\tau \geq m^{1/2}\ln m$  the  error scales as
\begin{equation}
    \abs*{f(\rvx) - \sum_{i=1}^m f_i(\rvx)}\leq O\parens*{\dfrac{1}{m^{{1/2}}}} , \; \rvx\in B \nonumber
\end{equation}
\end{theorem}

To encode $N$ different functions, we use the index $j\in[N]$ and write $c_{ji},\rva_{ji}$ for the coefficients of the sigmoids that approximate them or 

\begin{equation}
    f_j(\rvx) = \sum_{i=1}^m c_{ji} \sigmoid( \rvx^T\rva_{ji}) \text{ for }j = 1,\hdots,N \nonumber
\end{equation}
We here note that the terms $\tau, b$ can be incorporated in the term $\rva_{ij}$ by adding an extra coefficient of $1$ in $\rvx$ and multiplying everything with $\tau$.

We are now able to present the lemma on approximating functions using transformer blocks, in a format that is consistent with the FLEQ design outlined in the previous section.

\begin{lemma}\label{lem:sigmoids}
Fix $\epsilon>0$ and consider an input of the form 
\begin{equation}
    \Input = \bracks*{ \begin{array}{cc|cc|cc}
        \ve & \zero&\rvx&\zero &\zero &\zero \\
        \zero  &\zero & \zero& \zero &\zero&\zero\\
        \zero &\zero &\emb_{2d+1} &\zero &\zero &\zero\\
        \emb_1&\emb_{2:d}&\zero&\emb_{d+2:2d} &\emb_{2d+1} &\emb_{2d+2:3d}\\
        0 & 0_{2:d} & 1 & 0_{d+2:2d} & 0 & 0_{2d+2:3d}
    \end{array}} . \nonumber
\end{equation}
where $d$ is chosen according to the \texttt{FLEQ} construction from the previous section and $N$ is the number of functions we encode . $\ve = \rve_j\in\R^N$ is an indicator vector signifying the function we wish to execute.  Then there  exists a   transformer-based function block with 3 layers, $m$ heads and dimensionality $r = 2\log(d) +d+1 = O(d)$ such that  \begin{equation}
    f(\Input) =  \bracks*{ \begin{array}{cc|cc|cc}
        *&*&*&*&\sum_{i=1}^m c_{ji}\sigmoid(\rvx^T\rva_{ji}) +\epsilon & *\\
        \zero  &\zero & \rvx& \zero &\zero&\zero\\
         \zero &\zero &\emb_{2d+1} &\zero &\zero &\zero\\
       \emb_1&\emb_{2:d}&\zero&\emb_{d+2:2d} &\emb_{2d+1} &\emb_{2d+2:3d}\\
        0 & 0_{2:d} & 1 & 0_{d+2:2d} & 0 & 0_{2d+2:3d}
    \end{array} }\nonumber
\end{equation}
where $*$ denoted inconsequential values that will be ignored downstream. This implies that arbitrary function $g \in \Gamma_{C,B}$ can be well approximated by attention layers. 
\end{lemma}
\begin{remark}
Notice that in this case we don't use any extra scratchpad space and thus $s=3d$; however if this function block was to be used with another one that needs $s>3d$ scratchpad space, we would simply zero pad the input of \cref{lem:sigmoids} and ignore these columns. The same holds for the rest of the transformer-based function blocks and we will not mention it from now on. 
\end{remark}

In the expression $\sum_{i=1}^m c_{ji}\sigmoid(\rvx^T\rva_{ji})$, the number head is equal to the number of terms we need.%
We show in the appendix that we can actually encode these $m$ terms in the dimension of the transformer architecture with just one head (See \cref{cor:lem_5_alt}). %
The choice of which result to use can depend on the specific design and can affect both accuracy and efficiency of the implemented transformer network.

The proof of this Lemma is given in \cref{app:functions}.

\subsection{Matrix Transposition and Multiplication by Linearizing the Softmax}
We assume that a $d\times d$ matrix $\rmA$ in the input $\Input$ is represented by a sequence of length $d$, and each of these $d$ columns  has $d$ rows. %
While this representation has the advantage that it is well suited for the matrix multiplication operation (as we will see in the next sub-section), a vectorized form of the matrix is more suited to create transpose. This is how we implement the transpose; we first vectorize the matrix $\rmA$, then with a fixed permutation of the columns we create its vectorized version of a transpose.

\begin{lemma}\label{lem:transpose}
Fix $\epsilon >0$ and consider an input of the following form
\begin{equation}
    \Input = \left[\begin{array}{c|c|c|cc}
        \rmA  &\zero&\zero &\dots & \zero \\
 \zero&      \zero&\zero &\dots & \zero\\
         \emb_{1:d}&\emb_{1:d}&\emb_{1:d}& \dots  &\emb_{1:d}\\
         \rmP_1'&\rmP_{2}'&\rmP_{3}'& \dots  &\rmP_{d}'
    \end{array}\right]. \nonumber
\end{equation}
where $\rmA\in\R^{d\times d}$; then there exists transformer-based function block with 4 layers, 1 head and dimensionality $r = 2d+2\log d = O(d)$ that outputs the following matrix
\begin{equation}
    \Input = \left[\begin{array}{c|c|c|cc}
\rmA' &\rmA'&\rmA'&\dots &\rmA' \\
        \zero&\zero&\zero &\dots &\zero \\
         \emb_{1:d}& \emb_{1:d}& \emb_{1:d}& \dots &\emb_{1:d} \\
        \rmP_1'&\rmP_{2}'&\rmP_{3}'& \dots  &\rmP_{d}'
    \end{array}\right]. \nonumber
\end{equation}
where $\rmA' = \rmA^\top +\epsilon\rmM$, for some $\norm{\rmM}\leq 1$. The error $\epsilon$ depends on the choice of the temperature $\lambda$, as it is  a consequence of the \texttt{read/write} operations.
\end{lemma}

In order for matrix multiplication to fit in our unified template, we need to show for example for the result of $\rmA^\top\rmB$ , where $\rmA\in\R^{k\times m}$ and $\rmB\in\R^{k\times n}$ with $k,m,n <d$ we can achieve the following:
\begin{equation}
   \bracks*{  \begin{array}{cc|cc|cc}
        \rmA & \zero &\rmB &\zero &\zero&\zero\\
        \zero& \zero &\zero&\zero&\zero &\zero
    \end{array}}\xrightarrow[]{}\bracks*{ \begin{array}{cc|cc|cc}
        *& *&*&* &\rmA^\top\rmB &*\\
        \zero& \zero &\zero&\zero&\zero &\zero 
    \end{array} }\nonumber
\end{equation}
The idea we leverage is the linearization of the softmax, \ie for a column vector $\vz = [\vx\;, C]$ for some large constant $C$ we have that $$\softmax(\vz)  = [\vx + \eps \;, *]$$
The error $\eps$ is controlled by the constant $C$.
\begin{lemma}\label{lem:matrixmul}
Let $\rmA\in \R^{k\times m}$ and $\rmB \in \R^{k\times n}$; then for any $\epsilon >0$ there exists a transformer-based function block with  2 layers, 1 head and dimensionality $r = O(d)$ that  outputs the multiplication $\rmA^\top\rmB +\eps\rmM$, for some $\norm{\rmM}\leq 1$ .
\end{lemma}
The implementation of $\rmB^\top\rmA$, $\rmA^\top\rmA$ and $\rmB^\top\rmB$ are simple corollaries of the lemma presented above and we will freely use them in the subsequent sections.
In \cref{app:functions}, we provide the exact form of the input $\Input$ for implementing matrix transposition/multiplication,
as well as the proof of the corresponding Lemmas.

\subsection{Advantage of attention over fully-connected networks} 
It is possible to implement the functions and overall lexicographic functionality presented in previous sections using fully connected networks, as they are also universal function approximators. However, it is easy to demonstrate a depth separation between attention-based networks and fully connected networks. For example, to compute simple functions like polynomials of $x$ (\eg $x^2$), a ReLU network with a depth proportional to $\log(1/\epsilon)$ is required, where $\epsilon$ is the quality of approximation, \eg as showed in \citep{perekrestenko2018universal}. In contrast, we have shown how $x^2$ can be implemented in essentially $2$ layers. This simple depth separation argument highlights the constant vs scaling depth required for several functionalities in fully connected networks versus attention-based networks.
It is important to note that although these constructions are easy to demonstrate their existence, constructing them is not straightforward. In this work, we provide hardcoded attention layers that precisely do that, making it easier to implement these functionalities in practice.

\section{A Basic Calculator}\label{sec:calc}

We show that 
the FLEQ transformer introduced in Section~\ref{sec:fleq}, can be used to build a simple calculator. This transformer consists of  six transformer-based function blocks that implement addition, substraction, multiplication, percentage, division and square root. The formal statement is written as below.
\begin{theorem}
There exists a transformer with $12$ layers, $m$ heads and dimensionality $O(\log n)$ that uses the Unified Attention Based Computer framework in Section~\ref{sec:fleq} to implement a calculator which can perform addition, subtraction, multiplication, and computing the inverse, square root 
and percentage. For computing the inverse and square root, 
the operand needs to be in the range $[ -e^{O(m)}, -\Tilde{\Omega}(\frac{1}{\sqrt{m}})]\cup [\Tilde{\Omega}(\frac{1}{\sqrt{m}}), e^{O(m)}]$ and $[0, O(m^2)]$ respectively, and the returned output is correct up to an error of $O(1/\sqrt{m})$ and $O(1/m)$ respectively. Here, $n$ is the number of operations to be performed.
\end{theorem}
\begin{remark}
In the proof of this theorem, we use \cref{lem:sigmoids} to approximate the square root and the inversion function. That lemma provides error guarantees in terms of the number of heads $m$. We prove \cref{cor:lem_5_alt} in the appendix which provides equivalent error guarantees, but where the error decreases with the dimension $d$ of the transformer. Depending on the design choices of the transformer, either of the results can be used, and the calculator's error guarantee will also change accordingly.
\end{remark}

We show how one can implement a calculator in our FLEQ framework in \cref{alg:calculator}.
\begin{algorithm}[H]{\small
        \caption{A sample program for executing a basic calculator functionality. The following algorithm performs $\frac{\sqrt{1/(((a+b)-c)\cdot d)}}{100}$}\label{alg:calculator}
        \begin{algorithmic}[1]
            \Require $\texttt{mem}[p]=a, \texttt{mem}[q]=b, \texttt{mem}[r]=c, \texttt{mem}[s]=d$. \Comment{The location of the inputs.}
            \State $ \texttt{mem}[t] = f_{\text{add}}(\texttt{mem}[p],\texttt{mem}[q])$ \Comment{$\texttt{mem}[t] =a+b$.}
            \State $ \texttt{mem}[t] = f_{\text{sub}}(\texttt{mem}[t],\texttt{mem}[r])$ \Comment{$\texttt{mem}[t] =(a+b)-c$.}
            \State $ \texttt{mem}[t] = f_{\text{mul}}(\texttt{mem}[t],\texttt{mem}[s])$ \Comment{$\texttt{mem}[t] =((a+b)-c)*d$.}
            \State $ \texttt{mem}[t] = f_{\text{inv}}(\texttt{mem}[t])$ \Comment{$\texttt{mem}[t] =1/((a+b)-c)*d$.}
            \State $ \texttt{mem}[t] = f_{\text{sqrt}}(\texttt{mem}[t])$ \Comment{$\texttt{mem}[t] =\sqrt{1/((a+b)-c)*d}$.}
            \State $ \texttt{mem}[t] = f_{\text{perc}}(\texttt{mem}[t])$ \Comment{$\texttt{mem}[t] =\frac{\sqrt{1/((a+b)-c)*d}}{100}$.}
        \end{algorithmic}}
    \end{algorithm}

Looking at the algorithm, it is clear that for proving the theorem above, it is sufficient to implement the 6 functions (addition, subtraction, multiplication, inversion, square root and percentage) using the transformer-based function blocks defined in \cref{def:tf_func_block}. 
We start with two lemmas,  which can be proved by constructing transformers that add and subtract in a similar way to the OISC transformer constructed in \cref{sec:oisc}.

\begin{lemma}[\texttt{addition}]
    There exists a transformer-based function block with 3 layers, 1 head and dimensionality $O(1)$ which can implement $f(a,b) = a+b$.\label{cor:add}
\end{lemma}
\begin{proof}
Consider the input in the form of \cref{eqn:fleq_instruction_partition}
\begin{equation}
    \Input = \bracks*{ \begin{array}{cc|cc|cc}
    a&\zero &b &\zero &0 &\zero \\
    \zero &\zero &\zero &\zero &\zero &\zero\\
    \emb_{2d+1} &\zero &\zero &\zero &\zero &\zero\\
    \zero  &\emb_{2:d}&\emb_{d+1} &\emb_{d+2:2d} &\emb_{2d+1} &\emb_{2d+2:3d}\\
    1&\zero &0 &\zero &0 &\zero
    \end{array}}
\end{equation}
We can perform the following transformation
\begin{align}
   \bracks*{ \begin{array}{cc|cc|cc}
    a&\zero &b &\zero &0 &\zero \\
    \zero &\zero &\zero &\zero &\zero &\zero
    \end{array}} &\xrightarrow[]{} \bracks*{ \begin{array}{cc|cc|cc}
    a&\zero &b &\zero &0 &\zero \\
    a&\zero &b &\zero &0 &\zero\\
    \zero &\zero &\zero &\zero &\zero &\zero
    \end{array}}\\
    &\xrightarrow[]{}\bracks*{ \begin{array}{cc|cc|cc}
    a&\zero &0 &\zero &0 &\zero \\
    0&\zero &b &\zero &0 &\zero\\
    \zero &\zero &\zero &\zero &\zero &\zero
    \end{array}}\\
    &\xrightarrow[]{}\bracks*{ \begin{array}{cc|cc|cc}
    a+b&\zero &0 &\zero &0 &\zero \\
    0&\zero &0 &\zero &0 &\zero\\
    \zero &\zero &\zero &\zero &\zero &\zero
    \end{array}}\\
    &\xrightarrow[]{}\bracks*{ \begin{array}{cc|cc|cc}
    a+b&\zero &0 &\zero &a+b &\zero \\
    0&\zero &b &\zero &0 &\zero\\
    \zero &\zero &\zero &\zero &\zero &\zero
    \end{array}}
\end{align}
The first and second step are implemented with one feed-forward layer each. The third step with the \cref{ss:copy}. We have ignored the last three  rows since we don't change them and we only use them for the last step. 
\end{proof}

\begin{lemma}[\texttt{subtraction}]
    There exists a transformer-based function block with 3 layers, 1 head and dimensionality $O(1)$ which can implement $f(a,b) = a-b$.\label{cor:sub}
\end{lemma}

This lemma can be proved in the exact same way as the previous one. In addition, we can use the theory presented in \cref{lem:matrixmul} to get the following corollaries:
\begin{corollary}[\texttt{multiplication}]
There exists a transformer-based function block with 2 layers, 1 head and dimensionality $O(d)$ which can implement $f(a,b) = ab$.\label{cor:mul}
\end{corollary}

\begin{corollary}[\texttt{percentage}]
There exists a transformer-based function block with 2 layers, 1 head and dimensionality $O(1)$ which can implement $f(a) = a/100=a*0.01$.\label{cor:perc}
\end{corollary}

To implement inversion function, we introduce the following lemma.

\begin{lemma}
Given $\epsilon,\delta \in [0,1]$, and $C\geq 1$ there exists a function $f$ of the form
$
    f(x) = \sum_{i=1}^m c_i \sigmoid(w_i x + b_i),
$
where $\sigmoid$ is the sigmoid function, such that 
\begin{align*}
 \forall x\in\left[\delta, C\right],   \left|f(x) - \frac{1}{x}\right|\leq \epsilon,
\end{align*}
as long as $d=\Omega\left(\frac{\log(1/(\epsilon\delta))}{\epsilon\delta}+\log C\right)$.
\end{lemma}
We can use this lemma along with the result presented  in \cref{lem:sigmoids} to get the following corollary:
\begin{corollary}[\texttt{inversion}]
There exists a transformer-based function block with 3 layers and $m$ heads  which can implement $f(a) = \frac{1}{a}$ up to error $\Tilde{O}(\frac{1}{\sqrt{m}})$ for all $a\in [\Tilde{\Omega}(\frac{1}{\sqrt{m}}), \Tilde{O}(e^m)]$.\label{cor:inv}
\end{corollary}

Note that using \cref{cor:mul} (multiplication) and \cref{cor:inv} (inversion), the operation of division can be implemented as well. Next, we move on to showing the way of implementing square root.

\begin{lemma}
Given $\epsilon \in [0,1]$, and $C\geq 1$ there exists a function $f$ of the form
$
    f(x) = \sum_{i=1}^m c_i \sigmoid(w_i x + b_i),
$
where $\sigmoid$ is the sigmoid function such that 
\begin{align*}
 \forall x\in[0, C],   \left|f(x) - \sqrt{x}\right|\leq \epsilon,
\end{align*}
as long as $m=\Omega\left(\frac{\sqrt{C}}{\epsilon}\right)$.
\end{lemma}

We can use this lemma along with the result presented in \cref{lem:sigmoids} to get the following corollary:
\begin{corollary}[\texttt{sqrt}]
There exists a transformer-based function block with 3 layers and m heads  which can implement $f(a) = \sqrt{a}$ up to error $O(1/m)$ for all $a\in [0,O(m^2)]$.
\end{corollary}

The functions $f:x\to \frac{1}{x}$ (inversion) and $f:x\to\sqrt{x}$ (square root) since they can be approximated by sums of sigmoids, they can directly be encoded in the standard transformer-based function block form through \cref{lem:sigmoids}.

\paragraph{What other functions can our calculator implement?}
We have included some of the most commonly used operations in calculators in our construction, but it can be extended to include a wider variety of operations such as algebraic and trigonometric functions. When implementing these functions within our transformer architecture, there are typically two choices that can be made. One option is to approximate the target function $f(x)$ using sigmoids. Another option is to use an iterative numerical algorithm where the next output $y$ is calculated based on the previous output $y$ and the goal is to minimize the difference between the calculated output and the target function $f(x)$. This algorithm takes the form $y_{k+1} = g(y_k)$, where $g$ is typically an algebraic function. The desired accuracy is achieved when the difference between the calculated output and target function is less than or equal to a certain tolerance $\epsilon$.

\section{Linear Algebra}\label{sec:linearalg}

In  \cref{sec:function}, we demonstrated the implementation of matrix transpose and matrix multiplication as transformer-based function blocks. Utilizing these implementations, we proceed to execute two iterative algorithms for determining the inverse of a matrix through the Newton-Raphson Method and identifying the eigenvector corresponding to the maximum eigenvalue through the Power Iteration method.

\paragraph{Linear algebra using Transformers} In the study conducted by \cite{charton2021linear}, the author implemented some standard matrix method operations using a transformer-based architecture. Four distinct encoding schemes were proposed and applied to nine different operations, ranging from matrix multiplication to eigenvalue decomposition. We find  that the size of the networks in \cite{charton2021linear} is comparable to that of ours. 

As an example we compare the required network size of ours and \cite{charton2021linear}, for the task of transposing a matrix of size $30 \times 30$: our construction uses a transformer with 1 layer, 1 head and width of 168, while the transformer in~\cite{charton2021linear} has 1 layer, 8 heads and width of 256.
Notice that the number of layers, heads and width reported above may seem different with \cref{lem:transpose}; however, in the proof of  \cref{lem:transpose} we first vectorize the matrix ($1$ layer), then we implement the fixed permutation using \cref{lem:write} ($1$ layer) and finally we use another $2$ layers to bring back the matrix in its original representation. If the matrix is given to us, as in \cite{charton2021linear}, in its transposed form then we only need one layer and the two sets of encodings to perform the fixed permutation. Since the maximum size of the matrix is $30\times 30$, the sequence length is $n= 30^2$ and thus the size of each of the encodings will be $10$, leading to an input with  width $2\cdot 10 +1 =21$. This will lead to a total width of $168$, due to the ReLU layer in \cref{app:increase}, for adding two binary vectors, having a width eight times the input's width.

We  intend to further investigate our constructions, by implementing them  and evaluating the errors involved as a function of the constants used in the proof of \cref{lem:matrixmul} and the temperature in \cref{lem:read}, in future work.

\paragraph{Matrix Inversion.}

We can use the Unified Attention Based Computer to write a program for Matrix Inversion using the functions for matrix multiplications and a function for subtraction. We do so by implementing Newton's algorithm for matrix inversion using our unified framework. The pseudo code for the algorithm is as follows:
 \begin{algorithm}[H]{\small
        \caption{Pseudocode for running Newton's algorithm for Matrix inversion for $T$ iterations.}
        \begin{algorithmic}[1]
           \State $\rmX_{-T} = \epsilon \rmA$
           \For{$i=-T, \dots, 0$}
           \State $\rmX_{i+1} = \rmX_i(2\id - \rmA \rmX_i)$
           \EndFor
        \end{algorithmic}\label{alg:inverse_pseudo}}
    \end{algorithm}
    
    \begin{lemma}\label{lem:matrix_inversion}
   Consider a matrix $\rmA\in\R^{d\times d}$, then for any $\eps >0$ there exists a transformer  with 13 layers, 1 head and dimensionality $r = O(d)$ that emulates \cref{alg:inverse_pseudo}  with output $\rmX^{(\text{transf})}_{1}$ that satisfies $\norm{\rmX^{(\text{transf})}_{1} - \rmX_1}\leq \eps$.
    \end{lemma}
\proof
The proof of this lemma is  the code using the \texttt{FLEQ} instruction provided below  ( \cref{alg:inverse}). 
Let $f_\text{mul}$, $f_{\text{sub}}$ and $f_{\text{transp}}$  be the functions that implement multiplication, substraction and transpose respectively. Then, the following code runs Newton's algorithm for matrix inversion.
 \begin{algorithm}[H]{\small
        \caption{Program to compute the approximate inverse using our Unified Attention Based Computer}
        \begin{algorithmic}[1]
            \Require $\texttt{mem}[a]=\rmA$. \Comment{This is the location of the input.}
            \Require $\texttt{mem}[p]=2\id$, $\texttt{mem}[x]=\epsilon\id$, $\texttt{mem}[y]=\zero$, $\texttt{mem}[q]=-1$. \Comment{Constants.}
            \Require $\texttt{mem}[t]=-T$.\Comment{Iteration counter, $i$ initialized as $i:=-T$.}
            \vspace{0.6em}
            \State $\texttt{mem}[x] = f_{\text{mul}} (\texttt{mem}[x], \texttt{mem}[a])$.\Comment{Initializes the result, $\rmX_{-T} := \epsilon\rmA$.}
            \State $ \texttt{mem}[a] = f_{\text{transp}}(\texttt{mem}[a],\texttt{mem}[y])$ \Comment{Transpose $\rmA$.}
        \State $\texttt{mem}[y] = f_{\text{mul}} (\texttt{mem}[a], \texttt{mem}[x])$.\Comment{First sub-step of Newton's algorithm, $\rmY := \rmA\rmX_i$}
      \State $\texttt{mem}[y] = f_{\text{sub}} (\texttt{mem}[p], \texttt{mem}[y])$.\Comment{Second sub-step of Newton's algorithm, $\rmY := 2\id - \rmY$}
      \State $\texttt{mem}[y] = f_{\text{transp}}(\texttt{mem}[y],\texttt{mem}[q])$. \Comment{Transpose of $\rmY$.}
     \State $\texttt{mem}[x] = f_{\text{mul}} (\texttt{mem}[x], \texttt{mem}[y])$.\Comment{Updating the result, $\rmX_{i+1} := \rmX_i \rmY$}
        \State $\texttt{mem}[t] = f_{\text{sub}} (\texttt{mem}[t], \texttt{mem}[q])$.\Comment{Increment counter, $i:=i+1$.}
        \State if $\texttt{mem}[t]\leq 0$ goto instruction $3$.\Comment{Keep looping back as long as $i\leq 0$.}
        \State EOF.\Comment{End of File command.}
        \end{algorithmic}\label{alg:inverse}}
    \end{algorithm}

\paragraph{Power Iteration.}
The  Power Iteration algorithm (\cref{alg:powerIt_pseudo}) is used for finding the dominant eigenvalue, the one that has the maximum absolute value, and corresponding eigenvector of a diagonalizable matrix. The algorithm starts with an initial approximation of the eigenvector and converges linearly to the eigenvector associated with the dominant eigenvalue; below we provide the  pseudocode.
\begin{algorithm}[H]{\small
\caption{Power Iteration}\label{alg:powerIt}
\begin{algorithmic}[1]
 \Statex{ Input: $\rmA, T$ }
 \State Initialize $b_0=\vOne$
 \For{ $k= 0, \hdots, T-1$}
 \State $\rvb_{k+1} = \rmA\rvb_{k}$
 \EndFor
 \State $\rvb = \dfrac{\rvb_{T}}{\norm{\rvb_{T}}}$
\end{algorithmic}\label{alg:powerIt_pseudo}}
\end{algorithm}

The last step in the algorithm above needs a normalization by the norm of $\rvb_T$. While we can compute $\|\rvb_T\|^2$ easily and precisely using the matrix multiplication function block (since $\|\rvb_T\|^2 = \rvb_T^\top \rvb_T$), computing the norm and taking its inverse using the function block from \cref{sec:calc} would induce error. Hence, we use the following Newton's algorithm that converges quadratically.

\begin{algorithm}[H]{\small
\caption{Newton's algorithm to compute inverse square root: $1/\sqrt{S}$}\label{alg:newt_inv_sqrt}
\begin{algorithmic}[1]
 \Statex{ Input: $S$ }
 \State Initialize $x_0=1$
 \For{ $k= 0, \hdots, T$}
 \State $x_{k+1} = x_{k}\left(\frac{3}{2} - \frac{S}{2}x_k^2\right)$
 \EndFor
\end{algorithmic}}
\end{algorithm}
\begin{lemma}\label{lem:power_iteration}
Consider a matrix $\rmA\in\R^{d\times d}$, then for any $\eps >0$ there exists a transformer with $13$ layers, 1 head and dimensionality $r = O(d)$ that emulates \cref{alg:powerIt_pseudo} for $T=O(\log 1/\epsilon)$ iterations with output $\rvb_{T+1}^{(\text{transf})}$ that satisfies $\norm{\rvb_{T+1}^{(\text{transf})} - \rvb_{T+1}}\leq \eps$. 
\end{lemma}
\proof
The proof consists of translating each step of the pseudocode for \cref{alg:powerIt_pseudo} and \cref{alg:newt_inv_sqrt} to commands of our unified framework.
\begin{algorithm}[H]{\small
        \caption{Program to simulate Power Iteration using our Unified Attention Based Computer}
        \begin{algorithmic}[1]
            \Require $\texttt{mem}[a]=\rmA$, $\texttt{mem}[b]=\vOne$, $\texttt{mem}[\text{inv\_norm}]=1$. \Comment{Location of matrix and initialization.}
            \Require $\texttt{mem}[q] = 1$, 
$\texttt{mem}[p] = 0$, $\texttt{mem}[r] = 0.5$, $\texttt{mem}[s] = 1.5$\Comment{Constants.}
            \Require $\texttt{mem}[t_1]=\texttt{mem}[t_2]=-T+1$, 
            \vspace{0.6em}
            \State $\texttt{mem}[a] = f_{\text{transp}}(\texttt{mem}[a],\texttt{mem}[p])$. \Comment{Transpose of $\rmA$.}
            \State $\texttt{mem}[b] = f_{\text{mul}} (\texttt{mem}[a], \texttt{mem}[b])$. 
            \Comment{Inner product: $\rmA\rvb_k$.}
        \State $\texttt{mem}[t] = f_{\text{add}} (\texttt{mem}[t_1], \texttt{mem}[q])$.\Comment{Increment counter, $i:=i+1$.}
        \State if $\texttt{mem}[t_1]\leq 0$ goto instruction $2$.\Comment{Keep looping back as long as $i\leq 0$.}
        \State $\texttt{mem}[\text{norm\_square}] = f_{\text{mul}}(\texttt{mem}[b],\texttt{mem}[b])$. \Comment{Calculate $\norm{\rvb_T}^2$.}
        \Statex Code for \cref{alg:newt_inv_sqrt} begins.
        \State $\texttt{mem}[\text{y}] = f_{\text{mul}}(\texttt{mem}[\text{inv\_norm}],\texttt{mem}[\text{inv\_norm}])$. \Comment{Calculate $x_k^2$.}
        \State $\texttt{mem}[\text{y}] = f_{\text{mul}}(\texttt{mem}[\text{norm\_square}],\texttt{mem}[\text{y}])$. \Comment{Calculate $Sx_k^2$.}
        \State $\texttt{mem}[\text{y}] = f_{\text{mul}}(\texttt{mem}[r],\texttt{mem}[\text{y}])$. \Comment{Calculate $Sx_k^2/2$.}
        \State $\texttt{mem}[\text{y}] = f_{\text{sub}}(\texttt{mem}[s],\texttt{mem}[\text{y}])$. \Comment{Calculate $(3-Sx_k^2)/2$.}
        \State $\texttt{mem}[\text{inv\_norm}] = f_{\text{mul}}(\texttt{mem}[\text{inv\_norm}],\texttt{mem}[\text{y}])$. \Comment{Update $x_{k+1}:=x_{k}(3-Sx_k^2)/2$.}
            \State $\texttt{mem}[t_2] = f_{\text{add}} (\texttt{mem}[t_2], \texttt{mem}[q])$.\Comment{Increment counter, $j:=j+1$.}
        \State if $\texttt{mem}[t_2]\leq 0$ goto instruction $6$.\Comment{Keep looping back as long as $j\leq 0$.}
        \Statex Code for \cref{alg:newt_inv_sqrt} ends.
        \State $\texttt{mem}[b] = f_{\text{mul}}(\texttt{mem}[b],\texttt{mem}[\text{inv\_norm}])$. \Comment{$\rvb:=\rvb_T/\|\rvb_T\|$.}
         \State EOF.\Comment{End of File command.}
        \end{algorithmic}}
    \end{algorithm}

\paragraph{What other numerical linear algebra algorithms can transformers implement?}
The algorithms presented above serve as proof of concept for the potential to build small linear algebra libraries using our transformer construction. As demonstrated, the size of the looped transformer is constant regardless of the depth. To implement iterative numerical algorithms, additional functions can be incorporated into our architecture. For instance, QR decomposition, Gauss-Seidel, Arnoldi iteration, or Lanczos algorithm can be implemented. While we have not included detailed code for these specific algorithms, the above examples should provide sufficient insight on how to do so.

\section{Emulating Learning Algorithms at Inference Time}\label{sec:SGD}
In this section we demonstrate the ability of our unified template to emulate Stochastic Gradient Descent (SGD). We begin by examining the case of  linear models, before extending our results to the implementation of the backpropagation algorithm for two layer neural networks. Utilizing this as a ``function'' which we call at each step, we demonstrate the application of SGD in updating the implicit weights of a model. 

Our work demonstrates that looped transformers can effectively perform in-context learning for a wide range of models and achieve high levels of accuracy, given access to a sufficient number of inference calls/loops. Previous research, such as \cite{akyurek2022learning} and \cite{gargcan}, has limited in-context learning to a single inference call of a transformer model deeper than ours, which restricts the types of models that can be learned and the level of accuracy that can be achieved. To implement complex iterative programs like SGD, either a looped structure transformer or one that grows in size with the program's depth is required, unless widely believed complexity conjectures are falsified. Additionally, this is the first work to show that transformers can implement SGD on more general loss functions and models beyond linear regression.

\paragraph{Stochastic Gradient Descent in linear models.} In \cref{alg:sgd_linear} we provide the program for running SGD in linear models; that is we perform updates of the form: $\weight_{t+1} = \weight_t -\eta\sum_{i=1}^{\mathcal{D}} (\weight^\top\dpoint_i-y_i)\dpoint_i$, where $\weight$ is the weight vector,  $(\dpoint_i, y_i)$ is the feature-label pair of the $i-$th data point, and $\eta$ is the step-size.  The program  iterates through the $\mathcal{D}$ data points that the user gives and cycles back to the first point after one pass is completed. The step-size is given as input by the user.  %

\begin{lemma}\label{lem:sgd_linear}
Let $\eps > 0$, there exists a transformer with 13 layers, 1 head and dimensionality $O(\log(\mathcal{D})+d)$ that uses the Unified Attention Based Computer framework in \cref{sec:fleq} to implement $T$ iterations of SGD on a weight vector $\weight\in\R^d$, over a set of $\mathcal{D}$ data points $(\dpoint_i,y_i)\in\R^{d+1}$, $i=1,\hdots,\mathcal{D}$ with error up to $\eps$. The step size is given as a parameter to the program.%
\end{lemma}
\begin{remark}
The error is controlled by two parameters: the temperature $\lambda$ and the constants used in the proof of \cref{lem:matrixmul}. Implementing arbitrary loss functions $f$ and thus updates of the form $\weight_{t+1} = \weight_t -\eta\sum_{i=1}^{\mathcal{D}} f'(\weight^\top\dpoint_i-y_i)\dpoint_i$ would introduce an extra error as %
a result of Barron's theorem (\cref{th:Barron}) applied in \cref{lem:sigmoids}. Specifically,  we would need  in general $poly(T\mathcal{D})$ heads, in order to ensure control over this approximation error. However, if the derivative $f'(x)$ of the loss function $f(x)$ is a sum of sigmoids, the number of heads will be equal to the number of sigmoids required, and there will be no error associated with this aspect of the construction.
\end{remark}

\begin{algorithm}[H]{\small
        \caption{Program to simulate SGD using our Unified Attention Based Computer}
        \begin{algorithmic}[1]
            \Require $\texttt{mem}[w]=\weight$, $\texttt{mem}[\eta]=\eta$. 
            \Comment{Location of the weight and  step-size.}
            \Require  $\texttt{mem}[x_0+i-1]=\dpoint_i$, $i=1,\hdots,\mathcal{D}$. \Comment{Location of the data points.}
            \Require  $\texttt{mem}[y_0+i-1]=y_i$, $i=1,\hdots,\mathcal{D}$. \Comment{Location of the labels.}
            \Require  $\emb_{x_*}=x_0$.
            \Comment{$\emb_{x_*}$ is a pointer to the first data. }%
            \Require  $\emb_{y_*}=y_0$.
            \Comment{$\emb_{y_*}$ is a pointer to the first label. }
            \Require  $\emb_{\pointer} = \text{instr}_1$. \Comment{Program Counter points to first instruction. }
            \Require $\texttt{mem}[q] = 1$, $\texttt{mem}[p]=0$, $\texttt{mem}[z] = n$. \Comment{Constants.}
            \Require %
            $\texttt{mem}[j] =-\mathcal{D}$.\Comment{Within epoch iteration counter initialized to $-n$.}
            \Require %
            $\texttt{mem}[k] =-T$.\Comment{Epoch counter initialized to $-T$.}
            \vspace{0.5ex}
            \State ( $\text{instr}_1$) \quad $\texttt{mem}[temp] = f_{\text{mul}} ( \texttt{mem}[\emb_{x_*}],\texttt{mem}[w])$. \Comment{Inner product: $\weight^\top \dpoint_i$.}
            
            \State ( $\text{instr}_2$) \quad $\texttt{mem}[temp] = f_{\text{sub}} ( \texttt{mem}[temp],\texttt{mem}[\emb_{y_*}])$. \Comment{Substract the label: $\weight^\top\dpoint_i - y_i$.}

           \State ( $\text{instr}_3$) \quad $\texttt{mem}[temp] = f_{\text{mul}}( \texttt{mem}[\emb_{x_*}],\texttt{mem}[temp])$. \Comment{Multiply with the data point $\dpoint_i$. }
           
                \State $\texttt{mem}[temp] = f_{\text{mul}}(\texttt{mem}[temp],\texttt{mem}[\eta])$. \Comment{Multiply with the step-size.}
            
             \State $\texttt{mem}[w] = f_{\text{sub}}(\texttt{mem}[w],\texttt{mem}[temp])$. \Comment{Subtract from $\weight$ one gradient step.}
             
        \State $\texttt{mem}[\text{instr}_1] = f_{\text{incr\_pointer}}(\texttt{mem}[\text{instr}_1]).$\Comment{Increment pointer.}
        
         \State $\texttt{mem}[\text{instr}_2] = f_{\text{incr\_pointer}}(\texttt{mem}[\text{instr}_2]).$\Comment{Increment pointer.}
        
        \State $\texttt{mem}[\text{instr}_3] = f_{\text{incr\_pointer}}(\texttt{mem}[\text{instr}_3]).$ \Comment{Increment pointer.}
        
        \State $\texttt{mem}[j] = f_{\text{add}} (\texttt{mem}[j], \texttt{mem}[q])$.\Comment{Increment within epoch iteration counter by 1.}
        
        \State if $\texttt{mem}[j]\leq 0$ goto 1. \Comment{Cycle back to the first data point.}
        
        \State $\texttt{mem}[j] = -\mathcal{D}.$ \Comment{Reset counter.}
        
        \State $\texttt{mem}[\text{instr}_1] = f_{\text{reset\_pointer}}(\texttt{mem}[\text{instr}_1], x_0).$ \Comment{Reset pointer.}
         \State $\texttt{mem}[\text{instr}_2] = f_{\text{reset\_pointer}}(\texttt{mem}[\text{instr}_2], y_0).$ \Comment{Reset pointer.}
        \State $\texttt{mem}[\text{instr}_3] = f_{\text{reset\_pointer}}(\texttt{mem}[\text{instr}_3], x_0).$ \Comment{Reset pointer.}
        
        \State $\texttt{mem}[k] = f_{\text{add}} (\texttt{mem}[k], \texttt{mem}[q])$.\Comment{Increment epoch counter by 1.}
        
        \State if $\texttt{mem}[k]\leq 0$ goto 1. \Comment{Cycle back to the first data point.}
        
        \State EOF. \Comment{End of File command.}
        \end{algorithmic}\label{alg:sgd_linear}}
    \end{algorithm}

The following will detail the essential procedures for implementing the Stochastic Gradient Descent algorithm. We employ three pointers, namely $\emb_{\pointer}$, $\emb_{x_*}$ and $\emb_{y_*}$ in our algorithm. The first one, referred to as  program counter,   is used to iterate through the commands; after one pass over all data points is completed, the program counter is reset to the first instruction (line 16), until $T$ full passes have been completed. The second and third ones, referred to as data and label pointer respectively, iterate through the features and labels one by one. The increment of the pointer $\emb_{x_*}$ needs to occur in both instructions 1 and 3, as to in the next iteration they have been updated  from $\text{instr}_i(\emb_{x_*},w,temp) \to \text{instr}_i(\emb_{x_*}+1,w,temp)$, $i=1,3$.
The same holds for the pointer $\emb_{y_*}$ in line 7. Finally, we reset the two pointers in lines 13,14 to cycle back in the first feature and label.

To enhance understanding, we note that lines 6-8 modify the instructions themselves; instead of doing this we could have $\mathcal{D}$ copies of the lines 1-3, each one with parameters pointers of a different (feature,label) pair. In that case the number of \emph{instructions} would have been $7\mathcal{D}$. 

Notice that the functions $f_{\text{incr\_pointer}}$ and $f_{\text{reset\_pointer}}$ can be directly implemented using \cref{app:increase}.

\paragraph{Backpropagation and SGD.} 
We will now generalize the result of \cref{lem:sgd_linear} to two layer neural networks with non-linear activation functions; we demonstrate in \cref{alg:sgd} how this can be achieved if the activation function is  the sigmoid function.

Closest to this section  is the work of \cite{akyurek2022learning}, where the authors prove that constant number of layers is needed to perform one step SGD in linear models, using decoder only transformer architecture.

\begin{algorithm}[H]{\small
\caption{Backpropagation}\label{alg:backprop_pseudo}
\begin{algorithmic}[1]
\Statex {Loss function: $J(x)= \frac{1}{2}x^2$}
 \Statex{ Input: $\weights_1\in\R^{m\times d}$, $\bias_1\in \R^m$, $\weights_2\in\R^{m\times 1}$, $\bias_2\in\R$ $\dpoint\in\R^d$, $y\in\R$}
 \State Compute $\vz = \weights_1 \dpoint +\bias_1$.
 \State Compute $ \va = \sigma(\vz)$.
 \State Compute $o = \weights_2\va +\bias_2$.
 \State Compute $\delta_2 = (o-y)$.
 \State Compute $\delta_1 =  \sigma'(\vz) \odot\weights_2 (o-y) $.
 \State Compute $\frac{\partial J}{\partial\weights_2} = \delta_2 \va^\top$.
 \State Compute $\frac{\partial J}{\partial\bias_2} = \delta_2$.
 \State Compute $\frac{\partial J}{\partial\weights_1} = \delta_1\dpoint^\top$.
 \State Compute $\frac{\partial J}{\partial\bias_1} = \delta_1$.
\end{algorithmic}}
\end{algorithm}

\begin{lemma}\label{lem:SGD}
Let $\eps > 0$, there exists a transformer with 13 layers, 1 head and dimensionality $O(\log(\mathcal{D})+d)$ that uses the Unified Attention Based Computer framework in \cref{sec:fleq} to implement $T$ iterations of SGD on a two layer neural network, over a set of $\mathcal{D}$ data points $(\dpoint_i,y_i)\in\R^{d+1}$, $i=1,\hdots,\mathcal{D}$ with error up to $\eps$. The step size is given as a parameter to the program.
\end{lemma}
\begin{remark}
The program we provide in \cref{alg:backprop} is implemented as an independent function, which we call multiple times. Specifically, in line 1 of \cref{alg:sgd} we call the algorithm for backpropagation at each iteration with a different data point. In terms of our construction,  this translates to different instructions which will be in total $O(\mathcal{D})$, each one with parameters pointers to a different data point. However, as in \cref{alg:sgd_linear} the utilization of a pointer that changes the instructions themselves, would result in a program of constant length; we did not do this in order to contain the total length of the program.
\end{remark}
\begin{remark}
If we want to account for different activation functions we can use \cref{lem:sigmoids} to express the activation function and its derivative as sums of sigmoids. The number of heads would need to be in that case $poly(T\mathcal{D})$ to ensure control over the error induced by the approximation.
\end{remark}
\begin{algorithm}[H]{\small
        \caption{Program to simulate Backpropagation for two layer Neural Networks}
        \begin{algorithmic}[1]
        \Statex \textbf{Input:} $\emb_{w_1},\emb_{w_2},\emb_{b_1},\emb_{b_2}$ \Comment{Pointers to weights and biases.}
        
        \Statex \textbf{Input:} $\emb_{x},\emb_{y}$ \Comment{Pointer to data point and label.}
        
        \Statex \textbf{Input:} $\eta$. \Comment{Pointer to step size.}
        
            \Require $\texttt{mem}[q] = 1$, $\texttt{mem}[p]=0$, $\texttt{mem}[r] =-1$, $\texttt{mem}[m] = m$. \Comment{Constants.}
            
            \Require %
            $\texttt{mem}[k] =1$.\Comment{Iteration counter,  $k:=1$.}
            
            \Require $\emb_{z} = z_{\textsc{T}}^1$. 
            \Comment{Pointer for $z$.}
            
            \Require $\emb_{\delta} = \delta_{1,\textsc{T}}^1$. \Comment{Pointer for $\delta_1$.}
            \vspace{0.5ex}
            \State ( $\text{instr}_1$) $\texttt{mem}[temp] = f_{\text{trans}} (\texttt{mem}[\emb_{w_1}], \texttt{mem}[p])$. 
            \Comment{Create $\weights_1^\top$.}
            
            \State $\texttt{mem}[z] = f_{\text{mul}} (\texttt{mem}[temp], \texttt{mem}[\emb_{x}])$. 
            \Comment{Multiply: $\weights_1 \dpoint$.}
            
           \State $\texttt{mem}[z] = f_{\text{add}}(\texttt{mem}[z],\texttt{mem}[\emb_{b_1}])$. \Comment{Add the bias: Compute $\vz$. }
           
           \State $\texttt{mem}[a] = f_{\text{sigmoids}}(\texttt{mem}[z],\texttt{mem}[q])$. \Comment{Compute $\va = \sigma(\vz)$. }
           
           \State $\texttt{mem}[temp] = f_{\text{trans}} (\texttt{mem}[\emb_{w_2}], \texttt{mem}[p])$. 
            \Comment{Create $\weights_2^\top$.}
            
            \State $\texttt{mem}[o] = f_{\text{mul}}(\texttt{mem}[temp],\texttt{mem}[a])$. \Comment{Multiply: $\weights_2\va$.}
            
            \State $\texttt{mem}[o] = f_{\text{add}}(\texttt{mem}[o],\texttt{mem}[\emb_{b_2}])$. \Comment{Add bias: Compute $o$.}
            
             \State $\texttt{mem}[\delta_2] = f_{\text{sub}}(\texttt{mem}[o],\texttt{mem}[\emb_{y}])$. \Comment{Compute $\delta_2$.}
            
             \State $\texttt{mem}[\delta_1] = f_{\text{mul}}(\texttt{mem}[\emb_{w_2}],\texttt{mem}[\delta_2])$. \Comment{Multiply $\weights_2\delta_2$.}
             
             \State $\texttt{mem}[flag] = f_{\text{sub}}(\texttt{mem}[k],\texttt{mem}[m])$. \Comment{Create $k-m$.}
             
              \State $ \texttt{mem}[\emb_z] = f_{\text{trans}}(\texttt{mem}[z],\texttt{mem}[p])$. \Comment{Store $\vz$ to consecutive memory cells.}
             
             \State  $ \texttt{mem}[\emb_\delta] = f_{\text{trans}}(\texttt{mem}[\delta_1],\texttt{mem}[p])$. \Comment{Store $\delta_1$ to consecutive memory cells.}
             
             \State if $\texttt{mem}[flag]\leq 0$ goto 20.\Comment{If we iterated all the elements goto next command. }

             \State ( $\text{instr}_{14}$) $\texttt{mem}[temp'] = f_{\text{sigmoids}}(\texttt{mem}[p],\texttt{mem}[\emb_z])$. \Comment{Create $\sigma(z_i)$.}
             
             \State  $\texttt{mem}[temp''] = f_{\text{sub}}(\texttt{mem}[q],\texttt{mem}[temp'])$. \Comment{Create $1-\sigma(z_i)$.}
             \State  $\texttt{mem}[temp'] = f_{\text{mul}}(\texttt{mem}[temp'],\texttt{mem}[temp''])$. \Comment{Create $\sigma'(z_i)=\sigma(z_i)(1-\sigma(z_i))$.}

             \State ( $\text{instr}_{17}$) $\texttt{mem}[\emb_\delta] = f_{\text{mul}}(\texttt{mem}[temp'],\texttt{mem}[\emb_\delta])$. \Comment{Create $\sigma'(z_i)(\weights_2)_i(o-y)$.}
             
             \State $\texttt{mem}[\text{instr}_{14} ]= f_{\text{incr\_pointer}}(\texttt{mem}[\text{instr}_{14}])$. \Comment{Point to next element of $z$.}
             
             \State $\texttt{mem}[\text{instr}_{17} ]= f_{\text{incr\_pointer}}(\texttt{mem}[\text{instr}_{17}])$. \Comment{Point to next element of $\delta_1$.}
             
        \State $\texttt{mem}[k] = f_{\text{add}} (\texttt{mem}[k], \texttt{mem}[q])$.\Comment{Increment counter, $k:=k+1$.}

        \State If $\texttt{mem}[p]\leq 0$ goto 13.\Comment{Loop back.}

        \State $\texttt{mem}[\text{instr}_1] = f_{\text{reset\_pointer}}(\texttt{mem}[\text{instr}_{14}], z^1_\top)$. \Comment{Reset pointer.}
        
        \State $\texttt{mem}[\text{instr}_{15}] = f_{\text{reset\_pointer}}(\texttt{mem}[\text{instr}_{15}], \delta_{1,\top}^1)$. \Comment{Reset pointer.}
        
        \State  $\texttt{mem}[grad\_W_2] =f_{\text{mul}}(\texttt{mem}[\delta_2],\texttt{mem}[a])$. \Comment{Create $\frac{\partial J}{\partial \rmW_2}$.}
        
        \State  $\texttt{mem}[grad\_b_2] =f_{\text{mul}}(\texttt{mem}[\delta_2],\texttt{mem}[q])$. \Comment{Create $\frac{\partial J}{\partial \rvb_2}$.}
        
         \State  $\texttt{mem}[grad\_W_1] =f_{\text{mul}}(\texttt{mem}[\delta_1],\texttt{mem}[\emb_{x}])$. \Comment{Create $\frac{\partial J}{\partial \rmW_1}$.}
        
        \State  $\texttt{mem}[grad\_b_1] =f_{\text{mul}}(\texttt{mem}[\delta_1],\texttt{mem}[q])$. \Comment{Create $\frac{\partial J}{\partial \rvb_1}$.}
        
        \State  $\texttt{mem}[temp] =f_{\text{mul}}(\texttt{mem}[grad_{W_2}],\texttt{mem}[\eta])$. \Comment{Multiply with step-size.}
        
         \State  $\texttt{mem}[\emb_{w_2}] =f_{\text{sub}}(\texttt{mem}[\emb_{w_2}],\texttt{mem}[temp])$. \Comment{Update $\rmW_2$.}
        
        \State  $\texttt{mem}[temp] =f_{\text{mul}}(\texttt{mem}[grad_{W_1}],\texttt{mem}[\eta])$. \Comment{Multiply with step-size.}
        
         \State  $\texttt{mem}[\emb_{w_1}] =f_{\text{sub}}(\texttt{mem}[\emb_{w_1}],\texttt{mem}[temp])$. \Comment{Update $\rmW_1$.}
        
        \State  $\texttt{mem}[temp] =f_{\text{mul}}(\texttt{mem}[grad_{b_2}],\texttt{mem}[\eta])$. \Comment{Multiply with step-size.}
        
         \State  $\texttt{mem}[\emb_{b_2}] =f_{\text{sub}}(\texttt{mem}[\emb_{b_2}],\texttt{mem}[temp])$. \Comment{Update $\rvb_2$.}
        
        \State  $\texttt{mem}[temp] =f_{\text{mul}}(\texttt{mem}[grad_{b_1}],\texttt{mem}[\eta])$. \Comment{Multiply with step-size.}
        
         \State  $\texttt{mem}[\emb_{b_1}] =f_{\text{sub}}(\texttt{mem}[\emb_{b_1}],\texttt{mem}[temp])$. \Comment{Update $\rvb_1$.}
        
        \end{algorithmic}\label{alg:backprop}}
    \end{algorithm}

\begin{algorithm}[H]{\small
        \caption{Program to simulate SGD using our Unified Attention Based Computer}
        \begin{algorithmic}[1]
            \Require $\texttt{mem}[w_1]=\rmW_1,\texttt{mem}[w_2]=\rmW_2$.
            \Comment{Location weights and biases.}
            \Require $\texttt{mem}[b_1]=\rvb_1,\texttt{mem}[b_2]=\rvb_2$.
            \Comment{Location of biases.}
            \Require  $\texttt{mem}[x_0+i-1]=\dpoint_i$, $i=1,\hdots,\mathcal{D}$. \Comment{Location of the data points.}
            \Require  $\texttt{mem}[y_0+i-1]=y_i$, $i=1,\hdots,\mathcal{D}$. \Comment{Location of the labels.}
            
            \Require $\texttt{mem}[z] = \ve.$\Comment{Indicator for the choice of  loss function}

            \Require  $\emb_{x_*}=x_0$. \Comment{$\emb_{x_*}$ is a pointer to the first data. }%
            \Require  $\emb_{y_*}=y_0$. \Comment{$\emb_{y_*}$ is a pointer to the first label. }
            \Require  $\emb_{\pointer} = \text{instr}_1$. \Comment{Program Counter points to first instruction. }
            \Require $\texttt{mem}[q] = 1$, $\texttt{mem}[p]=0$, $\texttt{mem}[z] = n$. \Comment{Constants.}
            \Require %
            $\texttt{mem}[j] =-\mathcal{D}$.\Comment{Within epoch iteration counter initialized to $-n$.}
            \Require %
            $\texttt{mem}[k] =-T$.\Comment{Epoch counter initialized to $-T$.}
            \vspace{0.5ex}

        \State $\text{Backpropagation}(w_1,w_2,b_1,b_2,\emb_{x_*},\emb_{y_*})$ \Comment{Perform one step of SGD using Backpropagation}
        \State $\texttt{mem}[j] = f_{\text{add}} (\texttt{mem}[j], \texttt{mem}[q])$.\Comment{Increment within epoch iteration counter by 1.}
        \State $\emb_{x_*} = f_{\text{incr\_pointer}}(\emb_{x_*})$. \Comment{Show to next data point.}
        \State $\emb_{y_*} = f_{\text{incr\_pointer}}(\emb_{y_*})$ \Comment{Show to next label.}
        \State if $\texttt{mem}[j]\leq 0$ goto 1. \Comment{Cycle back until all data points are iterated.}
        
        \State $\texttt{mem}[j] = -\mathcal{D}.$ \Comment{Reset counter.}
        
        \State $\emb_{x_*} = f_{\text{reset\_pointer}}(\emb_{x_*}, x_0).$ \Comment{Reset pointer.}
        \State $\emb_{y_*} = f_{\text{reset\_pointer}}(\emb_{y_*}, y_0).$ \Comment{Reset pointer.}
        \State $\texttt{mem}[\text{instr}_3] = f_{\text{reset\_pointer}}(\texttt{mem}[\text{instr}_3], x_0).$ \Comment{Reset pointer.}
        \State $\texttt{mem}[k] = f_{\text{add}} (\texttt{mem}[k], \texttt{mem}[q])$.\Comment{Increment epoch counter by 1.}
        \State if $\texttt{mem}[k]\leq 0$ goto 1. \Comment{Cycle back to the first data point.}
        \State EOF. \Comment{End of File command.}
        \end{algorithmic}\label{alg:sgd}}
    \end{algorithm}

\paragraph{Generalizing to arbitrary depth.}
Our algorithm above is designed to emulate backpropagation on a neural network that contains only one hidden layer. However, it is important to note that this construction can be generalized to networks of arbitrary depth, with the caveat that the length of the code will scale with the number of layers in the network. This is because each line of code in our algorithm represents one cycle of the looped transformer, and the number of cycles required is directly proportional to the depth of the network. It's important to note that the number of cycles of the looped transformer will be equal to the depth of the network. So the cost of this algorithm is proportional to looping the transformer network as many times as the depth of the network. This means that as the network becomes deeper, the computational cost of training it using our algorithm will also increase.

\label{sec:conclusion}
\section{Conclusion and Open Problems}

In this paper, we have shown that transformer networks can be used as universal computers by programming them with specific weights and placing them in a loop. We demonstrate that a constant number of encoder layers can emulate basic computing blocks, such as lexicographic operations, non-linear functions, function calls, program counters, and conditional branches. We construct a one-instruction set computer (OISC) and use it to map iterative algorithms to programs that can be executed by a transformer network. Our results include constant-depth transformers that emulate a basic calculator, a basic linear algebra library, and even a full backpropagation, in-context learning algorithm. Our findings reveal the potential of transformer networks as programmable compute units and offer insight into the mechanics of attention.

Our study sheds light on the versatility of the attention mechanism and how even a single loop can enable the creation of models that can mimic complex iterative algorithms and execute general programs. Our findings also reveal the ability of transformer models to effectively perform intricate mathematical and algorithmic tasks. It is possible that advanced transformer models like GPT-3 use similar internal subroutines when given in-context examples and instructions. In a sense, these models may have the ability to call upon a specific skill or algorithm, similar to a function call, when given contextual examples and instructions. The unique aspect of this is that the programming language of transformers is in natural language, rather than traditional code. This opens up the possibility of using natural language commands to control and program these models, further expanding their potential as programmable computers.

In conclusion, there are several open problems that warrant further exploration in the field of programmable computers using transformer networks. One of the most intriguing possibilities is the potential to fuse hardcoded models with larger pretrained transformers, in order to harness the strengths of both. Additionally, as our constructions currently do not take into account the language aspect of the input, it would be interesting to investigate ways to tokenize input commands in order to map them to natural language.
Another promising avenue for research is the potential for model distillation, in which larger networks could learn the skills performed by these looped transformers. Additionally, experimental validation through the creation of even smaller networks, trained on input-output pairs as well as internal representations, could provide further insight into the capabilities of these designs.
 Finally considering what architecture changes would make the above designs easier to implement and train, could lead to new insights in the field.

 \bibliographystyle{plainnat}
\bibliography{bibliography.bib}

 \newpage
\appendix

\section{Ommited proofs}\label{app:proofs}

\subsection{Addition of pointers.}
\begin{lemma}\label{app:increase}
   There exists a 1-hidden layer feedforward, ReLU network, with $8d$ activations in the hidden layer and $d$ neurons in the output layer that when given two $d$-dimensional binary vectors representing two non-negative integers, can output the binary vector representation of their sum, as long as the sum is less than $2^{d+1}$. 
\end{lemma}
\begin{proof}
For the purpose of explaining this proof, we use the $\{0,1\}^d$ binary representation of the integers, instead of the $\{\pm 1\}^d$ binary representation. However, since the conversion of a bit between the two representations can be done easily using simple affine transformation, the proof will also work for the $\{\pm 1\}^d$ binary representation.

Let the two integers be $a$, $b$ and let $c:=a+b$. We assume that $c<2^{d}$. Futher, let $a_1$ be the least significant bit of $a$, $a_d$ the most significant, and $a_i$ be the $i$-th most significant bit, and similarly for $b$ and $c$. Further, let $a_{[i]}$ represent the integer formed by considering only the least $i$ significant bits of $a$.

Note that $c_i$ is only dependent on the least $i$ bits of $a$ and $b$, and not on the more significant bits of $a$ or $b$.
In particular, $c_i$ only depends on $a_{[i]} + b_{[i]}$.
Define $s := a_{[i]} + b_{[i]}$, and note that $c_i=s_i$. 
Further note that $s<2^{i+1}$ and hence can be represented in $i+1$ bits.
Then, whenever $c_i=1$, there can be two cases: $(s_{i+1}=1, s_{i}=1)$; or $(s_{i+1}=0, s_{i}=1)$.
This can be equivalently written as $c_i = 1$ iff $s\in [2^{i-1}, 2^{i}-1]\cup [3\cdot 2^{i-1}, 2^{i+1}-1]$. 
This can be computed by the following ReLU:
\begin{align*}
    c_i &= (\relu(s - 2^{i-1}+1) - \relu(s - 2^{i-1})) + (\relu(2^{i}-s) - \relu(2^{i}-s-1)) -1\\
    &\quad + (\relu(s - 3\cdot2^{i-1}+1) - \relu(s - 3\cdot2^{i-1})).
\end{align*}

Thus, each bit of $c$ can be computed using 6 neurons. Hence, computing the entire sum needs $8d$ activations, as to substract the  residual.
\end{proof}

\subsection{Non-linear functions as sum of sigmoids }\label{app:functions}

\begin{lemma}\label{lem:app_sigmoids}
 Consider an input of the form 
\begin{equation}
    \Input = \begin{bmatrix}
        \ve & \zero&\rvx&\zero &\zero &\zero \\
        \zero  &\zero & \zero& \zero &\zero&\zero\\
        \zero &\zero &\emb_{2d+1} &\zero &\zero &\zero\\
        \emb_1&\emb_{2:d}&\zero&\emb_{d+2:2d} &\emb_{2d+1} &\emb_{2d+2:3d}\\
        0 & 0_{2:d} & 1 & 0_{d+2:2d} & 0 & 0_{2d+2:3d}
    \end{bmatrix}\in\R^{N+d_x\times 3d} . \nonumber
\end{equation}
where $d$ is chosen, $N$ is the number of functions we encode and $d_x$ is the dimension of $\rvx$. $\ve = \rve_j$ an indicator vector of the function we want to choose.  Then there  exists a   transformer-based function block with 3 layers, $m$ heads and dimensionality $O(d)$ such that  \begin{equation}
    f(\Input) = \begin{bmatrix}
        *&*&*&*&\sum_{i=1}^m c_{ji}\sigmoid(\rvx^T\rva_{ji}) & *\\
        \zero  &\zero & \rvx& \zero &\zero&\zero\\
         \zero &\zero &\emb_{2d+1} &\zero &\zero &\zero\\
       \emb_1&\emb_{2:d}&\zero&\emb_{d+2:2d} &\emb_{2d+1} &\emb_{2d+2:3d}\\
        0 & 0_{2:d} & 1 & 0_{d+2:2d} & 0 & 0_{2d+2:3d}
    \end{bmatrix} \nonumber
\end{equation}
where $*$ denoted inconsequential values that will be ignored downstream. 
\end{lemma}
\begin{proof} 
The first thing we do is to move the $\rvx$ to the second row block, as follows:
\begin{equation*}
    \Input = \begin{bmatrix}
        \ve & \zero&\rvx &\zero &\zero &\zero\\
        \zero  &\zero & \zero& \zero &\zero&\zero\\
        \zero &\zero &\emb_{2d+1} &\zero &\zero &\zero\\
        \emb_1&\emb_{2:d}&\zero&\emb_{d+2:2d} &\emb_{2d+1} &\emb_{2d+2:3d}\\
        0 & 0_{2:d} & 1 & 0_{d+2:2d} & 0 & 0_{2d+2:3d}
    \end{bmatrix} \to \begin{bmatrix}
        \ve & \zero&\zero &\zero &\zero &\zero\\
        \zero  &\zero & \rvx & \zero &\zero&\zero\\
        \zero &\zero &\emb_{2d+1} &\zero &\zero &\zero\\
        \emb_1&\emb_{2:d}&\zero&\emb_{d+2:2d} &\emb_{2d+1} &\emb_{2d+2:3d}\\
        0 & 0_{2:d} & 1 & 0_{d+2:2d} & 0 & 0_{2d+2:3d}
    \end{bmatrix}
\end{equation*}
This can be done using a ReLU feedforward layer that performs this using the last row of the input as the indicator bit for the column containing $\dpoint$.

Then we want to create the following transformation
\begin{equation*}
    \begin{bmatrix}
        \ve & \zero&\zero &\zero &\zero &\zero\\
        \zero  &\zero & \rvx& \zero &\zero&\zero\\
        \zero &\zero &\emb_{2d+1} &\zero &\zero &\zero\\
        \emb_1&\emb_{2:d}&\zero&\emb_{d+2:2d} &\emb_{2d+1} &\emb_{2d+2:3d}\\
         0 & 0_{2:d} & 1 & 0_{d+2:2d} & 0 & 0_{2d+2:3d}
    \end{bmatrix} \xrightarrow[]{} \begin{bmatrix}
        *&*&*&*&\sum_{i=1}^m c_{ji}\sigmoid(\rvx^T\rva_{ji}) & *\\
        \zero  &\zero & \rvx& \zero &\zero&\zero\\
         \zero &\zero &\emb_{2d+1} &\zero &\zero &\zero\\
       \emb_1&\emb_{2:d}&\zero&\emb_{d+2:2d} &\emb_{2d+1} &\emb_{2d+2:3d}\\
        0 & 0_{2:d} & 1 & 0_{d+2:2d} & 0 & 0_{2d+2:3d}
    \end{bmatrix}
\end{equation*}
The proof follows that of \cref{lem:sigmoids}. We again ignore the last three rows by setting the corresponding rows in the key, query and values weight matrices to be zero. Let 
\begin{equation}
         \query^{i} = \begin{bmatrix}
         \zero &\id_d\\
         \zero &\zero 
     \end{bmatrix}, \key^i = \begin{bmatrix}
         [\rva_{1i} \; \hdots \; \rva_{Ni}] & \zero\\
         \zero &\zero
     \end{bmatrix} , \val^i = \begin{bmatrix}
        [c_{1i} \;\hdots \; c_{Ni}] &\zero \\
        \zero  &\zero
     \end{bmatrix} \nonumber
\end{equation}
We note that for the purpose of this proof, each $\rva_i$ has one extra element at the end equal to $-\log(3d-1)$, while the vectors $\rvx$ will have the last element equal to one. Then we will have
\begin{align*}
    \softmax((\key^i\Input)^T(\query^i\Input)) &= \begin{bmatrix}
\rva_{ji}^\top & \zero \\ 
\zero  & \zero  \\ 
\zero  & \zero  \\ 
\zero  & \zero  \\ 
\zero  & \zero  \\ 
\zero  & \zero  \\ 
\end{bmatrix}\begin{bmatrix}
\zero &\zero &\rvx &\zero &\zero &\zero\\
\zero &\zero &\zero &\zero &\zero &\zero\\
\end{bmatrix}\\
&= \begin{bmatrix}
    \zero &\zero &\rva_{ji}^\top\rvx&\zero &\zero &\zero\\
    \zero &\zero &\zero &\zero &\zero &\zero\\
    \zero &\zero &\zero &\zero &\zero &\zero\\
    \zero &\zero &\zero &\zero &\zero &\zero\\
    \zero &\zero &\zero &\zero &\zero &\zero\\
    \zero &\zero &\zero &\zero &\zero &\zero\\
\end{bmatrix}\\
    &=\begin{bmatrix}
* & * &\sigmoid(\rvx^T\rva_{ji}) & * & *   &*\\ 
* & * &*  & * & * &*\\ 
* & * &*  & * & * &*\\ 
* & * &*  & * & * &*\\ 
* & * &*  & * & * &*\\ 
* & * &*  & * & * &*\\ 
\end{bmatrix}
\end{align*}
since $\rva_{ji}^\top\rvx = \rva_{ji}^\top\rvx -\log{3d-1}$ and thus $e^{\rva_{ji}^\top\rvx}/(3d-1 + e^{\rva_{ji}^\top\rvx}) = \sigmoid(\rva_{ji}^\top\rvx)$ with a slight abuse of notation over the inner product $\rva_{ji}^\top\rvx$ to account for the extra corrections bias term. Thus, 
\begin{equation}
 \val\Input\softmax( (\key\Input)^T(\query\Input) ) = \begin{bmatrix}
 * & * &c_{ji}\sigmoid(\rvx^T\rva_{ji}) & * & *   &*\\
  \zero &\zero &\zero &\zero &\zero &\zero\\
    \zero &\zero &\zero &\zero &\zero &\zero\\
    \zero &\zero &\zero &\zero &\zero &\zero\\
    \zero &\zero &\zero &\zero &\zero &\zero\\
 \end{bmatrix} \nonumber
\end{equation}
By summing over all heads and adding the residual we get
\begin{equation}
    \begin{bmatrix}
        *&*&\sum_{i=1}^m c_{ji}\sigmoid(\rvx^T\rva_{ji})&*&* & *\\
        \zero  &\zero & \rvx& \zero &\zero&\zero\\
         \zero &\zero &\emb_{2d+1} &\zero &\zero &\zero\\
       \emb_1&\emb_{2:d}&\zero&\emb_{d+2:2d} &\emb_{2d+1} &\emb_{2d+2:3d}\\
        0 & 0_{2:d} & 1 & 0_{d+2:2d} & 0 & 0_{2d+2:3d}
    \end{bmatrix} \nonumber
\end{equation}
Finally, we use an extra layer similarly to \cref{lem:write} to write the result in the desired output. Hence, we get
\begin{equation}
    \begin{bmatrix}
        *&*&*&*&\sum_{i=1}^m c_{ji}\sigmoid(\rvx^T\rva_{ji}) & *\\
        \zero  &\zero & \rvx& \zero &\zero&\zero\\
         \zero &\zero &\emb_{2d+1} &\zero &\zero &\zero\\
       \emb_1&\emb_{2:d}&\zero&\emb_{d+2:2d} &\emb_{2d+1} &\emb_{2d+2:3d}\\
        0 & 0_{2:d} & 1 & 0_{d+2:2d} & 0 & 0_{2d+2:3d}
    \end{bmatrix} \nonumber
\end{equation} 
\end{proof}

However, we have another way of controlling the input, which is by the size of each attention mechanism, that is directly controlled by the dimension $d$ of the embedding.
\begin{lemma}\label{lem:sigmoids_alt}
    Consider an input of the form 
\begin{equation}
    \Input = \begin{bmatrix}
        \vx&\dots&\vx\\
        0&\dots&0&\\
         \vOne-\ve_1&\dots&\vOne-\ve_m\\
        \ve_1&\dots&\ve_m
    \end{bmatrix}
\end{equation}
where  $\ve_i$ is the one hot vector with $1$ in the $i-$th position and $\vx\in\R^d$. Let $m$ be the number of sigmoids we need to represent a function, then there  exists a one layer transformer with 1 head such that
\begin{equation}
    \att(\Input) = \begin{bmatrix}
                 \vx&\dots&\vx\\
\sigma(\rva_1^\top\vx)&\dots&\sigma(\rva_m^\top\vx)\\
                \vOne-\ve_1&\dots&\vOne-\ve_m\\
        \ve_1&\dots&\ve_m
     \end{bmatrix}
\end{equation}
\end{lemma}

\begin{proof}
     Let
     \begin{equation}
         \key=  \begin{bmatrix}
                \vZero^\top&0&\vZero^\top&\ve_1^\top\\
                \vdots& \vdots&\vdots& \vdots\\
                \vZero^\top&0&\vZero^\top&\ve_{m}^\top
     \end{bmatrix}, \query =\begin{bmatrix}
                \rva_1^\top&0&-C\ve_1^\top&\vZero^\top\\
                \vdots& \vdots& \vdots& \vdots\\
                \rva_m^\top&0&-C\ve_m^\top&\vZero^\top\\
                \vZero^\top&0&\vZero^\top&\vZero^\top
     \end{bmatrix}
     \end{equation}
     Hence,
     \begin{equation}
         \key\Input= \id_d, \query\Input = \begin{bmatrix}
                \rva_1^\top\vx&-C+\rva_1^\top\vx&\dots&-C+\rva_1^\top\vx\\
                -C+\rva_2^\top\vx&\rva_2^\top\vx&\dots&-C+\rva_2^\top\vx\\
                \vdots& \vdots&\vdots& \vdots\\
                -C+\rva_m^\top\vx&-C+\rva_m^\top\vx&\dots&\rva_m^\top\vx\\
                0&0&\dots&0
     \end{bmatrix}
     \end{equation}
     After applying softmax we get,
     \begin{align*}
         \sigma_s((\key\Input)^\top \query\Input) \approx \begin{bmatrix}
                \sigma(\rva_1^\top\vx)&0&\dots&0\\
                0&\sigma(\rva_2^\top\vx)&\dots&0\\
                \vdots& \vdots&\vdots& \vdots\\
                0&0&\dots&\sigma(\rva_m^\top\vx)\\
                *&0&\dots&*
     \end{bmatrix},
     \end{align*}
     for large enough $C$. Next we set 
     \begin{align*}
         \val =\begin{bmatrix}
             \zero &0 &\zero & 0 &\hdots &0\\
             \zero &0 &\zero &c_{1} &\hdots &c_m\\
             \zero &0 &\zero & 0 &\hdots &0\\
             \zero &0 &\zero & 0 &\hdots &0
         \end{bmatrix}
     \end{align*}
     thus resulting in 
     \begin{align*}
         \val\Input = \begin{bmatrix}
         0&0&\dots&0 & 0\\
              c_1&c_2&\dots&c_m &0\\
              0&0&\dots&0 & 0\\
              \vdots&\vdots&\vdots&\vdots & \vdots\\
0&0&\dots&0 & 0
     \end{bmatrix}%
     \end{align*}

     Hence, we get
     \begin{align*}
         \val\Input\sigma_s((\key\Input)^\top \query\Input) = \begin{bmatrix}
                 0&\dots&0\\
                 c_1\sigma(\rva_1^\top\vx)&\dots&c_m\sigma(\rva_m^\top\vx)\\
                0&\dots&0\\
                \vdots& \vdots&\vdots& \vdots\\
            0&\dots&0\\
                0&\dots&0
     \end{bmatrix},
     \end{align*}
     and 
        \begin{align*}
        \Input + \val\Input\sigma_s((\key\Input)^\top \query\Input) = \begin{bmatrix}
                 \vx&\dots&\vx\\
                 c_1\sigma(\rva_1^\top\vx)&\dots&c_m\sigma(\rva_m^\top\vx)\\
                \vOne-\ve_1&\dots&\vOne-\ve_m\\
        \ve_1&\dots&\ve_m
     \end{bmatrix}.
     \end{align*}
\end{proof}

\begin{corollary}\label{cor:lem_5_alt}
    Consider an input of the form 
\begin{equation}
    \Input = \begin{bmatrix}
        \vx&\dots&\zero\\
        0&\dots&0&\\
         \vOne-\ve_1&\dots&\vOne-\ve_m\\
        \ve_1&\dots&\ve_m
    \end{bmatrix}
\end{equation}
where  $m$ is the number of sigmoids we use and $\ve_i$ is an indicator vector and $\rvx\in\R^d$; then there  exists a 3 layer transformer with 1 head such that
\begin{equation}
    \att(\Input) = \begin{bmatrix}
                 \sum_{i=1}^m\sigma(\rva_i^\top\vx)&\dots&\sum_{i=1}^m\sigma(\rva_i^\top\vx)\\
                \zero&\dots&\zero
     \end{bmatrix}
\end{equation}
\end{corollary}

\begin{proof}
Given the input
    \begin{equation}
    \Input = \begin{bmatrix}
        \vx&\dots&\zero\\
        0&\dots&0&\\
         \vOne-\ve_1&\dots&\vOne-\ve_m\\
        \ve_1&\dots&\ve_m
    \end{bmatrix},
\end{equation}
we set the query and key matrices as follows:
\begin{align*}
    \key = \query = \begin{bmatrix}
        \zero^\top &0&\vOne&\vOne
    \end{bmatrix}.
\end{align*}
Then, we get 
\begin{align*}
    (\key\Input)^\top\query\Input = \begin{bmatrix}
        d&\dots&d\\
        \vdots&\dots&\vdots\\
        d&\dots&d
    \end{bmatrix}.
\end{align*}
Setting the value matrix to 
\begin{align*}
\begin{bmatrix}
    d\id&\zero&\zero&\zero\\
    \zero&\zero&\zero&\zero\\
    \zero&\zero&\zero&\zero\\
    \zero&\zero&\zero&\zero
\end{bmatrix},
\end{align*}
we get 
\begin{align*}
    \val \Input \softmax((\key \Input)^\top \query \Input) = \begin{bmatrix}
        \vx&\dots&\vx\\
        0&\dots&0&\\
        \zero&\dots&\zero\\
        \zero&\dots&\zero
    \end{bmatrix}.
\end{align*}
Hence, the output of the attention layer is:
\begin{align*}
   \Input + \val \Input \softmax((\key \Input)^\top \query \Input) = \begin{bmatrix}
       2\vx&\dots&\vx\\
        0&\dots&0&\\
\vOne-\ve_1&\dots&\vOne-\ve_m\\
        \ve_1&\dots&\ve_m
    \end{bmatrix}.
\end{align*}
Note that using the embeddings in the last rows and a feedforward network can be used to produce the following 

\begin{align*}
\begin{bmatrix}
       \vx&\dots&\vx\\
        0&\dots&0&\\
\vOne-\ve_1&\dots&\vOne-\ve_m\\
        \ve_1&\dots&\ve_m
    \end{bmatrix}.
\end{align*}

Now, passing this into the transformer of \cref{lem:sigmoids_alt} will result in 
\begin{equation}
    \att(\Input) = \begin{bmatrix}
                 \vx&\dots&\vx\\
                 c_1\sigma(\rva_1^\top\vx)&\dots&c_m\sigma(\rva_m^\top\vx)\\
                \vOne-\ve_1&\dots&\vOne-\ve_m\\
        \ve_1&\dots&\ve_m
     \end{bmatrix}.
\end{equation}

For the third layer, we set the key and query matrices as follows
\begin{align*}
    \key = \query = \begin{bmatrix}
        \zero^\top &0&\vOne&\vOne
    \end{bmatrix}.
\end{align*}
Then, we get 
\begin{align*}
    (\key\Input)^\top\query\Input = \begin{bmatrix}
        m&\dots&m\\
        \vdots&\dots&\vdots\\
        m&\dots&m
    \end{bmatrix}.
\end{align*}
Setting the value matrix to 
\begin{align*}
\begin{bmatrix}
    \zero&\zero&\zero&\zero\\
    \zero&m&\zero&\zero\\
    \zero&\zero&\zero&\zero\\
    \zero&\zero&\zero&\zero
\end{bmatrix},
\end{align*}
we get 
\begin{align*}
    \val \Input \softmax((\key \Input)^\top \query \Input) = \begin{bmatrix}
        \zero&\dots&\zero\\
\sum_{i=1}^mc_i\sigma(\rva_i^\top\vx)&\dots&\sum_{i=1}^mc_i\sigma(\rva_i^\top\vx)\\
        \zero&\dots&\zero\\
        \zero&\dots&\zero
    \end{bmatrix}.
\end{align*}
Hence, the output of the attention layer is:
\begin{align*}
   \Input + \val \Input \softmax((\key \Input)^\top \query \Input) = \begin{bmatrix}
       \vx&\dots&\vx\\
        \sum_{i=1}^mc_i\sigma(\rva_i^\top\vx)&\dots&\sum_{i=1}^mc_i\sigma(\rva_i^\top\vx)\\
\vOne-\ve_1&\dots&\vOne-\ve_m\\
        \ve_1&\dots&\ve_m
    \end{bmatrix}.
\end{align*}
Finally, the feedforward layers can be used to move the results to the first row.
\end{proof}
\subsection{Matrix Transposition}\label{app:transpose}

\begin{lemma} 
Fix $\epsilon >0$ and consider an input of the following form
\begin{equation}
    \Input = \left[\begin{array}{c|c|c|cc}
        \rmA  &\zero&\zero &\dots & \zero \\
 \zero&      \zero&\zero &\dots & \zero\\
         \emb_{1:d}&\emb_{1:d}&\emb_{1:d}& \dots  &\emb_{1:d}\\
         \rmP_1'&\rmP_{2}'&\rmP_{3}'& \dots  &\rmP_{d}'
    \end{array}\right]. \nonumber
\end{equation}
where $\rmA\in\R^{d\times d}$; then there exists transformer-based function block with 4 layers, 1 head and dimensionality $r = 2d+2\log d = O(d)$ that outputs the following matrix
\begin{equation}
    \Input = \left[\begin{array}{c|c|c|cc}
\rmA' &\rmA'&\rmA'&\dots &\rmA' \\
        \zero&\zero&\zero &\dots &\zero \\
         \emb_{1:d}& \emb_{1:d}& \emb_{1:d}& \dots &\emb_{1:d} \\
        \rmP_1'&\rmP_{2}'&\rmP_{3}'& \dots  &\rmP_{d}'
    \end{array}\right]. \nonumber
\end{equation}
where $\rmA' = \rmA^\top +\epsilon\rmM$, for some $\norm{\rmM}\leq 1$.
\end{lemma}
\proof
We can vectorize the matrix $\rmA$ into a $d^2$ dimensional vector using the attention mechanism, as shown in Eq. \eqref{eq:vectorized_matrix}. 
Notice that once we have the matrix in this form we can implement its transpose with a fixed permutation of the columns of the matrix to get the vectorized form of $\rmA^\top$. Once we have the transpose in vector form, we matricize it back to get the matrix transform using the attention mechanism. We explain the details of this process below:

\emph{Vectorization: } We assume that the input is of the following form, where $\rmA$ is the matrix to be vectorized.
\begin{equation}
    \Input = \left[\begin{array}{cccc}
        \rmA  &\zero&\dots & \zero\\
 \zero&      \zero&\dots & \zero\\
         \emb_{1:d}&\emb_{1:d}& \dots  &\emb_{1:d}\\
         \rmP_1'&\rmP_{2}'& \dots  &\rmP_{d}'
    \end{array}\right]. \nonumber
\end{equation}
Here, $\rmP_i'$ represents a matrix of $d$ columns, where each column is $\emb_i$.

The first layer uses the $\emb_{1:d}$ encodings to make $d$ copies of the matrix $\rmA$, as follows:
\begin{equation}
    \Input = \left[\begin{array}{cccc}
\rmA&\zero &\dots &\zero \\
 \rmA &\rmA&\dots & \rmA\\
         \emb_{1:d}&\emb_{1:d}& \dots &\emb_{1:d} \\
         \embed_1'&\embed_2'& \dots &\embed_{d}' 
    \end{array}\right]. \nonumber
\end{equation}

The feed forward part of the second layer then uses the encodings $\emb_i'$ to vectorize the matrix in the second row block as follows:

\begin{equation}
    \Input = \left[\begin{array}{ccc}
    \rmA&\dots&\zero\\
\begin{bmatrix}
    A_{(1,1)}&\dots &A_{(1,d)}\\
    \zero&\dots&\zero\\
\end{bmatrix} &\dots &\begin{bmatrix}
    A_{(d,1)}&\dots &A_{(d,d)}\\
    \zero&\dots&\zero\\
\end{bmatrix} \\
         \emb_{1:d}& \dots &\emb_{1:d} \\
         \embed_1'& \dots &\embed_{d}' 
    \end{array}\right].\label{eq:vectorized_matrix} 
\end{equation}
This is achieved, by explicitly defining a neural network that keeps the $i-$th row if the corresponding encoding is $\embed_i'$ and place it in the $d+1$ row.

\emph{Transposition in the vector form:}
Once we have the matrix vectorized as the second row block of the scratchpad, the following key and query matrices 
\begin{align*}
    \key = \begin{bmatrix}
        \zero &\zero & \id& \zero\\
        \zero &\zero & \zero& \id
    \end{bmatrix}, \query = \begin{bmatrix}
        \zero &\zero & \zero& \id\\
        \zero &\zero & \id&\zero
    \end{bmatrix},  
\end{align*}
results in the head outputting the following, which is the vectorized form of $\rmA^\top$ (in the second row block)
\begin{align*}
    \Input \softmax((\key \Input)^\top (\query \Input)) = \left[\begin{array}{ccc}
     *&\dots&*\\
\begin{bmatrix}
    A_{(1,1)}&\dots &A_{(d,1)}\\
    \zero&\dots&\zero\\
\end{bmatrix} &\dots &\begin{bmatrix}
    A_{(1,d)}&\dots &A_{(d,d)}\\
    \zero&\dots&\zero\\
\end{bmatrix} \\
         \embed_1'& \dots & \embed_{d}'\\
         \emb_{1:d}& \dots &\emb_{1:d} 
    \end{array}\right].
\end{align*} Then, using the following value matrix gives
\begin{align*}
    \val = \begin{bmatrix}
    \zero&\zero & \zero& \zero\\
        \zero&\id & \zero& \zero\\
        \zero&\zero & \zero& \zero\\
        \zero&\zero & \zero& \zero
    \end{bmatrix},
\end{align*}
\begin{align*}
  \val \Input \softmax((\key \Input)^\top (\query \Input)) = \left[\begin{array}{ccc}
  \zero&\dots&\zero\\
\begin{bmatrix}
    A_{(1,1)}&\dots &A_{(d,1)}\\
    \zero&\dots&\zero\\
\end{bmatrix} &\dots &\begin{bmatrix}
    A_{(1,d)}&\dots &A_{(d,d)}\\
    \zero&\dots&\zero\\
\end{bmatrix} \\
         \zero& \dots & \zero\\
         \zero& \dots &\zero 
    \end{array}\right],
\end{align*}
Adding back the $\Input$ (see \eqref{eq:TF}), results in 
\begin{equation}
    \Input + \val \Input \softmax((\key \Input)^\top (\query \Input)) = \left[\begin{array}{ccc}
    \rmA&\dots&\zero\\
\begin{bmatrix}
    A_{(1,1)}&\dots &A_{(d,1)}\\
    \zero&\dots&\zero\\
\end{bmatrix} &\dots &\begin{bmatrix}
    A_{(1,d)}&\dots &A_{(d,d)}\\
    \zero&\dots&\zero\\
\end{bmatrix} \\
         \emb_{1:d}& \dots &\emb_{1:d} \\
         \embed_1'& \dots &\embed_{d}' 
    \end{array}\right]. \nonumber
\end{equation}

Using the feedforward layers and the encodings $\embed_i'$, we get
\begin{equation}
    \Input = \left[\begin{array}{ccc}
    \rmA&\dots&\zero\\
\begin{bmatrix}
    A_{(1,1)}&\dots &A_{(d,1)}\\
    \zero&\dots&\zero
\end{bmatrix} &\dots &\begin{bmatrix}
    \zero&\dots&\zero\\
     A_{(1,d)}&\dots &A_{(d,d)}
\end{bmatrix}\\
         \emb_{1:d}& \dots &\emb_{1:d} \\
         \embed_1'& \dots &\embed_{d}' 
    \end{array}\right]. \nonumber
\end{equation}

Using an attention layer and the first row of encodings, we get 
\begin{equation}
    \Input = \left[\begin{array}{ccc}
\rmA^\top &\dots &\rmA^\top \\
        \zero &\dots &\zero \\
         \emb_{1:d}& \dots &\emb_{1:d} \\
         \embed_1'& \dots &\embed_{d}' 
    \end{array}\right]. \nonumber
\end{equation}

\subsection{Matrix Multiplication by Linearizing the Softmax}\label{app:matrixmul}

 We will show how we can implement matrix multiplication so that it will fit our unified template. To do so, we need to show for example for the result of $\rmA^\top\rmB$ , where $\rmA\in\R^{k\times m}$ and $\rmB\in\R^{k\times n}$ with $k,m,n <d$ we can achieve the following:
\begin{equation}
    \bracks*{\begin{array}{cc|cc|cc}
        \rmA & \zero &\rmB &\zero &\zero\\
        \zero& \zero &\zero&\zero&\zero
    \end{array}}\xrightarrow[]{}\bracks*{\begin{array}{cc|cc|cc}
        *& *&*&* &\rmA^\top\rmB &*\\
        \zero& \zero &\zero&\zero&\zero &\zero 
    \end{array}} \nonumber
\end{equation}
\begin{lemma}\label{lem:matrixmul_app}
Let $\rmA\in \R^{k\times m}$ and $\rmB \in \R^{k\times n}$; then for any $\epsilon >0$ there exists a transformer-based function block with  2 layers, 1 head and dimensionality $r = O(d)$ that  outputs the multiplication $\rmA^T\rmB^T +\eps\rmM$, for some $\norm{\rmM}\leq 1$ .
\end{lemma}
\begin{corollary}
Let $\rmA\in \R^{k\times m}$ and $\rmB \in \R^{k\times n}$; then for any $\epsilon >0$ there exists a transformer-based function block with  2 layers, 1 head and dimensionality $r = O(d)$ that  outputs the multiplication $\rmB^\top\rmA +\eps\rmM$, for some $\norm{\rmM}\leq 1$ .
\end{corollary}
\begin{corollary}
Let $\rmA\in \R^{k\times m}$ and $\rmB \in \R^{k\times n}$; then for any $\epsilon >0$ there exists a transformer-based function block with  2 layers, 1 head and dimensionality $r = O(d)$ that  outputs the multiplication $\rmB^\top\rmB +\eps\rmM$, for some $\norm{\rmM}\leq 1$ .
\end{corollary}
\begin{corollary}
Let $\rmA\in \R^{k\times m}$ and $\rmB \in \R^{k\times n}$; then for any $\epsilon >0$ there exists a transformer-based function block with  2 layers, 1 head and dimensionality $r = O(d)$ that  outputs the multiplication $\rmA^\top\rmA +\eps\rmM$, for some $\norm{\rmM}\leq 1$ .
\end{corollary}
We will prove just the first of these results and the rest are a simple corollary of it.
\begin{proof}
Let $\rmM\in \R^{2d\times 2d}$, $\rmA\in\R^{k\times m}$ and $\rmB\in\R^{k\times n}$ be the following matrices:
\begin{align*}
    \rmM = \begin{bmatrix}
        \rmA &\zero&\rmB  &\zero \\
        \zero&\zero &\zero &\zero
    \end{bmatrix}.
\end{align*}
The zeros pad the rows and columns to ensure that the matrix $M$ is $2d\times 2d$.
Then, consider the input matrix to be of the following form:
\begin{equation}
    \Input = \begin{bmatrix}
        \rmM& \zero &\zero \\
        \zero & \rvOne\rvOne^\top &\zero\\
        \id&\zero &\zero\\ 
        &\emb^{(1)}&\\
        &\emb^{(2)}&\\
        \zero &\rvOne^T &\zero 
    \end{bmatrix}  \nonumber
\end{equation} 
 where $\rvOne\in\R^{2d}$ is the all ones vector. The identity matrix $\id$ and the all ones matrix $\rvOne\rvOne^\top$ are part of the design of the input and they are always fixed. For now we ignore the  encodings and the last row, by setting the corresponding rows of the key,query and value weight matrices to be zero. These rows will be used to copy the output to the  place that we want. %

Focusing on the rest of the rows, we set the key  and query weight matrices to be 
\begin{equation}
\key = \id, \query = \begin{bmatrix}
     c\id & \zero &\zero\\
    \zero &\zero &C\id\\
     \zero & \id &\zero
\end{bmatrix}, \val = \begin{bmatrix}
\zero & \zero &\zero\\
\zero & \zero &ne^C\cD_d\\
\zero &\zero &\zero
\end{bmatrix} \nonumber
\end{equation}
where $\cD_d\in \R^{2d\times 2d}$ is the diagonal matrix with the first $d$ diagonal elements $1$, and the rest 0.
Thus we have
\begin{align}
    (\key\Input)^\top\query\Input &= \begin{bmatrix}
        \rmM& \zero &\zero \\
        \zero & \rvOne\rvOne^\top &\zero\\
        \id&\zero &\zero\\
    \end{bmatrix}^\top \begin{bmatrix}
        c\rmM & \zero &\zero\\
        C\id &\zero &\zero\\
        \zero &\rvOne\rvOne^\top &\zero 
    \end{bmatrix}\\
    &= \begin{bmatrix}
        c\rmM^\top\rmM &\rvOne\rvOne^\top &\zero\\
        C\rvOne\rvOne^\top &\zero &\zero\\
        &\zero &\zero &\zero 
    \end{bmatrix}
\end{align}
Each of the first $2d$ columns above looks as follows
\begin{equation}
    \begin{bmatrix}
      cz_{1i} &cz_{2i} &\dots &cz_{ni} & C\rvOne^\top &\zero
    \end{bmatrix} \nonumber
\end{equation}
 After we apply the softmax $\sigma_s$ per column, we get 
 \begin{equation}
   \sigma_s(cz_{ij}) = \dfrac{e^{cz_{ij}}}{\sum_{j=1}^n e^{cz_{ij}} + n(e^{C}+1) } \nonumber
\end{equation}
where $n=2d$, $z_{ij}$  is the $(i,j)$ element of the matrix $\rmM^\top\rmM$. Let $\ell(\cdot)$ be the transformation above then we have

\begin{align*} \val\Input\softmax((\key\Input)^\top\query\Input) &= \begin{bmatrix}
    \zero &\zero &\zero\\
    ne^C\cD_d &\zero &\zero\\
    \zero &\zero &\zero 
    \end{bmatrix}\begin{bmatrix}
        \ell(c\rmM^\top\rmM) &* &*\\
        * &* &*\\
        *&*&*
    \end{bmatrix}\\
    &=\begin{bmatrix}
        \zero &\zero &\zero\\
        ne^C\cD_d\ell(c\rmM^\top\rmM) &* &*\\
        \zero &\zero &\zero
    \end{bmatrix}\\
    &\approx \begin{bmatrix}
        \zero &\zero &\zero\\
        \rvOne\rvOne^\top + c\rmM^\top\rmM &* &*\\
        \zero &\zero &\zero 
    \end{bmatrix}
\end{align*}
and by adding back the residual we have
\begin{equation}
    \Input = \begin{bmatrix}
        \rmM &\zero &\zero\\
        \rvOne\rvOne^\top + c\rmM^\top\rmM &* &*\\
        \id&\zero &\zero 
    \end{bmatrix} \nonumber
\end{equation}
for small enough $c$ and large enough $C$. This is because 
\begin{align*}
    ne^C\dfrac{e^{cx_{ij}}}{\sum_{j=1}^n e^{cx_{ij}} + n(e^{C} +1)}  &= e^{cx_{ij}}\dfrac{1}{1+ \sum_{j=1}^ne^{cx_{ij}-C-\log n}+n}  \\
    &= (1 + cx_{ij} + O((cx_{ij})^2))( 1 - e^{cx_{ij}-C-\log n} + O( e^{2(cx_{ij}-C-\log n)}))\\
    &= (1 + cx_{ij} + O((cx_{ij})^2))( 1 - e^{cx_{ij}-C-\log n} )\\
    &\approx (1 + cx_{ij})
\end{align*}
We now use the feedforward layers to perform the following transform 
\begin{align*}
\Input &= \begin{bmatrix}
* & * & * \\
 \rmM^\top\rmM&*&*\\
* &*  &*
\end{bmatrix} \\ 
&= \begin{bmatrix}
    * & * & *  &* &* \\
    \rmA^\top \rmA &\zero &\rmA^\top\rmB &\zero & *\\
    \zero &\zero &\zero &\zero &*\\
    \rmB^\top\rmA &\zero &\rmB^\top\rmB &\zero&*\\
    \zero &\zero &\zero &\zero &*\\
    *&*&*&*&*
\end{bmatrix}
\end{align*}
Now if $\emb^{(1)} = \begin{bmatrix} \zero &\zero & \emb_{2d+1:2d+n}& \zero &\zero  \end{bmatrix}$ and  $\emb^{(2)} = \begin{bmatrix}
\emb_{1:n} &\emb_{n+1:d} &\zero &\emb_{d+n+1:2d} &\emb_{2d:3d}
\end{bmatrix}$ we can copy $\rmA^\top\rmB$ to the desired place using \cref{lem:read}.
\end{proof}

\section{Error Analysis}\label{app:error}

In all of this section we assume that each element of the input matrix $\rmX$ has values $v_i$ bounded by some constant $G$, \ie $\abs{v_i}\leq G$.
\paragraph{The error in the \texttt{read}/ \texttt{write} operation.}

The positional encodings as we have already mentioned have the following properties: 
$\emb_i$ is an $\log(n)$ dimensional $\pm 1$ vector which is the binary representation of $i$ with $-1$ in the place of $0$. %
Hence, we have $\emb_i^\top\emb_i = \log(n)$ and each $\emb_i^\top\emb_j < \log(n)$ for $i\neq j$. 

Each time a copy is implemented from one column to another, we create a permutation matrix (a matrix of zeros and ones) which then multiplies the input matrix $\Input\in\R^{d\times n}$ from the right and results in permutations of the column space. We thus focus on just one column of the $n\times n$ matrix that is created after we apply the softmax.
Let $\rvz$ be this column of the matrix, ideally  we want to output in one   position $1$ and in the rest $0$. In the place that we want to output $1$, say the $a-$th position, we have the inner product $\rvz_a = \emb_i^\top\emb_i$ for some $i\in[n]$. The rest of the elements in the same column would be $\rvz_b \leq \emb_i^\top\emb_j$  for $i\neq j$ and $a\neq b$. Then, %
\begin{align*}
    \bracks{\softmax((\key\Input)^\top \query\Input) }_{i,i}&= 
\dfrac{e^{\lambda \emb_i^\top\emb_i}}{e^{\lambda \emb_i^\top\emb_i} + \sum_{j\neq i} e^{\lambda \emb_i^\top\emb_j} }\\
&= \dfrac{1}{1 + \sum_{j\neq i} e^{\lambda \emb_i^\top\emb_j}/e^{\lambda \emb_i^\top\emb_i} }
\end{align*}
Since $\lambda\emb_i^\top\emb_j < \lambda \emb_i^\top\emb_i -\lambda$ for $i\neq j$, we have that 
\begin{align*}
    \bracks{\softmax((\key\Input)^\top \query\Input) }_{i,i}&\geq \dfrac{1}{1+ ne^{-\lambda}}\\
    &\geq \dfrac{1}{1+ e^{\log n-\lambda}}\\
    &\geq 1 - \dfrac{e^{\log n-\lambda}}{1+ e^{\log n-\lambda}}\\
    &\geq 1 - e^{\log n-\lambda}
\end{align*}
Thus, for  $i\neq j$, $\bracks{\softmax((\key\Input)^\top \query\Input) }_{i,j}\leq  e^{\log n-\lambda}$. This implies that there exist $\eps_i$, $i=1,\hdots,n$ such that 
\begin{align*}
    \rvz_a &= 1-\varepsilon_a, \text{ for some } \varepsilon_a \leq e^{\log n -\lambda}\\
    \rvz_b &= \varepsilon_b \text{ for }b\neq a \text{ and for some }\varepsilon_b\leq e^{\log n -\lambda}
\end{align*}%
Hence, we have that 
\begin{equation*}
    \rvz = \rvz^* + \varepsilon
\end{equation*}
where $\rvz^*$ is the targeted vector and $\varepsilon$ is the vector containing the errors $\varepsilon_a,\varepsilon_b$. 

Now let $\rvx_i$ be the $i-$th row of the input matrix $\Input$, then we have
\begin{align*}
    \Input\rvz &= \Input \rvz^* + \Input\varepsilon\\
    &=\Input \rvz^* + \begin{bmatrix}
      \inner{\rvx_1}{\varepsilon}\\
      \vdots\\
      \inner{\rvx_d}{\varepsilon}\\
      \end{bmatrix}
\end{align*}
In the general case that all the columns will change, let $\rmP  = \softmax((\key\Input)^\top\query\Input)$ and $\rmP^*$ be the targeted matrix then we have that 
\begin{align*}
    \Input\rmP &= \Input\rmP^* + \Input\rmE
\end{align*}
where $\rmE= \begin{bmatrix} \varepsilon_1 &\hdots &\varepsilon_n
\end{bmatrix}$ is the matrix containing all the errors and so
\begin{align*}
    \norm{\Input\rmP- \Input\rmP^*}  &= \max_{1\leq j\leq n}\sum_{i=1}^d\abs{\inner{\rvx_i}{\varepsilon_j}}\\
    &\leq Gn^2de^{\log n -\lambda}\\
    &\leq e^{\log Gdn^3 -\lambda}
\end{align*}
Thus, if $\lambda >\log \dfrac{Gdn^3}{\eps}$ we have that 
\begin{equation*}
    \norm{\Input\rmP- \Input\rmP^*}\leq \eps
\end{equation*}

\paragraph{The error in Matrix Multiplication . }
This error has already been calculated in \cref{app:matrixmul}, however we explicitly define it here as follows:

\begin{align*}
    ne^C\dfrac{e^{cx_{ij}}}{\sum_{j=1}^n e^{cx_{ij}} + n(e^{C} +1)}  &= e^{cx_{ij}}\dfrac{1}{1+ \sum_{j=1}^ne^{cx_{ij}-C-\log n}+n}  \\
    &= (1 + cx_{ij} + O((cx_{ij})^2))( 1 - e^{cx_{ij}-C-\log n} + O( e^{2(cx_{ij}-C-\log n)}))
    \end{align*}

Let $c = \frac{\eps_1}{C_1G}$ for some constant $C_1$ and $C = \log\dfrac{C_2}{\eps_2}$ for some $C_2$ then we have 
\begin{align*}
    A &= ne^C\dfrac{e^{cx_{ij}}}{\sum_{j=1}^n e^{cx_{ij}} + n(e^{C} +1)}\\  &= e^{cx_{ij}}\dfrac{1}{1+ \sum_{j=1}^ne^{cx_{ij}-C-\log n}+n}  \\
    &= (1 + cx_{ij} + \dfrac{\eps_1^2x_{ij}^2}{G^2})( 1 - \dfrac{e^{cx_{ij}\eps_2}}{n} +  \dfrac{e^{2cx_{ij}}\eps_2^2}{n^2})\\
    &= (1 + cx_{ij})( 1 - \dfrac{e^{cx_{ij}\eps_2}}{n} +  \dfrac{e^{2cx_{ij}}\eps_2^2}{n^2}) + \dfrac{\eps_1^2x_{ij}^2}{G^2}( 1 - \dfrac{e^{cx_{ij}}\eps_2}{n} +  \dfrac{e^{2cx_{ij}}\eps_2^2}{n^2})
    \end{align*}
    Thus,
\begin{align*}
    \abs{A - (1 + cx_{ij})} &= \abs{ -(1 + cx_{ij}) \dfrac{e^{cx_{ij}\eps_2}}{n} +  \dfrac{e^{2cx_{ij}}\eps_2^2}{n^2}+ \dfrac{\eps_1^2x_{ij}^2}{G^2}( 1 - \dfrac{e^{cx_{ij}}\eps_2}{n} +  \dfrac{e^{2cx_{ij}}\eps_2^2}{n^2})}\\
    &\leq \eps_1^2(\dfrac{e^{\eps_1/C_1}\eps_2}{n}+2\dfrac{e^{2\eps_1/C_1}\eps_2^2}{n^2}) +\dfrac{e^{\eps_1/C_1}\eps_2}{n}(1+\dfrac{\eps_1}{C_1})\\
    &\leq 4\dfrac{e^{\eps_1/C_1}\eps_2}{n}
    \end{align*}  
    Hence if $\eps_2 = \eps/4$ and $\eps_1=C_1\log(n\eps)$ we have that the total error is less than $\eps$.
    
\paragraph{Function approximation.} The error in \cref{lem:sigmoids} is an immediate consequence of \cref{th:Barron} and it is proportional to $1/\sqrt{m}$, where $m$ is the number of heads we are using.

\paragraph{Accumulation of error after $T$ operations.} %
Fix an $\eps >0$ and
assume that in the $t-$th iteration the input is $\Input_{t} = \Input_{t}^* +\eps_t\rmM_t$, where $\Input_{t}^*$ is the ideal input $0<\eps_t <\dfrac{t\eps}{T}$  and $\rmM_t$ is a matrix such that $\norm{\rmM_t}\leq 1$, we will show that $\Input_{t+1} = \Input_{t+1}^*+\eps_{t+1}\rmM_{t+1}$, where $\Input_{t+1}^*$ is the ideal input, $0<\eps_{t+1} <\dfrac{(t+1)\eps}{T}$  and $\rmM_{t+1}$ is a matrix such that $\norm{\rmM_{t+1}}\leq 1$. 
\begin{itemize}
    \item Matrix Multiplication with a matrix $\rmA$, $\norm{\rmA}\leq 1$\footnote{Notice that this can be assumed without loss of generality, since we can normalize all the errors with the maximum norm of a matrix to the power of $T$.} will have the following result: 
    $$\rmA\Input_t +\eps' = \rmA\Input_{t}^* + \eps_t\rmA\rmM_t +\eps'\rmM' = \Input_{t+1}^*  +(\eps_t +\eps')\rmM_{t+1} $$ where $\eps'$ is controlled by the constants we use in the design of the function block and $\rmM_{t+1}$ is some matrix with $\norm{\rmM_{t+1}}\leq 1$. If now $\eps'<\dfrac{\eps}{T}$, our claim follows.
    \item \texttt{Read/Write} operations will result to an error of 
    $$\rmX_t\rmP = \Input_t\rmP^* +\eps'\rmM' = \Input_t^*\rmP^* + \eps_t\rmM_t\rmP^* +\eps'\rmM' $$ 
    Notice that as before, since $\norm{\rmM'}\leq 1$ and $\norm{\rmM_t\rmP^*}\leq 1$ and thus we have $\Input_{t+1} = \rmX_t\rmP = \Input_{t+1}^* + \eps_{t+1}\rmM_{t+1}$, where $\eps_{t+1} =\eps_t +\eps'$. Again if $\eps' \leq \dfrac{\eps}{T}$ the result follows.
    \item The result for function approximation follows in a similar way.
\end{itemize}
\section{\texttt{subleq} is Turing Complete}\label{app:subleq}
In this section, we show that our slightly restricted version of the original SUBLEQ instruction \citep{mavaddat1988urisc} is indeed also Turing complete. To do this, we will utilize Minsky machines, which are also Turing complete. A Minksy machine comprises of registers and a list of instructions, where each instruction can be either of the following two instructions
\begin{itemize}
    \item \texttt{add}(a): $\texttt{mem}[a]:=\texttt{mem}[a]+1$, go to the next instruction.
    \item \texttt{sub}(a, n): If $\texttt{mem}[a]==0$, go to instruction $n$. Otherwise $\texttt{mem}[a]:=\texttt{mem}[a]-1$, go to the next instruction. 
\end{itemize}

Given a program written in a language above, we translate it into an equivalent one written in our \texttt{SUBLEQ} language. For this, we initialize three fixed locations / registers $c_{-1}, c_{0}$, and $c_{+1}$ such that $\texttt{mem}[c_{-1}]:=-1$, $\texttt{mem}[c_{0}]:=0$, and $\texttt{mem}[c_{+1}]:=+1$; as well as an extra register $\texttt{mem}[b]$. We translate the program instruction-by-instruction. Assume that we have translated the first $i-1$ instructions. Let $j-1$ be the index of the last (translated) \texttt{SUBLEQ} instruction, that is, the index of the next \texttt{SUBLEQ} instruction will be $j$. Then, for the $i$-th instruction in the Minsky machine language, we translate it into our language as follows:
\begin{itemize}
    \item Case 1, The $i$-th instruction of the Minsky machine program is $\texttt{add}(a)$. This is equivalent to $\texttt{SUBLEQ} (a, c_{-1}, j+1)$, and hence the $j$ instruction in our program will simply be $\texttt{SUBLEQ} (a, c_{-1}, j+1)$.
    \item Case 2, The $i$-th instruction in the Minsky machine program is $\texttt{sub}(a, n)$. This would be  equivalent to the sequence of the following 5 \texttt{SUBLEQ} instructions.
    \begin{algorithm}[H]
        \caption{Translation for $\texttt{sub}(a, n)$}
        \begin{algorithmic}
           \State Instr. $j$\quad\;\;:\; $\texttt{SUBLEQ} (b, b, j+1)$ 
           \State Instr. $j+1$:\; $\texttt{SUBLEQ} (b, a, j+3)$
           \State Instr. $j+2$:\; $\texttt{SUBLEQ} (a, c_{+1}, j+5)$
           \State Instr. $j+3$:\; $\texttt{SUBLEQ} (a, c_{0}, n')$
           \State Instr. $j+4$:\; $\texttt{SUBLEQ} (a, c_{+1}, j+5)$
        \end{algorithmic}
    \end{algorithm}
    Here $n'$ is the index of the translation of the $n$-th instruction of the Minsky machine program. 
    This can be computed as a function of the number of $\texttt{add}$ and $\texttt{sub}$ instructions up to instruction $n$.
    The correctness of the above can be verified by considering the three cases: $\texttt{mem}[a]\geq 1$, $\texttt{mem}[a]\leq -1,$ and $\texttt{mem}[a]=0$.  
\end{itemize}

\section{Single Instruction Set}\label{app:unified}
Each instruction consists of the following tuple: $(\emb_a, \emb_b, \emb_c, \emb_{\text{flag}}, \emb_m, \emb_p)$, and does the following
\begin{enumerate}
    \item $mem[c] = f_m (mem[a], mem[b])$
    \item if $mem[\text{flag}]_{(0,0)}\leq 0$ goto instruction $p$
\end{enumerate}

Here, locations $a, b$, and $c$ can contain either scalars, or $d$-dimensional vectors or $d\times d$ matrices, and $\texttt{mem}[\text{flag}]_{(0,0)}$ is the 1-st entry of $\texttt{mem}[\text{flag}]$ if it is a vector / matrix, else it is $\texttt{mem}[{\text{flag}}]$ if a scalar.

This can be implemented using the following steps (each may use a separate layer of transformer):\\

At the beginning of each iteration, the scratchpad starts with storing the pointer to the next instruction $\emb_t$. 
\begin{enumerate}
    \item Read the command  $(\emb_a, \emb_b,\emb_c, \emb_{\text{flag}}, \emb_p, \emb_m)$ from the location to the scratchpad.
    \item Copy the $d\times d$ data at locations $a,b$ to the scratchpad memory $scratchMem$ (assume the data is $d\times d$ even if actually scalar or vector, the $f_m$ implementation will handle that)
    \item Copy the data to the $i$-th function row block using the feed forward layer.
    \item Once in the correct row block, $f_m(\texttt{mem}[a], \texttt{mem}[b])$ is computed
    \item Feedforward layers copy back the data from $i$-th row block to the scratchpad memory $scratchMem$.
    \item Write result from scratchpad memory to $\emb_c$.
    \item if $\texttt{mem}[{\text{flag}}]_{(0,0)}\leq 0$ store $\emb_p$ in the scratchpad, else $\emb_{t+1}$
\end{enumerate}

\begin{figure}[H]
    \centering
    \includegraphics[scale = 0.4]{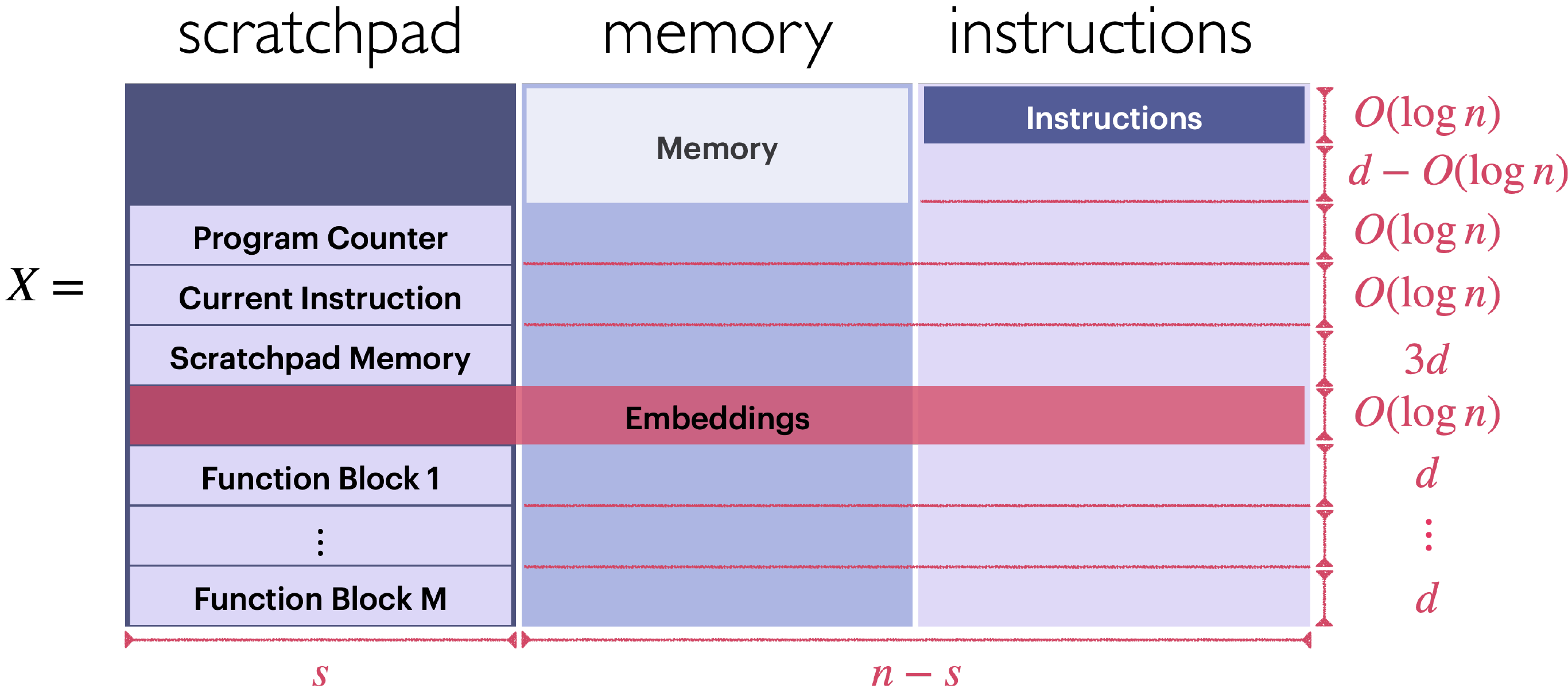}
    \caption{The structure of input $\Input$}
    \label{fig:unified_framework}
\end{figure}

The structure of the input $\Input$ is shown in \cref{fig:unified_framework}.
It has $n$ columns and $O(Md + \log n)$ rows.
It is partitioned into 3 column blocks: the Scratchpad block, the Memory block, and the Instructions block.
The Memory block is the storage and is the location where all the variables are stored. 
Each variable can be either a scalar, vector or matrix, as long as the number of rows in it are no larger than $d$.
For example, if a variable is a $d\times d$ matrix, it is stored in $d$ consecutive columns in the block, where each column has length $d$.
The address of this variable is the index of its first column in the input $\Input$.
The Instructions block contains instructions, where each instruction is a vector of the form 
\begin{align*}
    \vc = \begin{bmatrix}
        \emb_a\\
        \emb_b\\
        \emb_c\\
        \emb_m\\
        \emb_\text{flag}\\
        \emb_p\\
        d_h\\
        d_w\\
        b_{\text{mask}}^{(1)}\\
        b_{\text{mask}}^{(2)}\\
        b_{\text{mask}}^{(3)}
    \end{bmatrix}, 
\end{align*}
which encodes the following logic:
    \begin{equation*}
        \texttt{mem}[c] = f_m (\texttt{mem}[a], \texttt{mem}[b])\quad; \quad\text{if }\texttt{mem}[\text{flag}]\leq 0\text{ goto instruction }p.
    \end{equation*}
$\emb_a, \emb_b, \emb_c, \emb_p,$ and $\emb_\text{flag}$ are all binary $\pm 1$ vectors that point to the locations $a,b,c,p,$ and $\text{flag}$ respectively.
These are simply the binary representations of the integers $a,b,c,p$ and $\text{flag}$, and hence have length $\log_2 n$ each.
Similarly, $\emb_m$ is the binary vector representation of the integer $m$, and hence has length $\log_2 M$, where $M$ is the number of functions we implement. The $b_{\text{mask}}$ is mask bit used while writing the output back to memory.

The scratchpad has $s$ columns. The length $s$ depends on the maximum number of columns needed by the function blocks to operate, and can be as low as $O(1)$ for scalar and vector functions, $O(d)$ for matrix functions, and can be as high as $O(d^2)$ if functions like matrix vectorization are one of the $M$ functions.
The Scratchpad consists of the following parts:
\begin{itemize}
    \item The program counter is a row block with $\log_2 n$ rows and $s$ columns and takes the form:
\begin{align*}
    \begin{bmatrix}
        \emb_i&\emb_i&\cdots &\emb_i.
    \end{bmatrix}
\end{align*}
This signifies that the current program counter points to the $i$-th instruction.
Using this, the $i$-th instruction is read into all the $s$ columns of `Current Instruction' row block.
\item The Current Instruction row block has $O(\log n)$ rows and $s$ columns, and each column initially contains the $i$-th instruction once it is read. Then, the instructions in each column are slightly modified depending on the column index, to read memory blocks pointed to in the instruction. 
The memory blocks are read into the `Scratchpad Memory'.
\item The Scratchpad Memory is a temporary location where the data is first read into from the Memory column block, before it is moved to the correct function's Function Block, using the function index encoding $\emb_m$ in the instruction.
\item The encodings row block has $O(\log n)$ rows and $n$ columns, and is used to index every column in the input $\Input$. 
It contains the binary $\pm 1$ vector encodings of the column index for each column.
The details of this row block are explained later.
\item The Function Blocks are custom transformer blocks that can be added in a plug-n-play manner to the Unified Attention Based Computer depending on what `elementary' functions the user wants the computer to have access to.
\end{itemize}

\begin{align*}
    \Input = \left[\begin{array}{cccc|ccc|ccc}
       \zero&\zero&\dots&\zero&\mem_{s+1}&\dots&\mem_{m+s}&\begin{bmatrix}
           \vc_{m+s+1}\\
           \zero
       \end{bmatrix}&\dots&\begin{bmatrix}
           \vc_{n}\\
           \zero
       \end{bmatrix}\\
        \hline
    \emb_t&\emb_t&\dots&\emb_t&*&\dots&*&*&\dots&*\\
        \hline
\vc_t^1&\vc_t^2&\dots&\vc_t^{s}&*&\dots&*&*&\dots&*\\
        \hline
        \mem_{a_t}^1 &\mem_{a_t}^2 &\dots&\mem_{a_t}^{s} &\vZero&\dots&\vZero&\vZero&\dots&\vZero\\
        \mem_{b_t}^1 &\mem_{b_t}^2 &\dots&\mem_{b_t}^{s} &\vZero&\dots&\vZero&\vZero&\dots&\vZero\\
        \mem_{c_t}^1 &\mem_{c_t}^2 &\dots&\mem_{c_t}^{s} &\vZero&\dots&\vZero&\vZero&\dots&\vZero\\
        \hline 
        \vZero&\vZero&\dots&\vZero&\emb_{s+1}&\dots&\emb_{m+s}&\emb_{m+s+1}&\dots&\emb_{n}\\
        \emb_1&\emb_2&\dots&\emb_{s}&\vZero
        &\dots&\vZero&\vZero&\dots&\vZero\\
        \hline
        \hline
        \text{f}_1\texttt{mem}&\dots&\dots&\dots&*&\dots&*&\dots&*\\
        \vdots&\vdots&\vdots&\vdots&\vdots&\vdots&\vdots&\dots&*\\
        \text{f}_M\texttt{mem}&\dots&\dots&\dots&*&\dots&*&\dots&*\\
    \end{array}\right]
\end{align*}

\subsection{Step 1}

In this step, we need to copy the $t$-th instruction, pointed to by the program counter $\emb_t$, to the scratchpad's Current Instruction block. 
We denote the instruction by $\vc_t$ where
\begin{align*}
\vc_t=    \begin{bmatrix}
        \emb_{a_t}\\
        \emb_{b_t}\\
        \emb_{c_t}\\
        \emb_{\text{flag}_t}\\
        \emb_{p_t}\\
        \emb_{m_t}\\
        d_h\\
        d_w\\
        b_{\text{mask}}^{(1)}\\
                b_{\text{mask}}^{(2)}\\
                        b_{\text{mask}}^{(3)}
    \end{bmatrix}
\end{align*}
For this step, we only consider the following relevant subset of rows of the matrix $\rmX$:
\begin{align*}
    \Input = \left[\begin{array}{cccc|ccc|ccc}
            \zero&\zero&\dots&\zero&*&*&\dots&\vc_{m+s+1}&\dots&\vc_{n}\\
            \hline
        \emb_t&\emb_t&\dots&\emb_t&*&\dots&*&*&\dots&*\\
        \hline
\vc_t^1&\vc_t^2&\dots&\vc_t^{s}&*&\dots&*&*&\dots&*\\
        \hline
        \vZero&\vZero&\dots&\vZero&\emb_{s+1}&\dots&\emb_{m+s}&\emb_{m+s+1}&\dots&\emb_{n}
    \end{array}\right]
\end{align*}
The other rows will not be used or changed during this operation because we can simply set the corresponding rows of the $\key,\val, \query$ matrices to 0 for all heads and setting the feed forward layers to also pass the corresponding rows unchanged.

At the beginning of execution of each command, the Current Instruction row block would be empty, so the input would look like 
\begin{align*}
    \Input = \left[\begin{array}{cccc|ccc|ccc}
            *&*&\dots&*&*&*&\dots&\vc_{m+s+1}&\dots&\vc_{n}\\
            \hline
        \emb_t&\emb_t&\dots&\emb_t&*&\dots&*&*&\dots&*\\
        \hline
\vZero&\vZero&\dots&\vZero&\vZero&\dots&\vZero&\vZero&\dots&*\\
        \hline
        \vZero&\vZero&\dots&\vZero&\emb_{s+1}&\dots&\emb_{m+s}&\emb_{m+s+1}&\dots&\emb_{n}
    \end{array}\right]
\end{align*}

Then, consider an attention head with the following $\key, \query, \val$ matrices:
\begin{align*}
    \key = \begin{bmatrix}
        \zero&\zero&\zero&\id\\
    \end{bmatrix},
    \query = \begin{bmatrix}
        \zero&\id&\zero&\zero\\
    \end{bmatrix},
    \val = \begin{bmatrix}
         \zero&\zero&\zero&\zero\\
        \zero&\zero&\zero&\zero\\
        \id&\zero&\zero&\zero\\
        \zero&\zero&\zero&\zero
    \end{bmatrix}
\end{align*}
This will result in 
\begin{align*}
    \Input = \left[\begin{array}{cccc|ccc|ccc}
            *&*&\dots&*&*&*&\dots&\vc_{m+s+1}&\dots&\vc_{n}\\
            \hline
        \emb_t&\emb_t&\dots&\emb_t&*&\dots&*&*&\dots&*\\
        \hline
\vc_t&\vc_t&\dots&\vc_t&*&\dots&*&*&\dots&*\\
        \hline
        \vZero&\vZero&\dots&\vZero&\emb_{s+1}&\dots&\emb_{m+s}&\emb_{m+s+1}&\dots&\emb_{n}
    \end{array}\right].
\end{align*}

We apply \cref{app:increase} on the row blocks
\begin{align*}
    \left[\begin{array}{cccc|ccc|ccc}
\vc_t&\vc_t&\dots&\vc_t&*&\dots&*&*&\dots&*\\
        \emb_1&\emb_2&\dots&\emb_{s}&\vZero
        &\dots&\vZero&\vZero&\dots&\vZero
    \end{array}\right]
\end{align*}
to construct feedforward layers that convert $\vc_t$ to $\vc_t^i$, where 
\begin{align*}
    \vc_t^i=\begin{bmatrix}
        \emb_{a_t+i}\\
        \emb_{b_t+i-d}\\
        \emb_{c_t+i-2d}\\
        \emb_{\text{flag}_t}\\
        \emb_{p_t}\\
        \emb_{m_t}\\
        d_\text{h}\\
        d_\text{w}\\
         b_{\text{mask}}^{(1)}=1_{(i\leq  d_w)}\\
          b_{\text{mask}}^{(2)}=1_{(i>d)} + 1_{(i\leq d + d_w)} - 1\\
        b_{\text{mask}}^{(3)}=1_{(i>2d)} + 1_{(i\leq 2d + d_w)} - 1
    \end{bmatrix}.
\end{align*}
Note that the last three elements can be created using the following ReLU:
\begin{align*}
b_{\text{mask}}^{(1)}=&\relu(2d+d_w-i+1) - \relu(2d+d_w-i) \\
b_{\text{mask}}^{(2)}=&\relu(i-d) - \relu(i-d - 1) + \relu(d+d_\text{w}-i+1) - \relu(d+d_w-i) - 1\\
   b_{\text{mask}}^{(3)}= &\relu(i-2d) - \relu(i-2d - 1) + \relu(2d+d_w-i+1) - \relu(2d+d_\text{w}-i) - 1.
\end{align*}

At the end of this step, we get the following:
\begin{align*}
    \Input = \left[\begin{array}{cccc|ccc|ccc}
            \zero&\zero&\dots&\zero&*&*&\dots&\vc_{m+s+1}&\dots&\vc_{n}\\
            \hline
        \emb_t&\emb_t&\dots&\emb_t&*&\dots&*&*&\dots&*\\
        \hline
\vc_t^0&\vc_t^1&\dots&\vc_t^s&*&\dots&*&*&\dots&*\\
        \hline
        \vZero&\vZero&\dots&\vZero&\emb_{s+1}&\dots&\emb_{m+s}&\emb_{m+s+1}&\dots&\emb_{n}
    \end{array}\right],
\end{align*}

\subsection{Step 2}
Use three heads, one each for $\emb_{a},\emb_{b}$ and $\emb_c$. 

Using the vectors $\emb_{a_t+i},\emb_{b_t+i-d}$, and $\emb_{c_t+i-2d}$ we copy the data (using one head each and a similar technique as last step) to get the following in the Scratchpad Memory:
\begin{align*}
  \left[\begin{array}{ccccccccc|ccc|ccc} \mem_{a_t}&\dots&\mem_{a_t+d}&*&\dots&*&*&\dots&*&*&\dots&*&*&\dots&*\\
  *&\dots&*&\mem_{b_t}&\dots&\mem_{b_t+d}&*&\dots&*&*&\dots&*&*&\dots&*\\
  *&\dots&*&*&\dots&*&\mem_{c_t}&\dots&\mem_{c_t+s-2d}&*&\dots&*&*&\dots&*\\
    \end{array}\right]
\end{align*}

Using the mask bits at the end of $\vc_t^i$, we get
\renewcommand{\arraystretch}{1.5}
\renewcommand{\arraycolsep}{3pt}
\begin{align}
  \left[\begin{array}{cccccccccccc|cc|cc} \mem_{a_t}&\dots&\mem_{a_t+d_\text{w}-1}&\zero&\mem_{b_t}&\dots&\mem_{b_t+d_\text{w}-1}&\zero&\mem_{c_t}&\dots&\mem_{c_t+d_\text{w}-1}&\zero&\vZero&\dots&\vZero&\dots\\
  \zero&\dots&\zero&\zero&\zero&\dots&\zero&\zero&\zero&\dots&\zero&\zero&\vZero&\dots&\vZero&\dots\\
  \zero&\dots&\zero&\zero&\zero&\dots&\zero&\zero&\zero&\dots&\zero&\zero&\vZero&\dots&\vZero&\dots\\
    \end{array}\right]\label{eq:out_result} 
\end{align}

\begin{align*}
    \mem_i[1:d] &= \relu(\mem_i[1:d] - C(1- b_\text{mask}^{(1)})\vOne) - \relu(-\mem_i[1:d] - C(1- b_\text{mask}^{(1)})\vOne)\\
    &\quad + \relu(\mem_i[d+1:2d] - C(1- b_\text{mask}^{(2)})\vOne) - \relu(-\mem_i[d+1:2d] - C(1- b_\text{mask}^{(1)})\vOne)\\
    &\quad + \relu(\mem_i[2d+1:3d] - C(1- b_\text{mask}^{(1)})\vOne) - \relu(-\mem_i[2d+1:3d]- C(1- b_\text{mask}^{(1)})\vOne),\\
    \mem_i[d+1:3d]&=\vZero,
\end{align*}
where $C$ is a large positive constant.

Using the same mask bits, we also mask the row containing the output data pointers for $c$:

\begin{align}
  \left[\begin{array}{cccccccccccc|ccc|ccc}   \zero&\dots&\zero&\zero&\dots&\zero&\emb_{c_t}&\dots&\emb_{c_t+d_w-1}&\vZero&\dots&\vZero&\zero&\dots&\zero&\zero&\dots&\zero\\
    \end{array}\right] \label{eq:write_pointers}
\end{align}

\subsection{Step 3}

The following feedforward ReLU layer can move the data to the correct function blocks:
\begin{align*}
    \text{f}_k\texttt{mem}[1:d_h] &= (\relu(\mem[1:d_h] - C((1-b_{\text{mask}}^{(1)}-b_{\text{mask}}^{(2)})\vOne + \log M - \emb_k^\top \emb_m))\\&\qquad - \relu(-\mem[1:d_h] - C((1-b_{\text{mask}}^{(1)}-b_{\text{mask}}^{(2)})\vOne + \log M - \emb_k^\top \emb_m))),
\end{align*}
where $C$ is a large positive constant.

\subsection{Step 4}
Each of the $M$ functions have their own attention heads, which are constructed to be copies of their transformer based function blocks.
The results after the attention are written back into their respective row blocks. 
Since the row blocks are separate, the feedforward layers of each of the transformer based function blocks also work in parallel to store the final results in the respective row blocks.

\subsection{Step 5}
Similar to Step 3 we use the following feedforward ReLU layer to move the data from the function block back into the scratchpad memory
\begin{align*}
    \mem[1:d_h] &= \mem[1:d_h] + \sum_{k=1}^M\left(\relu((\text{f}_k\texttt{mem}[1:d_h] - \mem[1:d_h]) - C((1-b_{\text{mask}}^{(3)})\vOne + \log M - \emb_k^\top \emb_m )) \right.\\
    &\left.\qquad- \relu(-(\text{f}_k\texttt{mem}[1:d_h]-\mem[1:d_h]) - C((1-b_{\text{mask}}^{(3)})\vOne+ \log M - \emb_k^\top \emb_m))\right),
\end{align*}
where $C$ is a large positive constant.

\subsection{Step 6}
For this step we focus on the encoding row block, memory storage row block and the following rows in the input (see \eqref{eq:write_pointers}, \eqref{eq:out_result}):
\begin{align*}
  \left[\begin{array}{cccccccccccc|ccc|ccc} 
  \zero&\dots&\zero&\zero&\dots&\zero&\zero&\dots&\zero&\zero&\dots&\zero&\mem_{s+1}&\dots&\mem_{m+s}&\begin{bmatrix}
           \vc_{m+s+1}\\
           \zero
       \end{bmatrix}&\dots&\begin{bmatrix}
           \vc_{n}\\
           \zero
       \end{bmatrix}\\
  \hline
  \vZero&\dots&\vZero&\vZero&\dots&\vZero&\mem_{c_t}^\text{new}&\dots&\mem_{c_t+d_w}^\text{new}&\vZero&\dots&\vZero&\vZero&\dots&\vZero&\vZero&\dots&\vZero\\
 \hline \vZero&\dots&\vZero&\vZero&\dots&\vZero&\emb_{c_t}&\dots&\emb_{c_t+d_w}&\vZero&\dots&\vZero&\vZero&\dots&\vZero&\vZero&\dots&\vZero\\
 \hline
 \vZero&\dots&\vZero&\vZero&\dots&\vZero&\zero&\dots&\zero&\vZero&\dots&\vZero&\emb_s&\dots&\emb_{m-1}&\emb_m&\dots&\emb_{n-1}\\
    \end{array}\right]
\end{align*}

We set the Key and Query weight matrices as follows:
\begin{align*}
    \key =\query = \begin{bmatrix}
        \zero\\
        \zero\\
        \id\\
        \id
    \end{bmatrix}.
\end{align*}

\begin{align*}
    \val = \begin{bmatrix}
        \zero& \zero& \zero& \zero\\
                \id& \id& \zero& \zero\\
                        \zero& \zero& \zero& \zero\\
                                \zero& \zero& \zero& \zero
    \end{bmatrix}
\end{align*}
\renewcommand{\arraystretch}{2}
\renewcommand{\arraycolsep}{2.5pt}
\begin{align*}
&\val\Input\softmax((\key\Input)^\top \query\Input)\\
  &= \left[\begin{array}{ccccc|cccccccc|c}
  \dots&\zero&\dots&\zero&\dots&\zero&\dots&\zero&\zero&\dots&\zero&\zero&\dots&\dots\\
  \hline
 \dots&\frac{\vd_{c_t}^\text{new}+\vd_{c_t}}{2}&\dots&\frac{\vd_{c_t+d_w}^\text{new}+\vd_{c_t+d_w}}{2}&\dots&\vd_{0}&\dots&\vd_{c_t-1}&\frac{\vd_{c_t}^\text{new}+\vd_{c_t}}{2}&\dots&\frac{\vd_{c_t+d_w}^\text{new}+\vd_{c_t+d_w}}{2}&\vd_{c_t+d_w+1}&\dots&\dots\\
 \hline \dots&\zero&\dots&\zero&\dots&\vZero&\dots&\vZero&\vZero&\dots&\vZero&\zero&\dots&\dots\\
 \hline
 \dots&\zero&\dots&\zero&\dots&\zero&\dots&\zero&\zero&\dots&\zero&\zero&\dots&\dots
    \end{array}\right]
\end{align*}
Finally, we use the feedforward layers similar to the proof of \cref{lem:write} to write back $[\vd_{c_t}^\text{new}\;\dots\;\vd_{c_t+d_w}^\text{new}]$ to the correct rows.

\subsection{Step 7}
This step is identical to \cref{ss:goto}.

\section{Calculator}
\begin{lemma}
Given two constants $\epsilon,\delta \in [0,1]$, there exists a 1 hidden layer neural network $f$ with threshold activation and $d$ activations in the hidden layer, such that 
\begin{align*}
 \forall x\in\left[-C,-\delta\right]\cup\left[\delta, C\right],   \left|f(x) - \frac{1}{x}\right|\leq \epsilon,
\end{align*}
as long as $d=\Omega(\frac{\log(1/(\epsilon\delta))}{\epsilon\delta}+\log C)$.
\end{lemma}

\begin{proof}
    We partition $[\delta, C]$ into the following intervals $$[\delta, \delta (1+ \epsilon \delta)),[\delta (1+ \epsilon \delta), \delta (1+ \epsilon \delta)(1+\epsilon\delta (1+ \epsilon \delta))) \dots, [a_i, a_i(1+\epsilon a_i)),\dots,$$
    that is, if an interval begins at $a$, then it ends at $a(1+\epsilon a)$. Note that for any point $x\in [a_i, a_i(1+\epsilon a_i))$ 
    \begin{align*}
        \left|\frac{1}{x}-\frac{1}{a_i}\right| &= \frac{1}{a_i} - \frac{1}{x}\\
        &< \frac{1}{a_i} - \frac{1}{a_i(1+\epsilon a_i)}\\
         &=  \frac{\epsilon }{1+\epsilon a_i} < \epsilon.
    \end{align*}
    Hence two output activations of the form $\frac{1}{a_i}1_{x\geq a_i} - \frac{1}{a_i}1_{x< a_i(1+\epsilon a_i)}$ can be used to approximate $\frac{1}{x}$ in $[a_i, a_i(1+\epsilon a_i))$.

    Thus, all that remains is to compute the number of such intervals, and using that we get the number of output activations in the hidden layer.
    Towards that end, if the $i$-th interval begins at $a_i$,
    \begin{align*}
        a_i = a_{i-1}(1+\epsilon a_{i-1})  \geq a_{i-1}(1+\epsilon \delta) = \delta(1+\epsilon \delta)^{i-2}.
    \end{align*}
    Hence, 
    \begin{equation}
        \forall i\geq 2+ \frac{\log 1/(\epsilon\delta)}{\log (1+\epsilon\delta)}, a_i\geq \frac{1}{\epsilon}. \nonumber
    \end{equation}
    Noting that $\log (1+\epsilon\delta)>\frac{\epsilon \delta}{2}$ for $\epsilon,\delta \in [0,1]$, we get that 
    \begin{equation}
        \forall i\geq 2 + \frac{2\log 1/(\epsilon\delta)}{\epsilon\delta}, a_i\geq 1. \nonumber
    \end{equation}
    Once we have that $a_i\geq \frac{1}{\epsilon}$, the number of further partitions needed to reach $C$ would be $O(\log C)$ as shown below:
    \begin{align*}
        a_j = a_{j-1}(1+\epsilon a_{j-1}) \geq a_{j-1}\left(1+\epsilon\frac{1}{\epsilon}\right) = 2a_{j-1}.
    \end{align*}
    Hence, the total number of partitions needed is $O(\frac{\log(1/(\epsilon\delta))}{\epsilon\delta}+\log C)$.

    We can similarly approximate $1/x$ on $[-C, -\delta]$ with the same number of output activations.    
\end{proof}

\begin{lemma}
Given $\epsilon \in [0,1]$, there exists a 1 hidden layer neural network $f$ with threshold activation and $d$ activations in the hidden layer, such that 
\begin{align*}
 \forall x\in\left[0, C\right],   \left|f(x) - \sqrt{x}\right|\leq \epsilon,
\end{align*}
as long as $d=\Omega(\frac{\sqrt{C}}{\epsilon})$.
\end{lemma}

\begin{proof}
    We partition $[0, C]$ into the following intervals $$[0, \epsilon^2)),[\epsilon^2, 4\epsilon^2) \dots, [i^2\epsilon^2, (i+1)^2\epsilon^2),\dots.$$
    Note that for any point $x\in [i^2\epsilon^2, (i+1)^2\epsilon^2)$ 
    \begin{align*}
        |\sqrt{x}-\sqrt{i^2\epsilon^2}| < \sqrt{(i+1)^2\epsilon^2} -\sqrt{i^2\epsilon^2} = \epsilon.
    \end{align*}
    Hence two output activations of the form $i\epsilon1_{x\geq i^2\epsilon^2} - i\epsilon1_{x < (i+1)^2\epsilon^2}$ can be used to approximate $\sqrt{x}$ in $[i^2\epsilon^2, (i+1)^2\epsilon^2)$.

    Thus, all that remains is to compute the number of such intervals, and using that we get the number of output activations in the hidden layer. It is easy to see that the total number of intervals needed would be $\frac{\sqrt{C}}{\epsilon}$.
\end{proof}

\end{document}